\pdfoutput=1
\documentclass[accepted]{uai2023} 

\usepackage[british]{babel}

\usepackage{natbib} 
\usepackage{mathtools} 
\usepackage{booktabs} 
\usepackage{tikz} 

\usepackage{amsmath, bm}
\usepackage{amssymb}
\usepackage{bbm}
\usepackage{amsthm}
\usepackage{xcolor}
\usepackage{graphicx}
\usepackage{hyperref}
\usepackage{cleveref}
\usepackage{subcaption}
\usepackage{cancel}
\usepackage{enumitem}
\usepackage{algorithm}
\usepackage{algpseudocode}
\usepackage{multirow}
\usepackage{caption}
\usepackage{subcaption}
\usepackage{placeins}

\def\R{\mathbb{R}}

\def\D{\mathcal{D}}

\def\GCE{\text{GCE}}

\def\S{\mathcal{S}}
\def\P{\mathbb{P}}

\def\Par{\mathcal{B}}

\def\M{\mathcal{M}}
\def\X{\mathcal{X}}
\def\Y{\mathcal{Y}}

\newtheorem{proposition}{Proposition}

\newtheorem{definition}{Definition}

\newtheorem{remark}{Remark}
\newtheorem{example}{Example}

\algnewcommand\INPUT{\item[\textbf{Input:}]}%
\algnewcommand\OUTPUT{\item[\textbf{Output:}]}%

\title{TCE: A Test-Based Approach to Measuring Calibration Error}
\author[1,2]{\href{mailto:<tmatsubara@turing.ac.uk>}{Takuo Matsubara}{}}
\author[3]{Niek Tax}
\author[3]{Richard Mudd}
\author[3]{Ido Guy}
\affil[1]{The Alan Turing Institute}
\affil[2]{Newcastle University}
\affil[3]{Meta Platforms, Inc.}

\begin{document}
\maketitle

\begin{abstract}
This paper proposes a new metric to measure the calibration error of probabilistic binary classifiers, called \emph{test-based calibration error} (TCE).
TCE incorporates a novel loss function based on a statistical test to examine the extent to which model predictions differ from probabilities estimated from data.
It offers (i) a clear interpretation, (ii) a consistent scale that is unaffected by class imbalance, and (iii) an enhanced visual representation with repect to the standard reliability diagram.
In addition, we introduce an optimality criterion for the binning procedure of calibration error metrics based on a minimal estimation error of the empirical probabilities.
We provide a novel computational algorithm for optimal bins under bin-size constraints.
We demonstrate properties of TCE through a range of experiments, including multiple real-world imbalanced datasets and ImageNet 1000.
\end{abstract}

\section{Introduction}

In recent years, it has become ubiquitous to deploy complex machine learning models in real-world production systems.
Many of these systems rely on probabilistic classifiers that predict the probability that some target outcome occurs.
For such systems, it is often crucial that their predictive probabilities are \emph{well-calibrated}, meaning that the predictive probability accurately reflects the true frequency that the target outcome occurs.
In some contexts, failures to achieve calibration can lead to negative consequences.
In applications like medical diagnoses \citep{Topol2019} and autonomous driving \citep{Grigorescu2020}, associated risks are often assessed based on model predictions and the consequences of a misguided risk evaluation can be severe.
In online advertising auctions~\citep{Li2015}, it is common to incorporate a prediction of the probability of some outcome of interest (e.g., a click on an advert) when calculating an advertiser's bid.

While a number of metrics---such as log-likelihood, user-specified scoring functions, and the area under the receiver operating characteristic (ROC) curve---are used to assess the quality of probabilistic classifiers, it is usually hard or even impossible to gauge whether predictions are well-calibrated from the values of these metrics.
For assessment of calibration, it is typically necessary to use a metric that measures \emph{calibration error}, that is,~a deviation between model predictions and probabilities of target occurrences estimated from data.
The importance of assessing calibration error has been long emphasised in machine learning~\citep{Nixon2019, Minderer2021} and in probabilistic forecasting more broadly~\citep{Dawid1981, Degroot1983}.

However, existing metrics of calibration error have several drawbacks that in certain scenarios can mean that their values do not appropriately reflect true calibration performance.
In particular, we will demonstrate that values of existing calibration error metrics have an inconsistent scale that is influenced by the target class proportion.
In applications such as fraud detection~\citep{Abdallah2016,Tax2021} and advertising conversion prediction~\citep{Yang2022}, the prevalence, i.e., the proportion of instances belonging to the target class, is often very low.
This leads to situations where one may be unable to identify whether the values of calibration error metrics are small due to good calibration performance or due to the low prevalence. 
This is also problematic for monitoring applications aimed at tracking the calibration performance of a model in a production system, where the prevalence can change over time (i.e., \emph{prior probability shift}~\citep{Storkey2009}) and that makes it difficult to understand whether to attribute changes in the metric to an actual change in calibration performance or to the change in prevalence.

Furthermore, \emph{binning} of model predictions---an essential component of most calibration error metrics \citep{Naeini2015}---is often based on heuristics and lacks clear design principles.
For calibration error metrics, empirical probabilities of target occurrences are typically estimated by clustering data into several subsets based on binning of the associated model predictions.
The design of the binning scheme is a vital factor in the accurate estimation of the empirical probabilities, yet few principles guiding the design of binning schemes have emerged to date.

In this paper, we elaborate on the issues of existing calibration error metrics in~\Cref{sec:background}.
We establish a simple yet novel metric that counterbalances the issues in~\Cref{sec:methodology}.
\Cref{sec:experiment} empirically demonstrates properties of the proposed metric by experiments based on various datasets.
Related works are discussed in~\Cref{sec:related_work}, followed by the conclusion in \Cref{sec:conclusion}.
This paper focuses on the methodological aspects of the proposed new metric for binary classification, while theoretical development is left for future research.
Our contributions are summarised as follows:

\paragraph{Contributions}

\begin{itemize}[leftmargin=12.5pt]
    \item{Our primary contribution is a novel calibration error metric called \emph{test-based calibration error} (TCE). TCE is based on statistical hypothesis testing and is interpretable as a percentage of model predictions that deviate significantly from estimated empirical probabilities. TCE produces values in a normalised, comparable range $[0, 100]$ regardless of the class prevalence.}
    \item{We propose an explanatory visual representation of TCE called the \emph{test-based reliability diagram}. It carries more information than the standard reliability diagram and facilitates a better understanding of calibration performance (See~\Cref{fig:reliability_diagram}).}
    \item{We introduce an optimality criterion for bins under which optimal bins minimise an estimation error of the empirical probabilities. We then propose a novel algorithm to compute optimal bins approximately under the constraints of the minimum and maximum size of each bin.}
\end{itemize}

\section{Background} \label{sec:background}

In this section, we introduce the definition of \emph{calibration} and recap one of the most common \emph{calibration error} metrics.
We then outline several critical challenges of existing calibration error metrics.
The basic notation used in this paper is introduced below.

Denote input and output spaces respectively by $\X$ and $\Y$.
We focus on probabilistic binary classification, i.e.~$\Y = \{ 0, 1 \}$, in which a probabilistic classifier $P_\theta: \mathcal{X} \to [0, 1]$ models a conditional probability of $Y = 1$ given an input $x \in \X$.
The data $\D := \{ x_i, y_i \}_{i=1}^{N}$ are assumed to be i.i.d.~realisations from a random variable $(X, Y) \sim \P$.
To simplify notation, for any data subset $\S \subseteq \D$, we denote by $\S^x$ a set of all inputs $x$ in $\S$ and by $\S^y$ a set of all outputs $y$ in $\S$. 
By ``a set of bins'' or simply ``bins'', we mean a set of arbitrary disjoint intervals whose union is the unit interval $[0, 1]$.
For example, a set $\{ \Delta_b \}_{b=1}^{2}$ of intervals $\Delta_1 = [0.0, 0.4)$ and $\Delta_2 = [0.4, 1.0]$ is a set of bins.

\subsection{Calibration Error}

A probabilistic classifier $P_\theta: \mathcal{X} \to [0, 1]$ is said to be \emph{calibrated} \citep{Dawid1981,Brocker2009} if
\begin{align}
\P( Y = 1 \mid P_\theta(X) = Q ) = Q \label{eq:calibration}
\end{align}
for all $Q \in [0, 1]$ s.t.~the conditional probability is well-defined.
Informally, this criterion implies that the model prediction coincides with the actual probability of $Y=1$ for all inputs.
Any deviation between the actual probabilities and the model predictions in~\cref{eq:calibration} is often referred to as \emph{calibration error}, which quantifies to what degree the classifier $P_\theta$ is calibrated. 
The empirical computation of such a deviation involves estimating conditional probability $\P( Y = 1 | P_\theta(X) = Q )$ from data.
For given bins $\{ \Delta_b \}_{b=1}^{B}$, define disjoint subsets $\{ \D_b \}_{b=1}^{B}$ of data $\D$ by
\begin{align}
    \D_{b} := \{ (x_i, y_i) \in \D \mid P_\theta(x_i) \in \Delta_b \}. \label{eq:emp_d_b}
\end{align}
Simply put, $\D_{b}$ is a subset of data whose model predictions have similar values.
The conditional probability $\P( Y = 1 \mid P_\theta(X) = Q )$ for any $Q \in \Delta_b$ can then be estimated by the empirical mean of the labels in subset $\D_b$:
\begin{align}
    \P( Y = 1 \mid P_\theta(X) = Q ) \approx \widehat{P}_{b} := \frac{1}{N_b} \sum_{y_i \in \D_b^y} y_i \label{eq:estimate_p}
\end{align}
where we denote by $\widehat{P}_{b}$ the estimated conditional probability in $\D_b$ and by $N_b$ the sample size of $\D_b$.

One of the most common metrics to measure calibration error is \emph{expected calibration error} (ECE)~\citep{Naeini2015}.
ECE uses equispaced bins $\{ \Delta_b \}_{b=1}^{B}$ over $[0, 1]$ for a given number $B$ and measures an absolute difference between the averaged model predictions and the estimated conditional probability $\widehat{P}_{b}$ within each data subset $\D_b$.
The value of ECE is defined as
\begin{align}
    \text{ECE} := \sum_{b=1}^{B} \frac{N_b}{N} \left| \widehat{P}_{b} - \frac{1}{N_b} \sum_{x_i \in \D_b^x} P_\theta(x_i) \right| . \label{eq:ece}
\end{align}
ECE has an associated practical visual representation known as the \emph{reliability diagram}~\citep{Degroot1983,Niculescu-Mizil2005}, which aligns the averaged model prediction and the estimated conditional probability in each $\D_b$ (see Figure~\ref{fig:reliability_diagram}).
The reliability diagram is a powerful tool to intuitively grasp the deviation between the model and the estimated probability in ECE.

\subsection{Challenges in Calibration Error} \label{sec:challenge}

Calibration error metrics, such as ECE, are widely used in real-world applications. 
There nonetheless exist several challenges that may cause a misassessment of calibration.
These problems become evident especially when a distribution of model predictions $\{ P_\theta(x_i) \}_{i=1}^{N}$ is not well-dispersed. 
This scenario often arises in imbalanced classification where model predictions tend to be severely skewed towards either $0$ or $1$.
The following paragraphs illustrate challenges of existing calibration error metrics, which we aim to address.

\paragraph{Challenge 1 (Scale-Dependent Interpretation)}
In most calibration error metrics, the deviation between the model prediction and the estimated probability $\widehat{P}_{b}$ in each $\D_b$ is measured by the absolute difference as in~\cref{eq:ece}.
However, the use of the absolute difference can result in values that have an inconsistent scale influenced by the class prevalence.
To illustrate this problem, consider an estimated probability $\widehat{P}_b$ and an averaged model prediction denoted $\overline{Q}_b$ for some $b$ in~\cref{eq:ece}.
If $\widehat{P}_b = 0.50$ and $\overline{Q}_b = 0.49$, their absolute difference is $0.01$. 
On the other hand, if $\widehat{P}_b = 0.01$ and $\overline{Q}_b = 0.0001$, their absolute difference is $0.0099$.
Despite the comparison under the absolute difference suggesting that the probability $\overline{Q}_b = 0.0001$ with respect to $\widehat{P}_b = 0.01$ in the latter case is better calibrated than in the former case, one may reasonably argue that the latter is not well-calibrated---or at least not comparable to the former---given the stark difference in the order of magnitude.
Similarly to this illustration, the values of existing calibration metrics built on the absolute difference can be proportionally small whenever the scales of $\widehat{P}_b$ and $\overline{Q}_b$ are small.
This issue makes it difficult to distinguish whether the metric values are low due to good calibration performance or due to the small scale of the probabilities as in imbalanced classification.

\paragraph{Challenge 2 (Lack of Normalised Range)}
The range of values of calibration error metrics built on absolute differences is not normalised.
The range can vary depending on the choice of bins $\{ \Delta_b \}_{b=1}^{B}$.
To illustrate this problem, consider a bin $\Delta_b$ for some $b$.
If $\Delta_b = [0.4, 0.6]$, the absolute difference between $\widehat{P}_b$ and $\overline{Q}_b$ falls into a range $[0.0, 0.6]$ because $\widehat{P}_{b}$ is the estimated probability in $[0.0, 1.0]$ and the averaged model prediction $\overline{Q}_b$ in the bin $\Delta_b$ takes the value within $\Delta_b$.
Similarly, a different choice of bin $\Delta_b$ leads to a different range of the absolute difference.
Consequently, the choice of bins $\{ \Delta_b \}_{b=1}^{B}$ impacts the range of the final value of calibration error metrics that are built on the absolute difference.
To assure rigorous comparability of the final value of a calibration error metric, it is desirable to establish a measurement of the deviation whose value has a fixed, normalised range independent of the choice of bins.

\paragraph{Challenge 3 (Arbitrary Choice of Bins)}
An appropriate choice of bins is critical because it meaningfully impacts on final values of calibration error metrics.
Equispaced bins $\{ \Delta_b \}_{b=1}^{B}$ over $[0, 1]$ for a given number $B$ are one of the most common choices of bins in practice, as used in ECE.
However, equispaced bins can often cause a situation where a few particular bins contain the majority of the model predictions when they are not well-dispersed over $[0, 1]$, as often happens in imbalanced classification.
If some bin $\Delta_b$ contains the majority of model predictions, the corresponding estimated probability $\widehat{P}_b$ coincides approximately with the empirical mean of all labels.
On the other hand, estimated probabilities of the bins other than $\Delta_b$ become unreliable due to the small size of samples contained.
A potential solution to this problem is to use bins that adapt based on the dispersion of model predictions.
\cite{Nixon2019} proposed \emph{adaptive calibration error} (ACE) that computes the value of \cref{eq:ece} using bins $\{ \Delta_b \}_{b=1}^{B}$ based on $B$-quantiles of model predictions $\{ P_\theta(x_i) \}_{i=1}^{N}$ for given $B$.
However, questions remain regarding the optimal number $B$ of bins and the appropriate quantile to use for each bin.
To the best of our knowledge, there is no established notion of what makes bins optimal, nor do clear design principles for bins exist.

\section{Calibration Error Based on Test and Optimal Bins} \label{sec:methodology}

We propose a new calibration error metric that offers a simple yet novel solution to the challenges outlined in~\Cref{sec:challenge}.
First, in \Cref{sec:gce}, we present a general formulation of calibration error metrics that encompasses most metrics used in practice.
This general formulation allows for a structured understanding of the design of calibration error metrics.
In \Cref{sec:tce}, we derive from the general formulation a new calibration error metric, called TCE, which incorporates a loss based on a statistical test to compare model predictions with estimated empirical probabilities.
TCE produces a value that has a clear interpretation as a percentage of model predictions determined to deviate significantly from estimated empirical probabilities, which leads to a normalised range of possible values $[0, 100]$ regardless of the choice of bins $\{ \Delta_b \}_{b=1}^{B}$.
In \Cref{sec:lso}, we consider an optimal criterion of bins $\{ \Delta_b \}_{b=1}^{B}$ from the perspective of minimising an estimation error of the empirical probabilities $\{ \widehat{P}_b \}_{b=1}^{B}$.
We then develop a practical regularisation approach that ensures a minimum and maximum sample size in each subset $\D_b$.

\subsection{General Calibration Error} \label{sec:gce}

The following definition presents an abstract formulation of calibration error metrics, which we call \emph{general calibration error} (GCE) for terminological convenience.
Denote by $2^\D$ a power set of $\D$, i.e.~a space of all subsets of $\D$ and by $\M$ a space of all probabilistic classifiers below.

\begin{definition} \textbf{\emph{(GCE)}} \label{def:GCE}
Let $L: 2^\D \times \M \to \R$ be a loss of any probabilistic classifier evaluated for any data subset.
Let $\Par$ be a set of bins $\{ \Delta_b \}_{b=1}^{B}$ that define data subsets $\{ \D_b \}_{b=1}^{B}$ as in~\cref{eq:emp_d_b}.
Let $\| \cdot \|$ be a norm of a $B$-dimensional vector space.
For a given probabilistic classifier $P_\theta: \mathcal{X} \to [0, 1]$, define a scalar $\GCE_b \in \R$ for each $b=1, \cdots, B$ by
\begin{align}
    \GCE_b := L\left( \D_b, P_\theta \right) . \label{eq:gce_b}
\end{align}
Then, GCE of the probabilistic classifier $P_\theta$ is defined by
\begin{align}
    \GCE = \| ( \GCE_1, \cdots, \GCE_B ) \| .
\end{align}
\end{definition}

This formulation translates the problem of designing a calibration error metric into a problem of choosing the tuple $(L, \Par, \| \cdot \|)$.
Most existing calibration error metrics used in practice can be derived by selecting an appropriate tuple of the loss $L$, the bins $\Par$, and the norm $\| \cdot \|$ in GCE.
See~\Cref{ex:ece_gce} below for the case of ECE.
It is also immediate to show that ACE can be recovered from GCE.

\begin{example} \label{ex:ece_gce}
    Let $\Par$ be equispaced bins $\{ \Delta_b \}_{b=1}^{B}$ over $[0, 1]$, let $L$ be $L(\D_b, P_\theta) = | \frac{1}{N_b} \sum_{y \in \D_b^y} y - \frac{1}{N_b} \sum_{x \in \D_b^x} P_\theta(x) |$, and let $\| \cdot \|$ be a weighted 1-norm $\| v \| = \sum_{b=1}^{B} \frac{N_b}{N} \times | v_b |$.
    The ECE corresponds to the GCE under this tuple.
\end{example}

We aim to choose the tuple $(L, \Par, \| \cdot \|)$ so that it addresses the aforementioned challenges in~\Cref{sec:challenge}.
\Cref{sec:tce} addresses a loss $L$ based on a statistical test and presents the resulting TCE.
Subsequently, \Cref{sec:lso} addresses a choice of bins $\Par$ that is obtained through optimisation to minimise an estimation error of the empirical probabilities $\{ \widehat{P}_b \}_{b=1}^{B}$.
All norms $\| \cdot \|$ are equivalent in finite dimensions, and hence we do not focus on any particular choice. 
As with ECE, we use the weighted 1-norm $\| \cdot \|$ in \Cref{ex:ece_gce} for TCE.

\subsection{Test-based Calibration Errors} \label{sec:tce}

We present our main contribution, a new calibration error metric called TCE, that is derived from GCE by specifying a novel loss $L$ based on a statistical test.
Our proposed loss $L$ summarises the percentage of model predictions that deviate significantly from the empirical probabilities in each subset $\D_b$.
We effectively test a null hypothesis ``the probability of $Y=1$ is equal to $P_\theta(x)$'' at each $x \in \D_b^x$ using the output data $\D_b^y$.
A rigorous formulation of this loss $L$ is provided below, combined with the definition of the TCE.
Note that the bins $\{ \Delta_b \}_{b=1}^{B}$ and the norm $\| \cdot \|$ of TCE are arbitrary, while the weighted 1-norm is our default choice of $\| \cdot \|$.

\begin{definition} \textbf{\emph{(TCE)}} \label{def:TCE}
Given a statistical test and its significance level $\alpha \in [0, 1]$, let $R$ be a function of any observed dataset of random variable $Y \in \{0, 1\}$ and any probability $Q \in [0, 1]$, which returns $1$ if a hypothesis $P(Y=1) = Q$ is rejected based on the dataset and returns $0$ otherwise.
In~\Cref{def:GCE}, let $L$ be an average rejection percentage s.t.~
\begin{align}
    L(\D_b, P_\theta) = 100 \times \frac{1}{N_b} \sum_{x \in \D_b^x} R\left( \D_b^y, P_\theta(x) \right) .
\end{align}
GCE in \Cref{def:GCE} is then called TCE.
\end{definition}

In contrast to existing metrics that examine the difference between averaged model predictions and empirical probabilities in each bin, TCE examines each prediction $P_\theta(x)$ and summarises the rejection percentage in each bin.
The procedure of TCE can be intuitively interpreted as follows.

\begin{remark}
    Informally speaking, TCE examines whether each model prediction $P_\theta(x)$ can be regarded as an outlier relative to the empirical probability of the corresponding data $\D_b^y$, where the test in function $R$ acts as a criterion for determining outliers.
    The level of model-calibration is then measured by the rate of outliers produced by the model.
\end{remark}

In this paper, we use the Binomial test as the \emph{de facto} standard statistical test to define $R$ in the TCE.
TCE based on other tests, including Bayesian testing approaches, is an open direction for future research.
\Cref{alg:TCE} summarises the computational procedure of TCE.
There are multiple advantages of TCE as follows.

\begin{algorithm}[b]
\caption{Computation of TCE} \label{alg:TCE}
\begin{algorithmic}
\INPUT data $\D$, model $P_\theta$, norm $\| \cdot \|$, bins $\{ \Delta_b \}_{b=1}^{B}$, function $R$ based on a chosen test and significant level
\OUTPUT a value $\text{TCE} \in \R$
\For{$b = 1, \dots, B$}
    \State $\D_{b} \gets \{ (x_i, y_i) \in \D \mid P_\theta(x_i) \in \Delta_b \}$ \Comment{make subset}
    \State $\text{TCE}_b \gets 0$
    \For{$x_i \in \D_b^x$}
        \State $\text{TCE}_b \gets \text{TCE}_b + R(\D_b^y, P_\theta(x_i))$ \Comment{test each}
    \EndFor
    \State $\text{TCE}_b \gets 100 / N_b \times \text{TCE}_b$ 
\EndFor
\State $\text{TCE} \gets \| (\text{TCE}_1, \dots, \text{TCE}_B) \|$
\end{algorithmic}
\end{algorithm}

\paragraph{Advantage 1 (Clear Interpretation)}
The final value of TCE has a clear interpretation as a percentage of model predictions that are determined by the test of choice (here the Binomial test) to deviate significantly from estimated empirical probabilities.
Because the value is a percentage, the range of the value is normalised to $[0, 100]$.

\begin{figure*}[t!]
    \begin{subfigure}[t]{0.49\linewidth}
    \centering
    \includegraphics[trim={0 10pt 0 10pt},clip,width=\columnwidth]{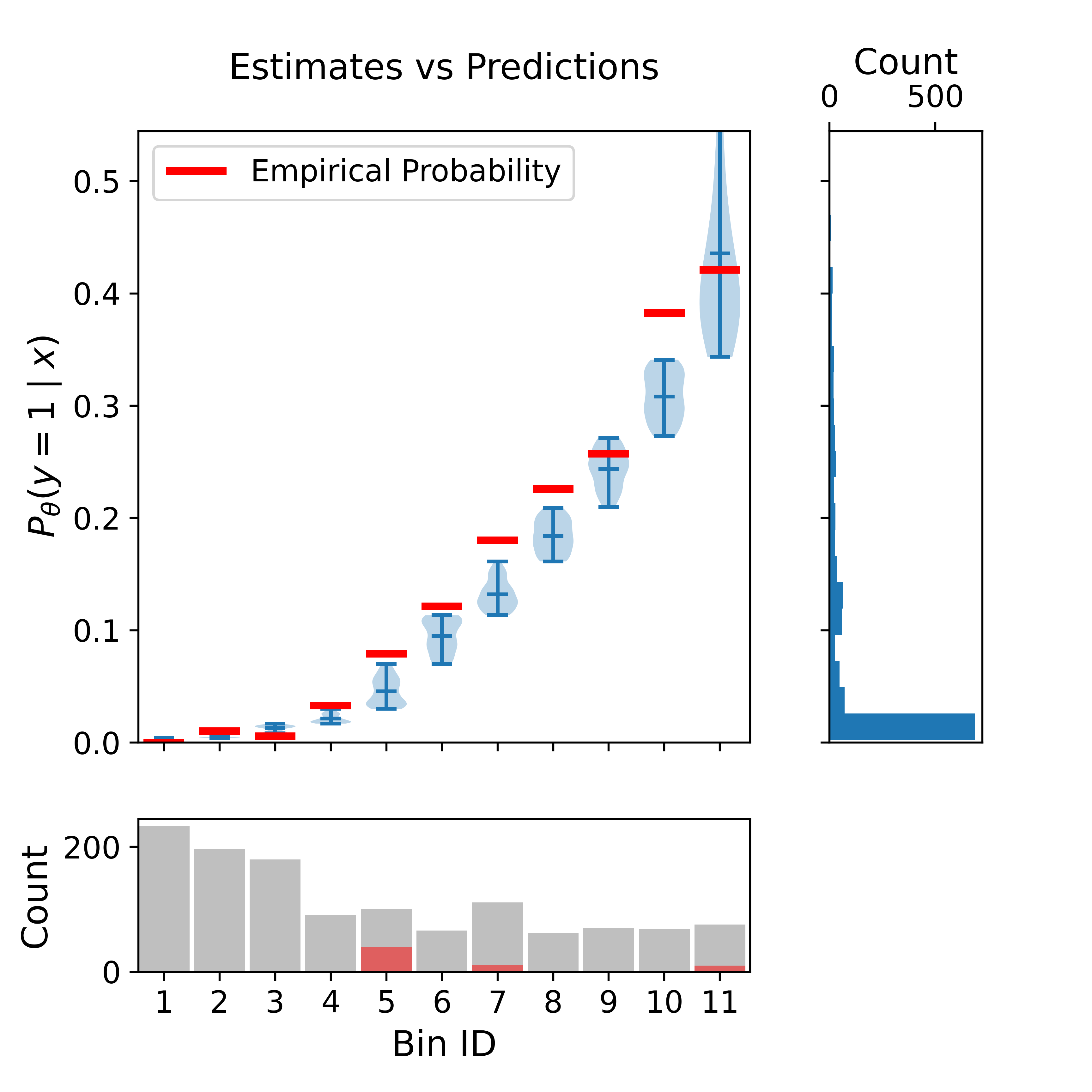}
    \end{subfigure}
    \hfill
    \begin{subfigure}[t]{0.49\linewidth}
    \centering
    \includegraphics[trim={0 10pt 0 10pt},clip,width=\columnwidth]{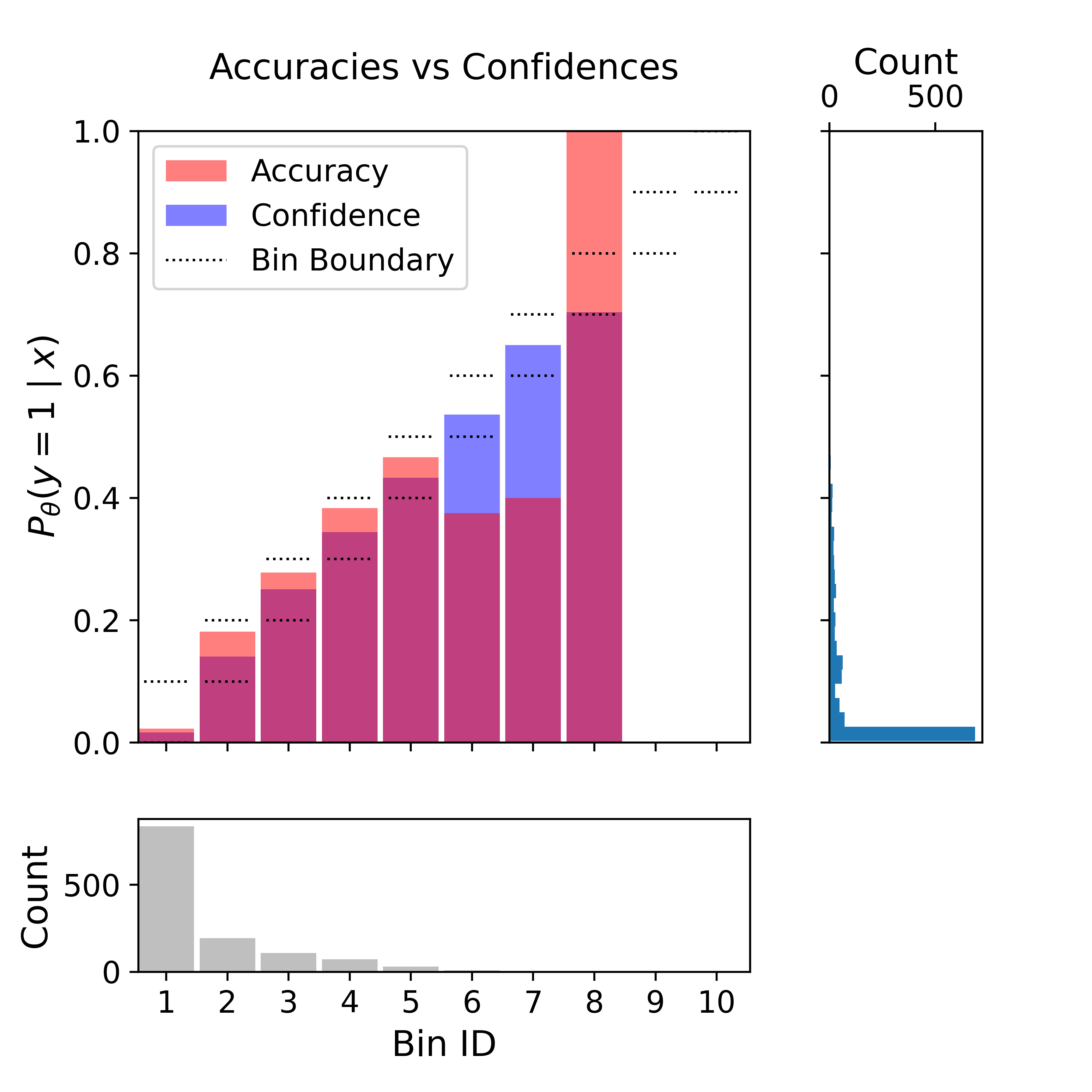}
    \end{subfigure}
    \caption{Comparison of two visual representations both applied for a gradient boosting model trained on the \emph{abalone} dataset used in \Cref{sec:experiment_uci}. (Left) A new visual representation, which we call the \emph{test-based reliability diagram}. The central plot shows a violin plot of model predictions in each bin, whose estimated probability is presented by a red line. The bottom plot shows by grey bar the sample size of each bin and by red bar the percentage of model predictions that deviate significantly from the estimated probability in each bin. The right plot shows a histogram of all model predictions. (Right) The standard reliability diagram with the bin-size plot on the bottom and the histogram plot on the right added for comparison.} \label{fig:reliability_diagram}
\end{figure*}

\paragraph{Advantage 2 (Consistent Scale)}
The test evaluates the statistical deviation of data from a model prediction $P_\theta(x)$ adaptively and appropriately for each scale of $P_\theta(x)$ and data size $N_b$.
Informally, TCE is the number of relative outliers determined for each $P_\theta(x)$ adaptively. 
This endows the value with a consistent scale robust to class imbalance.

\paragraph{Advantage 3 (Enhanced Visualisation)}
TCE leads to a new visual representation that shows the distribution of model predictions, and the proportion of model predictions that deviate significantly from an empirical probability in each bin.
See \Cref{fig:reliability_diagram} for the description and comparison with the standard reliability diagram.

Our interest is in the aggregated rejection percentage of all the tests performed, and so multiple testing corrections---e.g., the Bonferroni correction to offer a frequentist guarantee to control the familywise error rate---are not considered.
If all the null hypotheses were simultaneously true, TCE would simply coincide with the false positive rate which equals in expectation to type I error specified by the significant level of the test.
Full discussion on when and how adjustments for multiple hypotheses tests should be made may be found in \cite{Bender2001}.

Given that TCE is based on a statistical testing procedure, it may be possible to apply ideas from power analysis to inform the desired sample size in each $\D_b$.
Such analysis may also benefit the algorithm in the next subsection to compute optimal bins under the bin-size constraints, providing insights on what bin-size should be used as the constraints.
Finally, it is worth noting that TCE can be extended to multi-class classification.
The following remark presents one straightforward approach to the extension.

\begin{remark}
    Any calibration error metric defined for binary classification can be extended to multi-class classification by considering classwise-calibration \citep[e.g.][]{Kull2019}, where the calibration error metric is applied for one-vs-rest classification of each class independently.
    A modification of TCE in multi-class classification settings can then be defined as an average of TCEs applied for one-vs-rest classification of each class.
\end{remark}

\subsection{Optimal Bins by Monotonic Regressor and Bin-Size Constraints} \label{sec:lso}

It is a fundamental challenge to establish a practical and theoretically sound mechanism to design bins used in calibration error metrics.
Ideally designed bins provides accurate probability estimates $\{ \widehat{P}_b \}_{b=1}^{B}$ from data $\D$ while keeping the size of each bin reasonable.
To this end, we propose a novel algorithm to compute bins that aim to minimise an estimation error of the probability estimates $\{ \widehat{P}_b \}_{b=1}^{B}$ under the constraint of the size of each bin.

Recently, \cite{Dimitriadis2021} pointed out that an existing quadratic programming algorithm, called pool-adjacent-violators algorithm (PAVA), can be directly applied to compute ``optimal'' bins in the context of obtaining a better reliability diagram.
The bins are designed in a manner that minimises the \emph{Brier score}~\citep{Brier1950} of resulting empirical probabilities by virtue of PAVA.
Forging ahead with this observation, we introduce the following definition that makes explicit in what sense bins $\{ \Delta_b \}_{b=1}^{B}$ can be considered optimal given an arbitrary estimation error $\mathrm{D}$ of the probability estimates $\{ \widehat{P}_b \}_{b=1}^{B}$ from data $\D$.

\begin{definition} \textbf{\emph{(Optimal Bins)}} \label{def:LSO}
Let $\Pi$ be a space of all sets of bins $\{ \Delta_b \}_{b=1}^{B}$ for any $B$, with associated data subsets denoted by $\{ \D_b \}_{b=1}^{B}$ and probability estimates from $\{ \D_b^y \}_{b=1}^{B}$ denoted by $\{ \widehat{P}_b \}_{b=1}^{B}$.
Let $\mathrm{D}$ be any error function between an observed dataset of random variable $Y \in \{0, 1\}$ and a given probability $Q \in [0, 1]$.
Any set of bins that satisfies
\begin{align}
    & \min_{ \{ \Delta_b \}_{b=1}^{B} \in \Pi } ~ \sum_{b=1}^{B} W_b \times \mathrm{D}(\D_b^y, \widehat{P}_b ) \nonumber \\
    & \hspace{100pt} \text{\emph{subject to}}~\widehat{P}_1 \le \dots \le \widehat{P}_B \label{eq:M2E2_partition}
\end{align}
can be considered an optimal set of bins under the estimation error $\mathrm{D}$, where $W_b := N_b / N$ is the weight associated with the error of subset $\D_b^y$ of size $N_b$.
\end{definition}

The monotonic constraint $\widehat{P}_1 \le \cdots \le \widehat{P}_B$ of the probability estimates $\{ \widehat{P}_b \}_{b=1}^{B}$ is a natural requirement because the choice of bins becomes trivial otherwise.
For example, consider bins $\{ \Delta_b \}_{b=1}^{B}$ with $B = N$ such that $\Delta_b$ contains one single point $y_b$ and the probability estimate $\widehat{P}_b = y_b$ for each $b$.
This clearly achieves that $\sum_{b=1}^{B} W_b \times \mathrm{D}(\D_b^y, \widehat{P}_b) = \frac{1}{N} \sum_{b=1}^{N} \mathrm{D}(\{ y_b \}, y_b) = 0$.
Under the monotonic constraint, the choice of bins becomes non-trivial.

Under some choices of the estimation error $\mathrm{D}$, the optimisation of \cref{eq:M2E2_partition} can be solved as a monotonic regression problem.
Given an ordered dataset $\{ y_i \}_{i=1}^{N}$, a monotonic regression algorithm finds $N$ monotonically increasing values $\widehat{y}_1 \le \cdots \le \widehat{y}_N$ that minimise some loss between $\{ \widehat{y}_i \}_{i=1}^{N}$ and $\{ y_i \}_{i=1}^{N}$.
There exist algorithms for various losses, including the $l_p$ loss, the Huber loss, and the Chebyshev loss \citep{Leeuw2009}.
PAVA solves a monotonic regression problem under the squared error $\sum_{i=1}^{N} ( \widehat{y}_i - y_i )^2$.
If we choose the error $\mathrm{D}$ as the variance of each $\D_b^y$, i.e.,
\begin{align}
    \mathrm{D}(\D_b^y, \widehat{P}_b ) = \frac{1}{N_b} \sum_{i=1}^{N_b} ( y_i - \widehat{P}_b )^2 \label{eq:D_variance}
\end{align}
the optimal set of bins under $\mathrm{D}$ can be obtained using PAVA, which corresponds to the case of \cite{Dimitriadis2021}.
See \Cref{sec:appendix_a} for the proof that the optimisation criterion of \cref{eq:M2E2_partition} is indeed minimised at bins obtained using PAVA.
The approach using PAVA is a highly appealing solution to the design of bins $\{ \Delta_b \}_{b=1}^{B}$ because it achieves a fully-automated design of the bins based on the clear criterion of \cref{eq:M2E2_partition}.
However, such a fully-automated design can occasionally generate a bin that contains an excessively small or large number of data for the sake of minimising the aggregated estimation error over all $\{ \widehat{P}_b \}_{b=1}^{B}$.
Imposing a certain regularisation on the minimum and maximum size of each $\D_b$ can aid in keeping some baseline quality of the estimation of each individual $\widehat{P}_b$.

\begin{algorithm}
    \caption{PAVA-BC (PAVA with Block Constraints)} \label{alg:NC-PAVA}
    \begin{algorithmic}
	\INPUT ordered scalars $\{ y_i \}_{i=1}^{N}$, size constraints $N_{\text{min}}$ and $N_{\text{max}}$ s.t.~$0 \le N_{\text{min}} \le N_{\text{max}} \le N$. 
	\OUTPUT sequence $\{ \widehat{y}_i \}_{i=1}^{N}$
	\State $B \gets 0$ 
	\For{$i = 1, \dots, N - N_{min}$}
	\State $B \gets B + 1$ 
	\State $Y_B \gets y_i$ 
	\State $W_B \gets 1$ 
	\While{$B > 1$}
        \If{$W_{B-1}+W_{B} > N_{min}$}
        \State \textbf{If} $W_{B-1}+W_{B} > N_{max}$ \textbf{then} Break
        \State \textbf{If} $Y_{B-1}/W_{B-1} < Y_{B}/W_{B}$ \textbf{then} Break
        \EndIf
	\State $Y_{B-1} \gets Y_{B-1} + Y_B$
	\State $W_{B-1} \gets W_{B-1} + W_B$
	\State $B \gets B - 1$
	\EndWhile
	\EndFor
        \If{$W_B + N_{min} \le N_{max}$}
        \State $Y_B \gets Y_B + \sum_{i=N-N_{\text{min}+1}}^{N} y_i$
	\State $W_B \gets W_B + N_{min}$
        \Else
	\State $B \gets B + 1$
	\State $Y_B \gets \sum_{i=N-N_{\text{min}}+1}^{N} y_i$
	\State $W_B \gets N_{min}$
        \EndIf
        \State $s \gets 0$
        \For{$j = 1, \dots, B$}
        \For{$k = 1, \dots, W_j$}
        \State $\widehat{y}_{s + k} \gets Y_{j} / W_{j}$
        \EndFor
        \State $s \gets s + W_j$
        \EndFor
\end{algorithmic}
\end{algorithm}

\begin{algorithm}
    \caption{Near-Optimal Bins Based on PAVA-BC} \label{alg:LSO}
    \begin{algorithmic}
        \INPUT data $\D$, model $P_\theta$, size constraints $N_{\text{min}}$ and $N_{\text{max}}$ s.t.~$0 \le N_{\text{min}} \le N_{\text{max}} \le N$.  
        \OUTPUT a set of bins $\{ \Delta_b \}_{b=1}^{B}$
        \State $\{ y_i \}_{i=1}^{N} \gets \text{Sort}(\D, P_\theta)$
        \State $\{ \widehat{y}_i \}_{i=1}^{N} \gets \text{PAVA-BC}(\{ y_i \}_{i=1}^{N}, N_{\text{min}}, N_{\text{max}})$
        \State $B \gets 1$
        \State $L \gets 0$
        \State $R \gets 0$
        \For{$i = 2, \dots, N$}
        \If{$\widehat{y}_{i-1} \ne \widehat{y}_i$}
        \State $R \gets ( P_\theta(x_{i-1}) + P_\theta(x_i) ) / 2$
	\State $\Delta_{B} \gets [L, R)$
	\State $L \gets R$
        \State $B \gets B + 1$
	\EndIf
	\EndFor
	\State $\Delta_{B} \gets [L, 1.0]$
    \end{algorithmic}
\end{algorithm}

Therefore, we propose a modified version of PAVA that regularises based on the given minimum and maximum size of each subset $\D_b^y$.
\Cref{alg:NC-PAVA} summarises the full algorithm, which we call \emph{PAVA with block constraints} (PAVA-BC), followed by \Cref{alg:LSO} that summarises how to compute bins using PAVA-BC accordingly, where $\text{Sort}(\mathcal{D}, P_\theta)$ in \Cref{alg:LSO} denotes any algorithm that sorts labels $\{ y_i \}_{i=1}^{N}$ in acending order of model predictions $\{ P_\theta(x_i) \}_{i=1}^{N}$.
By \Cref{alg:LSO}, we can obtain bins that satisfy the given minimum and maximum size constraints $N_{\text{min}}$ and $N_{\text{max}}$ in each $\D_b$, while benefitting from the automated design of bins by PAVA.
A set of bins based on PAVA can be recovered by replacing PAVA-BC with PAVA in \Cref{alg:LSO}.
In general, the introduction of the regularisation can cause mild violation of the monotonicity $\widehat{P}_1 \le \cdots \le \widehat{P}_B$, meaning that there may exist a few values $\widehat{P}_b$ that is smaller than $\widehat{P}_{b-1}$.
See \Cref{sec:appendix_b} for each example where mild violation of the monotonicity by PAVA-BC occured and did not occur.
In practice, mild violation of the monotonicity can often be a reasonable cost to achieve better properties of bins.
For example, \cite{Tibshirani2011} studied settings where the monotonicity is only ``nearly" satisfied.

See \Cref{fig:comparison_bins} for a comparison of the bins computed by three different approaches: PAVA, PAVA-BC, and binning based on $10$-quantiles.
The bins produced by PAVA-BC interpolate between the optimal bins produced by PAVA and the well-sized bins produced by binning based on quantiles.
This is further confirmed by \Cref{tab:comparison_bins} which shows the total estimation error in \cref{eq:M2E2_partition} and the estimation error within each bin in \cref{eq:D_variance} for each approach.
The total estimation error is minimised by PAVA, while an average of the estimation error within each bin is minimised by binning based on quantiles.
In contrast, PAVA-BC takes a balance between the total and individual estimation error.

\begin{figure}[t]
    \centering
    \includegraphics[width=\columnwidth]{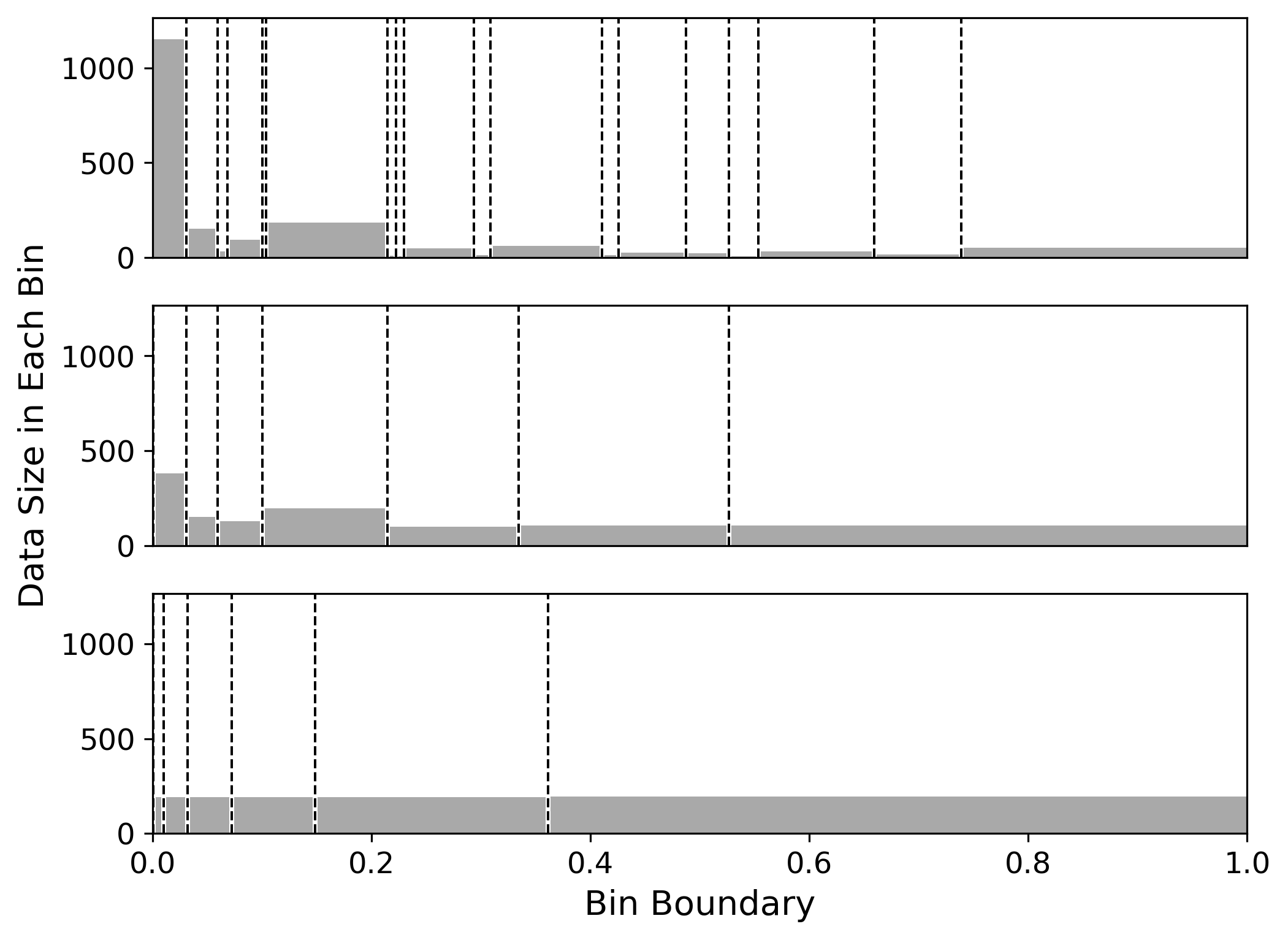}
    \caption{Comparison of bins for a random forest model on the \emph{satimage} dataset used in \Cref{sec:experiment_uci} based on (top) PAVA, (middle) PAVA-BC, (bottom) binning based on $10$-quantiles. The dotted line represents the boundary of each bin and the grey bar represents the size of each bin.}
    \label{fig:comparison_bins}
\end{figure}

\begin{table}[h]
\caption{The total estimation error and an average of the estimation error within each bin for the bins in \Cref{fig:comparison_bins}.} \label{tab:comparison_bins}
\centering
\setlength{\tabcolsep}{0.55em}
\resizebox{\linewidth}{!}{ 
\begin{tabular}{@{}lccc@{}}
\toprule
 & \bfseries PAVA    & \bfseries PAVA-BC    & \bfseries Quantile \\
\midrule
Total Error & 0.040 & 0.042 & 0.048 \\
Averaged Within-Bin Error & 0.132 & 0.077 & 0.047 \\
\bottomrule
\end{tabular}
}
\end{table}

\section{Empirical Evaluation} \label{sec:experiment}

In this section, we demonstrate the properties of TCE via three experiments.
The first experiment uses synthetic data to examine the properties of TCE under controlled class imbalance.
The second experiment involves ten real-world datasets from the University of California Irvine (UCI) machine learning repository \citep{Dua2017}, where nine are designed as benchmark tasks of imbalanced classification, and one is a well-balanced classification task for comparison.
In the second experiment, we also demonstrate that ECE and ACE may produce misleading assessments of calibration performance under class imbalance.
TCE has the potential to reduce such misinterpretation risks.
The final experiment uses the ImageNet1000 dataset to illustrate that TCE is applicable to large-scale settings.
In all experiments, models are fitted to training data first and any calibration error metric are computed using validation data.
Source code to reproduce the experiments is available in \url{https://github.com/facebookresearch/tce}.

We compute TCE with bins based on PAVA-BC unless otherwise stated.
The minimum and maximum size of each bin for PAVA-BC are set to $N / 20$ and $N / 5$ for a given dataset size $N$.
Under these constraints, the number of bins based on PAVA-BC falls into a range between 5 and 20.
In addition to ECE and ACE, we include the maximum calibration error (MCE) \citep{Naeini2015} for comparison.
MCE is defined by replacing the weighted 1-norm with the supremum norm over $b = 1, \dots, B$ in \Cref{ex:ece_gce}.
We denote, by TCE(Q) and MCE(Q), TCE and MCE each with bins based on $B$-quantiles.
For all metrics, $B$-equispaced bins and $B$-quantiles bins are computed with $B = 10$.

\subsection{Synthetic Data with Controlled Class Imbalance}

We first examine TCE using synthetic data from a simulation model considered in~\cite{Vaicenavicius2019}.
The data are simulated from a Gaussian discriminant analysis model $(x, y) \sim P(x \mid y) P(y)$. 
The output $y \in \{0, 1\}$ is first sampled from a Bernoulli distribution $P(y)$ with parameter $\pi$ and the input $x \in \R$ is then sampled from a Gaussian distribution $P(x \mid y) = \mathcal{N}(m_y, s_y)$ with mean $m_y$ and scale $s_y$ dependent of $y$.
We set $m_y = (2 \times y - 1)$ and $s_y = 2$, and change the parameter $\pi$ for each setting below.
By Bayes' theorem, the conditional probability of $y$ given $x$ corresponds to a logistic model: $P(y \mid x) = 1 / (1 + \exp(\beta_0 + \beta_1 \times x))$ where $\beta_0 = \log( \pi / (1 - \pi))$ and $\beta_1 = 4$.
A logistic model is therefore capable of reproducing the probability $P(y \mid x)$ of this synthetic data perfectly.

We consider two baseline cases of (i) well-balanced classification and (ii) imbalanced classification in this experiment.
We train a logistic model for the training data simulated with the parameter $\pi = 0.5$ (i.e.~50\% prevalence) in case (i) and with $\pi = 0.01$ (i.e.~1\% prevalence) in case (ii).
In each case (i) and (ii), we generate three different test datasets to create situations where the trained model is (a) well-calibrated, (b) over-calibrated, and (c) under-calibrated.
We examine the performance of TCE under these senarios.
Test datasets for senarios (a), (b), and (c) are generated from the simulation model with prevalences $50\%$, $40\%$, and $60\%$ in case (i) and with prevalences $1\%$, $0\%$, and $2\%$ in case (ii).
We generate 20000 data points in total, of which 70\% are training data and 30\% are test data.

\Cref{tab:table_41} shows the values of four calibration error metrics applied to the logistic regression model in each scenario.
\Cref{tab:table_41} demonstrates that all values of ECE and ACE in imbalanced case (ii) can be smaller than---or very close to---values for well-calibrated senario (a) in well-balanced case (i).
For example, the ECE value for case (ii)-(b) was smaller than that for case (i)-(a).
In contrast, TCE provides values with a consistent scale in both well-balanced and imbalanced cases.
More simulation studies of TCE with different hyperparameters are presented in \Cref{sec:appendix_c1}. 

\begin{table}[h]
\caption{Comparison of four calibration error metrics under senarios (a) - (c) in each case (i) and (ii).} \label{tab:table_41}
\centering
\resizebox{\linewidth}{!}{ 
\begin{tabular}{@{}lccccc@{}}
\toprule
\bfseries Prevalence & \bfseries TCE & \bfseries TCE(Q) & \bfseries ECE    & \bfseries ACE \\
\midrule
50\% vs 50\% & 7.28\%  & 10.88\%  & 0.0138 & 0.0150 \\
50\% vs 40\% & 96.10\%  & 96.47\%  & 0.0963 & 0.0951 \\
50\% vs 60\% & 98.83\%  & 98.93\%  & 0.1097 & 0.1096 \\
1\% vs 1\% & 3.40\%  & 0.18\%  & 0.0017 & 0.0031 \\
1\% vs 0\% & 95.50\%  & 68.73\%  & 0.0094 & 0.0094 \\
1\% vs 2\% & 92.32\%  & 89.73\%  & 0.0139 & 0.0139 \\
 \bottomrule
\end{tabular}
}
\end{table}

\subsection{Imbalanced UCI Datasets} \label{sec:experiment_uci}

Next, we compare calibration error metrics using real-world datasets in the regime of severe class imbalance. 
We use nine UCI datasets that were preprocessed by \cite{Lemaitre2017} as benchmark tasks of imbalanced classification.
We also use one additional UCI dataset with a well-balanced prevalence for comparison.
For each dataset, 70\% of samples are used as training data and 30\% of samples are kept as validation data.
We train five different algorithms: logistic regression (LR), support vector machine (SVM), random forest (RF), gradient boosting (GB), and multi-layer perceptron (MLP).
We evaluate the calibration performance of each model by five different calibration error metrics in the following tables.
\Cref{tab:table_41_1,tab:table_41_2} show results for the imbalanced datasets, \emph{abalone} and \emph{webpage} \citep{Dua2017}, respectively.
Results for all the other datasets are presented in \Cref{sec:appendix_c2}.
In \Cref{tab:table_41_1}, the best model ranked by TCE and ACE agree with each other while ECE identifies RF as the best model.
It can be observed from the reliability diagram of ECE for both the datasets in \Cref{sec:appendix_c2} that a large majority of model predictions are contained in a single bin of ECE. 
In such cases, ECE becomes essentially equivalent to a comparison of global averages of all labels and all model predictions.
\Cref{tab:table_41_2} demonstrates a situation where ECE and ACE risk misleading assessments of calibration performance.
Several values of ECE and ACE are all sufficiently small in \Cref{tab:table_41_2}, by which one may conclude that it is reasonable to use a model with the smallest calibration error.
However, the values of TCE indicate that no model has a good calibration performance.
In fact, relatively large statistical deviations between model predictions and empirical probabilities can be observed from the test-based reliability diagram for the webpage dataset in \Cref{sec:appendix_c2}.

\begin{table}[h]
\caption{Comparison of five calibration error metrics for five different algorithms trained on the abalone dataset.} \label{tab:table_41_1}
\centering
\resizebox{\linewidth}{!}{ 
\begin{tabular}{@{}lccccc@{}}
\toprule
 & \bfseries TCE    & \bfseries ECE    & \bfseries ACE    & \bfseries MCE    & \bfseries MCE(Q) \\
\midrule
 LR  & 7.26\%  & 0.0140 & 0.0252 & 0.0946 & 0.0851 \\
 SVM & 47.21\% & 0.0436 & 0.0473 & 0.8302 & 0.1170 \\
 RF  & 33.89\% & 0.0127 & 0.0177 & 0.0670 & 0.0547 \\
 GB  & 4.86\%  & 0.0182 & 0.0160 & 0.2965 & 0.0418 \\
 MLP & 3.83\%  & 0.0167 & 0.0122 & 0.0806 & 0.0540 \\
\bottomrule 
\end{tabular}
}
\end{table}

\begin{table}[h]
\caption{Comparison of five calibration error metrics for five different algorithms trained on the webpage dataset.} \label{tab:table_41_2}
\centering
\resizebox{\linewidth}{!}{ 
\begin{tabular}{@{}lccccc@{}}
\toprule 
 & \bfseries TCE    & \bfseries ECE    & \bfseries ACE    & \bfseries MCE    & \bfseries MCE(Q) \\
\midrule 
 LR  & 40.16\% & 0.0044 & 0.0034 & 0.3134 & 0.0214 \\
 SVM & 59.83\% & 0.0043 & 0.0057 & 0.5402 & 0.0239 \\
 RF  & 99.66\% & 0.0234 & 0.0241 & 0.5980 & 0.1189 \\
 GB  & 71.12\% & 0.0086 & 0.0107 & 0.2399 & 0.0436 \\
 MLP & 49.81\% & 0.0090 & 0.0018 & 0.4344 & 0.0076 \\
\bottomrule 
\end{tabular} 
}
\end{table}

\subsection{K-vs-Rest on ImageNet1000}

Finally, we demonstrate that TCE is applicable for a large-scale binary classification task using ImageNet1000 data.
We consider a K-vs-rest classification problem by using a set of all dog-kind classes (from class 150 to class 275) as a positive class and the rest as a negative class.
Under this setting, 12.5\% of validation samples belong to the positive class.
We used 5 different trained models: AlexNet, VGG19, ResNet18, ResNet50, and ResNet152.
Their calibration errors were measured based on the ImageNet1000 validation dataset consisting of 50000 data points.
\Cref{tab:table_imagenet} demonstrates that TCE produces interpretable values, with model rankings that largely agree with other metrics in this setting.
The last row of \Cref{tab:table_imagenet} shows the average computational time of each metric.
Computation of all the procedures in TCE required only 71.78 seconds for 50000 data points with 1 CPU on average.
The reliability diagrams corresponding to the results are presented in \Cref{sec:appendix_c3}.

\begin{table}[ht]
\caption{Comparison of five calibration error metrics for five different deep learning models on ImageNet1000 data.} \label{tab:table_imagenet}
\centering
\setlength{\tabcolsep}{0.55em}
\resizebox{\linewidth}{!}{ 
\begin{tabular}{@{}lccccc@{}}
\toprule
 & \bfseries TCE    & \bfseries ECE    & \bfseries ACE    & \bfseries MCE    & \bfseries MCE(Q)    \\
\midrule
AlexNet & 42.74\% & 0.0070 & 0.0070 & 0.1496 & 0.0528 \\
VGG19   & 23.57\% & 0.0028 & 0.0028 & 0.2148 & 0.0247 \\
Res18   & 29.93\% & 0.0042 & 0.0042 & 0.2368 & 0.0350 \\
Res50   & 24.60\% & 0.0020 & 0.0018 & 0.1911 & 0.0152 \\
Res152  & 16.09\% & 0.0012 & 0.0013 & 0.1882 & 0.0102 \\
\midrule
\bfseries{Time (s)} & 71.78  & 0.4873  & 0.4221  & 0.0046  & 0.0063 \\ 
\bottomrule
\end{tabular}
}
\end{table}


\section{Related Work} \label{sec:related_work}

Several calibration error metrics have been proposed, including the aforementioned ECE. 
MCE is a widely used variant of ECE that replaces the summation over $b=1, \dots, B$ in~\eqref{eq:ece} with the supremum over $b=1, \dots, B$.
\citep{Kumar2019} introduce a more general $l_p$ calibration error, which includes both ECE and MCE.
ACE replaces the equispaced bins in ECE with bins designed based on quantiles of model predictions, which prevents high concentration of data in one bin when data is imbalanced~\citep{Nixon2019}.
These calibration error metrics can be extended to multi-class classification~\citep{Kumar2019}.
Other than calibration error, scoring functions~\citep{Gneiting2007} are commonly used measurements to evaluate a probabilistic classifier.
\citep{wallace2014improving} reported a limitation of the Brier score for imbalanced classification, and proposed the \emph{stratified} Brier score that aggregates multiple Brier scores.

This paper designed a new calibration error metric based on a statistical test.
While statistical tests have been used in the context of calibration, we are the first to incorporate a statistical test into the design of a calibration error metric.
\cite{Vaicenavicius2019} performed a statistical test on whether ECE computed for synthetic data generated from predictive probabilities is significantly different from ECE computed for actual data.
Similarly, \cite{Widmann2019} proposed a statistical test of the value of their calibration error metric built on kernel methods.
In contrast to existing works which considered a test for final values of calibration error metrics, our approach incorporates a test into the metric itself.

While the use of binning is vital in the vast majority of calibration metrics, there are a few works on the \emph{binning-free} design of calibration error metrics.
The main idea is to use an cumulative distribution function (CDF) of predictive probabilities, which can be estimated without binning, and evaluate how significantly it differs from an ideal CDF that occurs if the predictive probabilities are all well-calibrated.
For example, \cite{Gupta2021} and \cite{Arrieta-Ibarra2022} considered the Kolmogorov-Smirnov test for the empirical CDF, where \cite{Gupta2021} further proposed a spline interpolation to obtain a continuous approximation of the CDF.
An approach proposed by \cite{Kull2017} can also be regarded as binning-free.
It uses a continuous CDF of the beta distribution produced by their calibration method, mentioned below, rather than the empirical CDF.

\emph{Calibration methods} refer to algorithms used to improve the calibration performance of a model $P_\theta$.
Usually, they learn some `post-hoc' function $\varphi: [0, 1] \to [0, 1]$ to be applied to each model predictio so that the new prediction $\varphi(P_\theta(x))$ is better calibrated.
Various calibration algorithms have been proposed in parallel to the development of calibration error metrics.
Platt scaling uses a logistic function for the post-hoc function $\varphi$ \citep{Platt1999}. 
Alternatively, \cite{Kull2017,Kull2019} proposed to use a beta distribution in binary classification and a Dirichlet distribution in multi-class classification.
Isotonic regression is a powerful non-parametric approach to find a monotonically increasing function $\varphi$ that minimises the Brier score~\citep{Zadrozny2002}.
Finally, Bayesian Binning into Quantiles by~\cite{Naeini2015} extends a classical histogram-based calibration~\citep{Zadrozny2001} to an ensemble of histogram-based calibrations based on Bayesian model averaging.

\section{Conclusion} \label{sec:conclusion}
In this paper, we proposed a new calibration error metric TCE that incorporates a novel loss function based on a statistical test.
TCE has (i) a clear interpretation as a percentage of model predictions determined to deviate significantly from estimated empirical probabilities, (ii) a consistent scale that is robust to class imbalance, and (iii) an informative visual representation that facilitates a better understanding of calibration performance of probabilistic classifiers.
We further introduced an optimality criterion of bins associated with a minimal estimation error of the empirical probabilities and a new algorithm to compute optimal bins approximately under the constraint of the size of each bin.

Our proposal opens up room for new research directions in the context of calibration.
This paper focuses on the methodological development of TCE.
There are various directions to investigate in terms of theoretical properties of TCE.
These include the convergence properties of TCE in the limit of data size $N$, understanding the minimum number of data points that should be contained in each subset $\D_b$, and a rigorous theoretical analysis of PAVA-BC. 
By continuing to investigate these areas, we can refine and expand our understanding of the capabilities of TCE.

\begin{acknowledgements}
The authors would like to thank Abbas Zaidi, Michael Gill, and Will Bullock for their useful feedback on early work of this paper.
TM is supported by The Alan Turing Institute under the EPSRC grant EP/N510129/1.
\end{acknowledgements}

\bibliography{bibliography}


\onecolumn 
\title{TCE: A Test-Based Approach to Measuring Calibration Error\\(Supplementary Material)}
\maketitle

\appendix

This supplement contains all the additional results referred to in the main text. \Cref{sec:appendix_a} contains the proof that the optimisation criterion of \cref{eq:M2E2_partition} is indeed minimised using PAVA.
\Cref{sec:appendix_b} shows an example of bins obtained using PAVA-BC that caused mild violation of the monotonic constraint of the empirical probabilities $\{ \widehat{P}_b \}_{b=1}^{B}$.
Finally, additional experimental results are presented in \Cref{sec:appendix_c}.

\section{Optimal Bins Based on PAVA} \label{sec:appendix_a}

The optimal bins defined by \Cref{def:LSO} can be exactly computed under the error function $\mathrm{D}$ specified by \cref{eq:D_variance} which corresponds to the variance of each $\D_b^y$.
The optimal bins result in minimisation of a weighted average of the variance of each $\D_b^y$ over all $b$, where the weights are proportional to the size of each bin.
The following proposition shows that \Cref{alg:LSO} with PAVA-BC replaced by PAVA generates the optimal bins under the error function $\mathrm{D}$.
In what follows, we assume a standard setting where the solution of \cref{eq:M2E2_partition} is at least not a set of only one single bin, i.e., $\{ \Delta_b \}_{b=1}^{1} = \{ [0, 1] \}$.

\begin{proposition}
The minimum of \cref{eq:M2E2_partition} in \Cref{def:LSO} under the error function $\mathrm{D}$ in \cref{eq:D_variance} is attained at bins computed by \Cref{alg:LSO} with PAVA-BC replaced by PAVA.
\end{proposition}

\begin{proof}
First, we show that the optimasation problem of \cref{eq:M2E2_partition} in \Cref{def:LSO} under the loss function $\mathrm{D}$ in \cref{eq:D_variance} is equivalent to the monotonic regression problem under the squared error.
Recall that, given a choice of bins $\{ \Delta_b \}_{b=1}^{B}$, each label subset $\D_{b}^y$ is defined by $\D_{b}^y := \{ y_i \in \D^y \mid P_\theta(x_i) \in \Delta_b \}$.
The input of \Cref{alg:LSO} is a set of labels $\D^y = \{ y_i \}_{i=1}^{N}$ ordered by in ascending order of $\{ P_\theta(x_i) \}_{i=1}^{N}$.
This means that each label subset $\D_{b}^y$ is a set of consecutive elements in the ordered set $\{ y_i \}_{i=1}^{N}$.
Therefore, there exist corresponding indices $n_b$ and $n_{b+1}$ s.t.~each label subset $\D_{b}^y$ can expressed by
\begin{align*}
    \D_{b}^y = \{ y_i \in \D^y \mid P_\theta(x_i) \in \Delta_b \} = \{  y_i  \in \D^y \mid i~~\text{s.t.}~~n_b \le i < n_{b+1} \} .
\end{align*}
Accordingly, with the ordered labels $\D^y$, each empirical probability $\widehat{P}_b$ in $\D_{b}^y$ can be expressed by
\begin{align*}
    \widehat{P}_b = \frac{1}{N_b} \sum_{y \in \D_b^y} y = \frac{1}{n_{b+1} - n_b} \sum_{j = n_b}^{n_{b+1}-1} y_j . 
\end{align*}
Define a set of scalars $\{ g_i \}_{i=1}^{N}$ whose element $g_i \in [0, 1]$ corresponds to the empirical probability $\widehat{P}_b$ of the bin index $b$ if $n_b \le i < n_{b+1}$.
Namely,
\begin{align}
    g_i  := \widehat{P}_b = \frac{1}{n_{b+1} - n_b} \sum_{j = n_b}^{n_{b+1}-1} y_j \quad \text{for each} \quad i \quad \text{s.t.} \quad n_b \le i < n_{b+1} . \label{eq:g_i}
\end{align}
Under these notations, the optimisation criterion in \cref{eq:M2E2_partition} can be rewritten as
\begin{align}
    \sum_{b=1}^{B} W_b \times \mathrm{D}(\mathcal{D}_b, \widehat{P}_b) & = \frac{1}{N} \sum_{b=1}^{B} \sum_{y \in \mathcal{D}_b} \left( y - \widehat{P}_b \right)^2 = \frac{1}{N} \sum_{b=1}^{B} \sum_{i = n_b}^{n_{b+1}-1} \left( y_i - \widehat{P}_b \right)^2 = \frac{1}{N} \sum_{i=1}^{N} ( y_i - g_i )^2 . \label{eq:average_loss}
\end{align}
This formulation translates a problem of choosing bins $\{ \Delta_b \}_{b=1}^{B}$ into a problem of finding a monotonically increasing sequence $\{ g_i \}_{i=1}^{N}$ that is determined by the choice of indices $\{ n_b \}_{b=1}^{B}$, so that \cref{eq:average_loss} is minimised.
Therefore the optimasation problem of \cref{eq:M2E2_partition} in \Cref{def:LSO} under the loss function $\mathrm{D}$ in \cref{eq:D_variance} is equivalent to the monotonic regression problem under the squared error whose solution sequence $\{ g_i \}_{i=1}^{N}$ is restriced to a form of \cref{eq:g_i}.

Next, consider a standard monotonic regression problem under the square error $\sum_{i=1}^{N} ( y_i - \widehat{y}_i )^2$ for the ordered set $\{ y_i \}_{i=1}^{N}$.
PAVA finds a monotonically increasing sequence $\{ \widehat{y}_i \}_{i=1}^{N}$ that minimises the square error.
The solution sequence $\{ \widehat{y}_i \}_{i=1}^{N}$ by PAVA is given in a form of \cref{eq:g_i}; see e.g.~\citep{Leeuw2009,Henzi2022}.
This means that there exists a set of indices $\{ n_b^* \}_{b=1}^{B}$ s.t.~the solution sequence $\{ \widehat{y}_i \}_{i=1}^{N}$ by PAVA is expressed as
\begin{align*}
    \widehat{y}_i = \frac{1}{n_{b+1}^* - n_b^*} \sum_{j = n_b^*}^{n_{b+1}^*} y_j \quad \text{for each} \quad i \quad \text{s.t.} \quad n_b^* \le i < n_{b+1}^* 
\end{align*}
and the sequence $\{ \widehat{y}_i \}_{i=1}^{N}$ satisfies the monotonic constraint $\widehat{y}_1 \le \dots \le \widehat{y}_N$ holds.
We can obtain such a solution sequence $\{ \widehat{y}_i \}_{i=1}^{N}$ by applying any standard implementation of PAVA.

An output of most implementations of PAVA is the solution sequence $\{ \widehat{y}_i \}_{i=1}^{N}$ rather than the associated indices $\{ n_b^* \}_{b=1}^{B}$.
However, the indices $\{ n_b^* \}_{b=1}^{B}$ can be easily recovered from a given solution sequence $\{ \widehat{y}_i \}_{i=1}^{N}$ of PAVA by simply finding all indeces $i$ s.t.~$\widehat{y}_i \ne \widehat{y}_{i+1}$.
Finally, we consider constructing bins $\{ \Delta_b \}_{b=1}^{B}$ based on the recovered indices $\{ n_b^* \}_{b=1}^{B}$.
Recall that the set of labels $\D^y = \{ y_i \}_{i=1}^{N}$ are ordered in ascending order of $\{ P_\theta(x_i) \}_{i=1}^{N}$.
If we construct each bin $\Delta_b$ by 
\begin{align*}
    \Delta_b := \left[ \frac{ P_\theta(x_{n_b^* - 1}) + P_\theta(x_{n_b^*}) }{2}, \frac{ P_\theta(x_{n_{b+1}^* - 1}) + P_\theta(x_{n_{b+1}^*}) }{2} \right] ,
\end{align*}
it is sufficient to generate each label subset $\D_{b}^y$ that corresponds to
\begin{align*}
    \D_{b}^y = \{ y_i \in \D^y \mid P_\theta(x_i) \in \Delta_b \} = \{  y_i  \in \D^y \mid i~~\text{s.t.}~~n_b^* \le i < n_{b+1}^* \} .
\end{align*}
Then the optimisation criterion in \cref{eq:M2E2_partition}, which is translated to the error of the monotonic regression problem of PAVA, is minimised by the choice of bins produced in this procedure.
Observing that \Cref{alg:LSO} with PAVA-BC replaced by PAVA performs this procedure concludes the proof.
\end{proof}

\section{Mild Violation of Monotonicity by PAVA-BC} \label{sec:appendix_b}

A monotonic regression algorithm finds a monotonically increasing sequence $\widehat{y}_1 \le \cdots \le \widehat{y}_N$ that minimises some error $\mathrm{D}( \{ \widehat{y}_i \}_{i=1}^{N}, \{ y_i \}_{i=1}^{N} )$ for a given ordered set $\{ y_i \}_{i=1}^{N}$.
PAVA is one of the most common monotonic regression algorithms that uses the square error $\sum_{i=1}^{N} ( \widehat{y}_i - y_i )^2$.
For some partition $\mathcal{A}$ of indices $I = \{ 1, \dots, N \}$ whose element $A \in \mathcal{A}$ is a set of consequentive indices in $I$, PAVA produces a solution sequence s.t.~each element $\widehat{y}_i$ is given by $\widehat{y}_i = (1 / |A| ) \sum_{i \in A} y_i$ for $A$ in which $i \in A$.
We refer to each element $A$ in the partition $\mathcal{A}$ of indices $I$ as \emph{block}.
PAVA-BC produces a solution sequence that approximates the solution sequence by PAVA under the contraints of the minimum and maximum size of each block.
For some partition $\mathcal{A}'$ of indices $I$, each element $\widehat{y}_i$ of the solution sequence is given by $\widehat{y}_i = (1 / |A'| ) \sum_{i \in A'} y_i$ for $A'$ in which $i \in A'$ in the same manner as PAVA.
PAVA-BC meets the minimum and maximum size constraints of each block $A' \in \mathcal{A}'$ at the cost of the possibility of mild violation of the monotonic constraint.
It depends on the minimum and maximum size constraints, data, and models whether violation of the monotonic constraint occurs by PAVA-BC.
\Cref{fig:section_b1} shows an example where bins based on PAVA-BC did not violate the monotonicity of the empirical probabilities $\{ \widehat{P}_b \}_{b=1}^{B}$.
\Cref{fig:section_b1} was computed using a random forest model trained on the \emph{satimage} dataset used in \Cref{sec:experiment_uci}, and corresponds to \Cref{fig:comparison_bins} presented in \Cref{sec:methodology}.
The total estimation error in \cref{eq:M2E2_partition} and an average of the estimation error within each bin in \cref{eq:D_variance} for each set of the bins in \Cref{fig:section_b1} were summerised in \Cref{tab:comparison_bins} presented in \Cref{sec:methodology}.
\Cref{fig:section_b2} shows an example where bins based on PAVA-BC violated the monotonic constraint of the empirical probabilities $\{ \widehat{P}_b \}_{b=1}^{B}$.
\Cref{fig:section_b2} was computed using a random forest model trained on the \emph{coil\_2000} dataset used in \Cref{sec:experiment_uci}.
The total estimation error in \cref{eq:M2E2_partition} for each set of the bins in \Cref{fig:section_b2} was $0.0509$, $0.0517$, and $ 0.0521$ for PAVA, PAVA-BC, binning based on $10$-quantiles, respectively.
An average of the estimation error within each bin in \cref{eq:D_variance} for each set of the bins in \Cref{fig:section_b1} was $0.0834$, $0.0627$, and $0.0520$ for PAVA, PAVA-BC, binning based on $10$-quantiles, respectively.

\FloatBarrier

\begin{figure}
    \centering
    \begin{subfigure}[t]{0.475\textwidth}
        \centering
        \includegraphics[width=\textwidth]{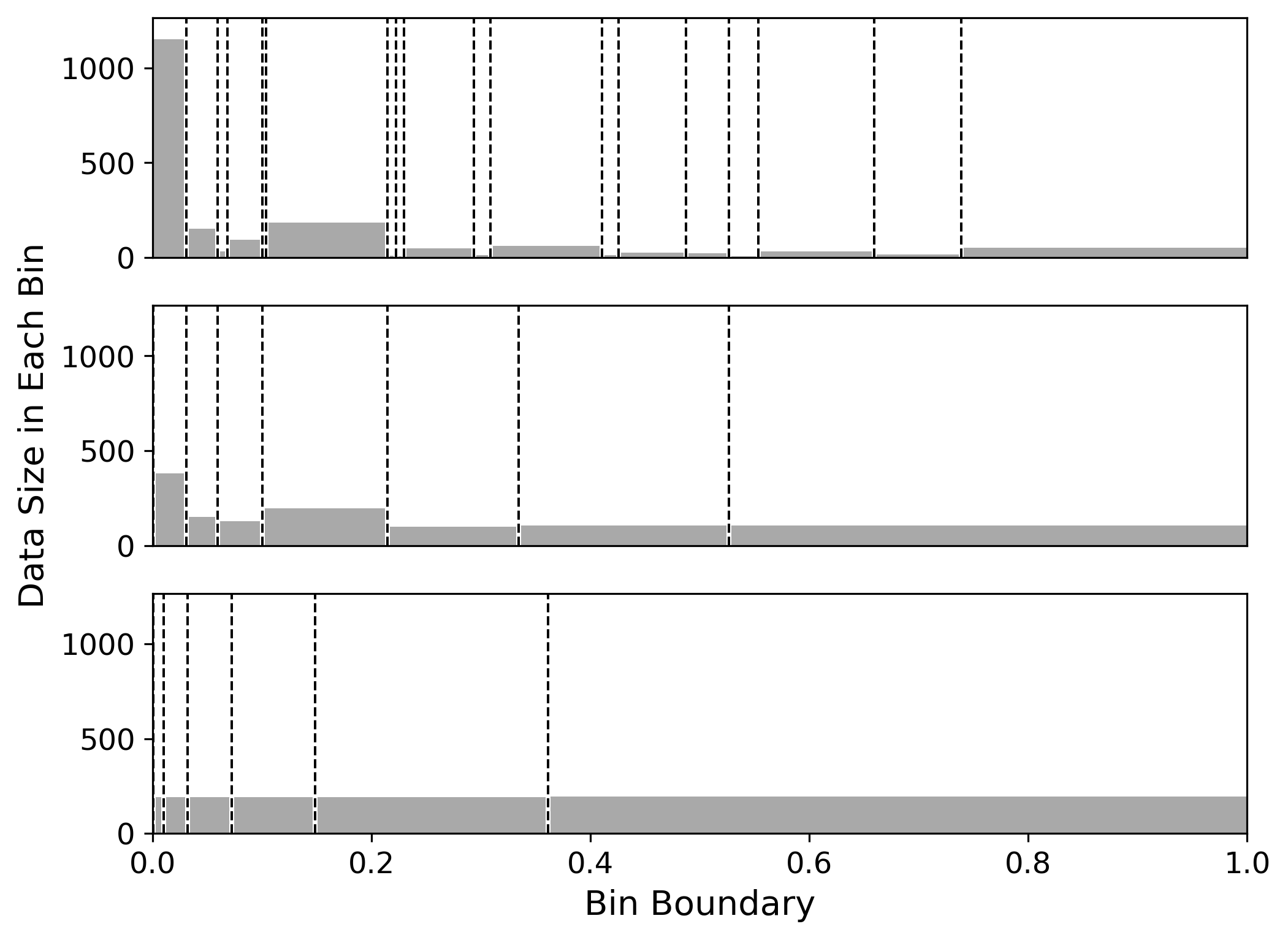}
    \end{subfigure}
    \hfill
    \begin{subfigure}[t]{0.475\textwidth}
        \centering
        \includegraphics[width=\textwidth]{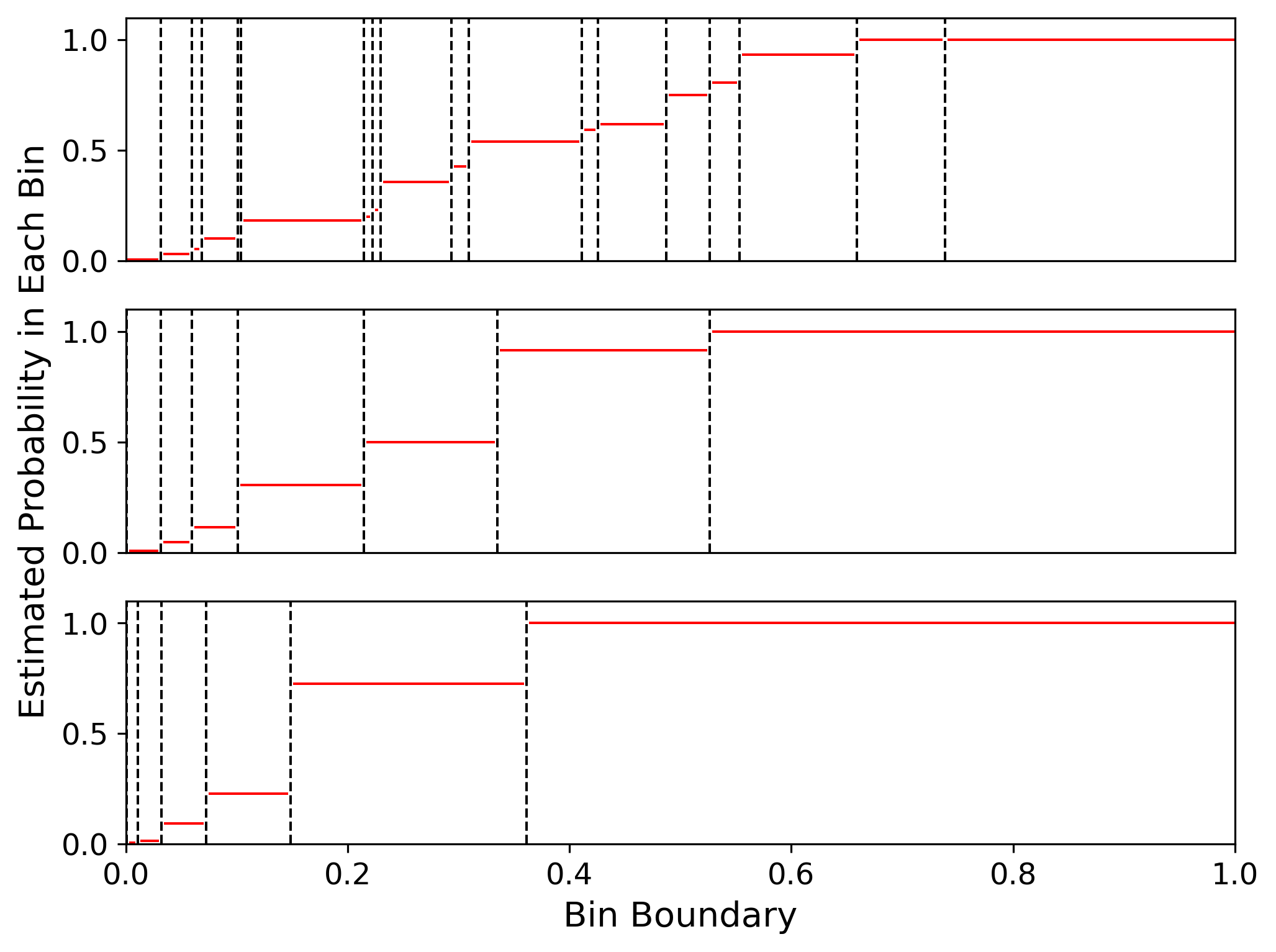}
    \end{subfigure}
    \caption{Comparison of bins based on three different approaches for a random forest model on the satimage dataset: (top) PAVA, (middle) PAVA-BC, (bottom) binning based on $10$-quantiles. The dotted line in the left and right panels represents the boundary of each bin. The grey bar in the left panel repsents the size of each bin. The red line in the right panel repsents the empirical probability of each bin.}
    \label{fig:section_b1}
\end{figure}

\begin{figure}
    \centering
    \begin{subfigure}[t]{0.475\textwidth}
        \centering
        \includegraphics[width=\textwidth]{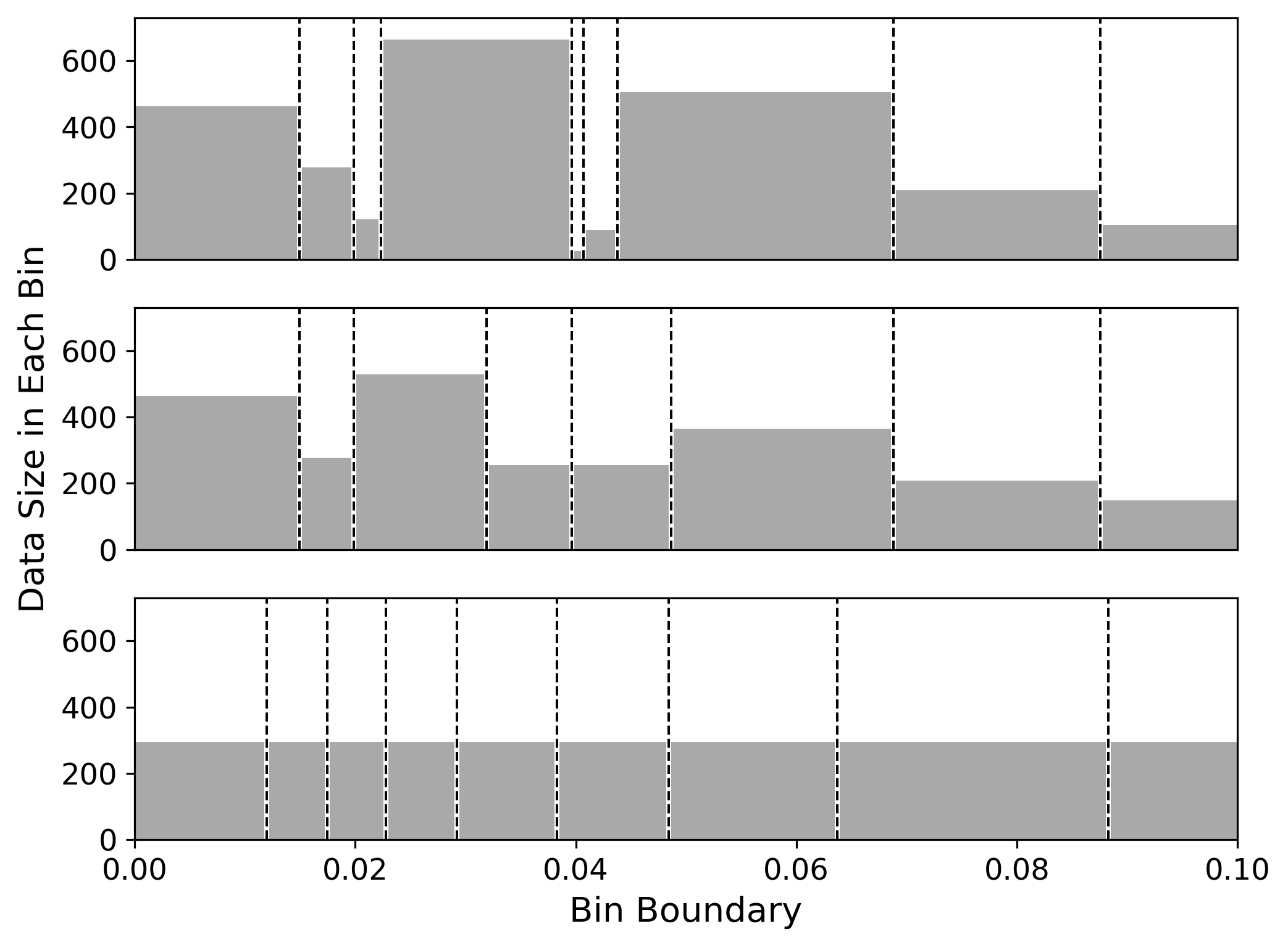}
    \end{subfigure}
    \hfill
    \begin{subfigure}[t]{0.475\textwidth}
        \centering
        \includegraphics[width=\textwidth]{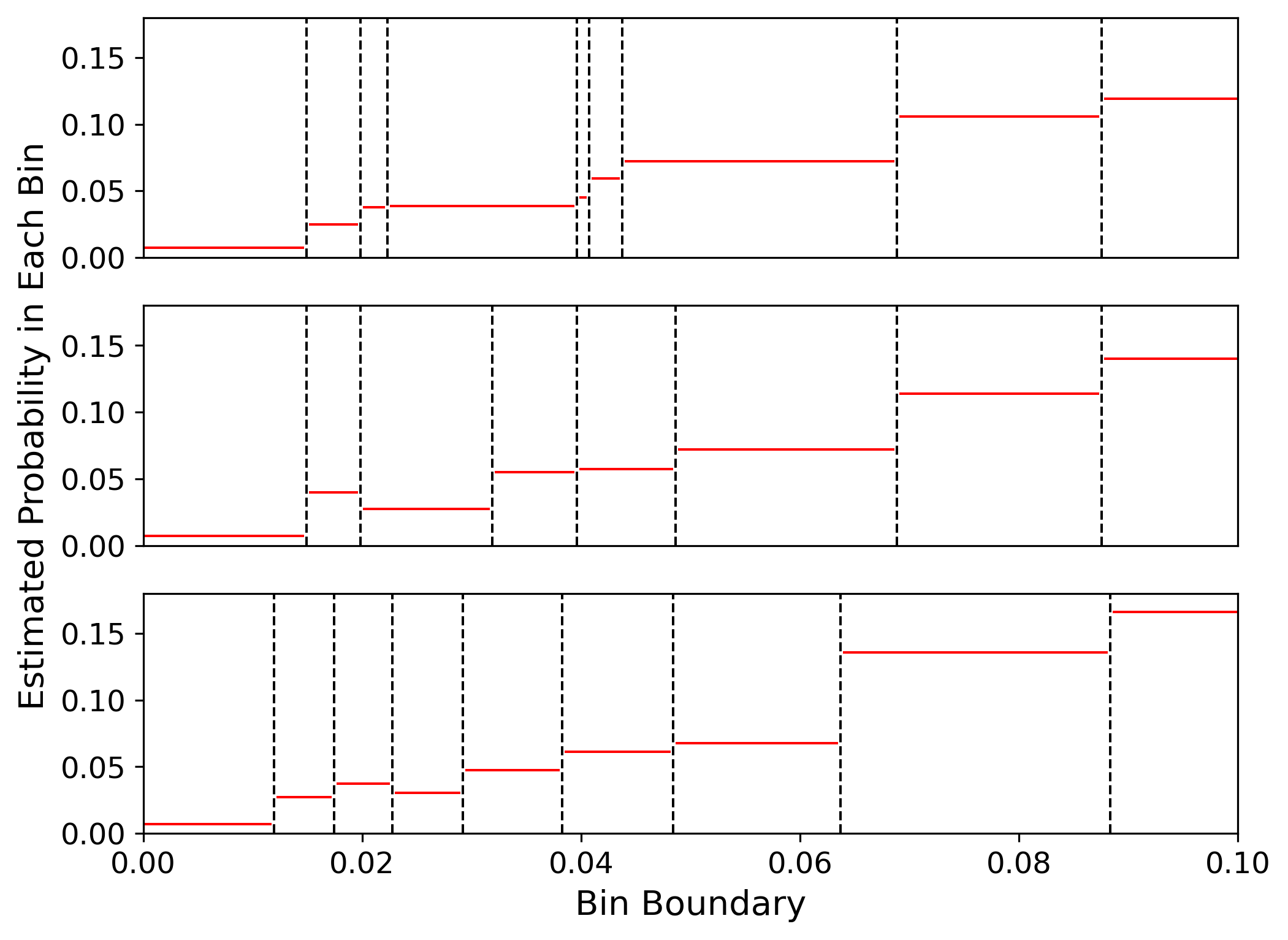}
    \end{subfigure}
    \caption{Comparison of bins based on three different approaches for a random forest model on the satimage dataset: (top) PAVA, (middle) PAVA-BC, (bottom) binning based on $10$-quantiles. Each xaxis is restricted to a range $[0.0, 0.1]$ as the majority of bins were contained in the range in this example. The dotted line in the left and right panels represents the boundary of each bin. The grey bar in the left panel repsents the size of each bin. The red line in the right panel repsents the empirical probability of each bin. A random forest model trained on the coil\_2000 dataset was used.}
    \label{fig:section_b2}
\end{figure}

\FloatBarrier

\section{Additional Experiments} \label{sec:appendix_c}

We present additional experiments in each section that complement the experiments illustrated in the main text.
We use the same settings as the main text for the minimum and maximum size for bins based on PAVA-BC as well as the bin number $B$ for equi-spaced and quantile-based bins.
 
\subsection{Simulation Study of TCE} \label{sec:appendix_c1}

We perform detailed simulation studies of TCE in the same simplified setting as Section 4.1.
We demonstarate sensitivity of TCE to its hyperparameters, an impact of different dataset size and prevalence, and sensitivity to a small purtabation to model predictions.
In all experiments, we generated training and test data from the Gaussian discriminant analysis in Section 4.1, each with the prevalence $P_{\text{training}}(y)$ and $P_{\text{test}}(y)$, and compute TCE of a logistic model fitted to the training data.
In all experiments except ones on an impact of different dataset size and prevalence, we set the training data size to $14000$ and set the test data size to $6000$.
We then examine two cases where the model is calibated and miscalibrated synthetically, setting $P_{\text{training}}(y) = 0.5$ and $P_{\text{test}}(y) = 0.5$ for the first case and setting $P_{\text{training}}(y) = 0.5$ and $P_{\text{test}}(y) = 0.4$ for the second case.
In summary, we present the following experimental analyses:

\begin{itemize}
    \item Sensitivity to the minimum bin size $N_{\text{min}}$ in PAVA-BC from $N_{\text{min}} = 1$ to $N_{\text{min}} = 3000$;
    \item Sensitivity to the maximum bin size $N_{\text{min}}$ in PAVA-BC from $N_{\text{max}} = 6$ to $N_{\text{max}} = 6000$;
    \item Sensitivity to a pair of $(N_{\text{min}}, N_{\text{max}})$ in PAVA-BC chosen so that each binsize fall into selected ranges;
    \item Sensitivity to a small purtabation of predictions by a logit-normal noise with scale $\sigma$ from $\sigma = 0.0$ to $\sigma = 1.0$;
    \item Sensitivity to a choice of significance level $\alpha$ in the Binomial test from $\alpha = 0.0001$ to $\alpha = 0.1$;
    \item Comparison of TCE by different choices of test, binomial test and t-test;
    \item Comparison of TCE by different total sizes $N$ of test dataset from $N = 30$ to $N = 60000$;
    \item Comparison of TCE by different prevalences $P$ of dataset from $P = 0.5$ to $P = 0.02$.
\end{itemize}

\Cref{tab:section_40_1,tab:section_40_2,tab:section_40_3,tab:section_40_4,tab:section_40_5,tab:section_40_6,tab:section_40_7,tab:section_40_8} presents the result of each experiment above in order.
In each table, TCE(P) denotes TCE based on PAVA-BC, TCE(Q) denotes TCE based on quantile-binning, and TCE(V) denotes TCE based on PAVA.
For reference, we include values of ECE, ACE, MCE, and MCE(Q), where MCE(Q) denotes MCE based on quantile-binning.
Observations from each result in are summarised as follows.

\begin{itemize}
    \item \Cref{tab:section_40_1}: The performance of TCE(P) to evidence the well-calibrated model was consistently reasonable for any minimum binsize constaint between $N_{\text{max}}=1$ and $N_{\text{max}}=600$, while there was a breakdown point between $N_{\text{min}}=600$ and $N_{\text{min}}=3000$ where TCE(P) was no longer able to do so. This is likely because the number of bins produced under the contraint $N_{\text{min}}=3000$ for the total datasize $6000$ was $2$ at maximum, which was too small to estimate the empirical probabilities $\{ \widehat{P}_b \}_{b=1}^{B}$ accurately.
    \item \Cref{tab:section_40_2}: The performance of TCE(P) to evidence the miscalibrated model was consistently reasonable for any maximum binsize constaint between $N_{\text{max}}=300$ and $N_{\text{max}}=6000$, while there was a breakdown point between $N_{\text{min}}=60$ and $N_{\text{min}}=300$ where TCE(P) was no longer able to do so. This is likely because the number of bins produced under the contraint $N_{\text{max}}=60$ for the total datasize $6000$ was $100$ at minimum, which is too large to estimate the empirical probabilities $\{ \widehat{P}_b \}_{b=1}^{B}$ accurately.
    \item \Cref{tab:section_40_3}: The performance of TCE(P) to evidence both the well-calibrated and miscalibrated models was arguably the most reasonable when $(N_{\text{min}}, N_{\text{max}})$ was chosen so that the number of bins produced falls into the range $[5, 20]$. This suggests a huristic to use such $(N_{\text{min}}, N_{\text{max}})$ for other experiments.
    \item \Cref{tab:section_40_4}: At each model prediction $P_\theta(x)$, we sample a new prediction from a logit-normal distribution centred at $P_\theta(x)$ with scale $\sigma$ to generate a perturbed prediction by a small noise. All calibration error metrics were shown to have similar sensitivities to the noise. The scale between $\sigma = 0.10$ and $\sigma = 0.50$ was the breakdown point where each metric started to produce an unreasonable score for the well-calibrated model.
    \item \Cref{tab:section_40_5}: The performance of TCE(P) to evidence both the well-calibrated and miscalibrated models was consistently reasonable for any significant level between $\alpha=0.001$ and $\alpha=0.1$, while there was a breakdown point between $\alpha=0.1$ and $\alpha=0.5$ where TCE(P) was no longer able to do so for the well-calibrated model.
    \item \Cref{tab:section_40_6}: TCE based on the Binomial test outperformed one based on the t-test in the majority of the settings. It is possible that the Binomial test produces more accurate outcomes than the t-test, given that it is an exact test whose test statistics does not involve any apporoximation.
    \item \Cref{tab:section_40_7}: The performance of TCE(P) to evidence both the well-calibrated and miscalibrated models was consistently reasonable for any dataset size between $N_{\text{test}}=3000$ and $N_{\text{test}}=60000$, while there was a breakdown point between $N_{\text{test}}=600$ and $N_{\text{test}}=3000$ where TCE(P) was no longer able to do so for the well-calibrated model. This is likely because the dataset size $N_{\text{test}}=600$ was not big enough to estimate the empirical probabilities $\{ \widehat{P}_b \}_{b=1}^{B}$ accurately. This result may be improved by using different settings of the minimnum and maximum binsize constaints.
    \item \Cref{tab:section_40_8}: The performance of TCE(P) on both the well-calibrated and miscalibrated models was reasonable for any prevalence. While there was a fluctuation in values of TCE(P) for different values of prevalence, TCE(P) overall produced better values than TCE(Q and TCE(V).
\end{itemize}

\FloatBarrier
\begin{table}[h]
\caption{Sensitivity to the minimum binsize $N_{\text{min}} = 1, 6, 30, 300, 600, 3000$ in PAVA-BC. For comparison purpose, the number of bins $B$ of quantile-binning and equispaced-binning was varied as $B = 1000, 500, 100, 50, 10, 5, 1$ along with $N_{\text{min}}$. Note that TCE(V) is a constant across all the row because PAVA does not involve any binsize constraint.} \label{tab:section_40_1}
\centering
\begin{tabular}{ccccccccc}
\toprule 
\bfseries Test Prevalence    & \bfseries Min Binsize   & \bfseries TCE(P) & \bfseries TCE(Q) & \bfseries TCE(V) & \bfseries ECE  & \bfseries ACE    & \bfseries MCE    & \bfseries MCE(Q) \\
\midrule
\multirow{7}{*}{\shortstack[c]{50\%\\(Calibrated)}} & 1    & 3.4500  & 5.1000  & 3.4500 & 0.1143 & 0.1651 & 0.8767 & 0.6392 \\
 & 6    & 3.3833  & 4.2000  & 3.4500 & 0.0839 & 0.1142 & 0.8767 & 0.5016 \\
 & 30   & 2.3500  & 4.3000  & 3.4500 & 0.0382 & 0.0457 & 0.8767 & 0.1705 \\
 & 60   & 2.6333  & 3.5667  & 3.4500 & 0.0271 & 0.0370 & 0.2533 & 0.1189 \\
 & 300  & 7.2833  & 10.8833 & 3.4500 & 0.0138 & 0.0150 & 0.1020 & 0.0528 \\
 & 600  & 13.5667 & 38.7500 & 3.4500 & 0.0116 & 0.0086 & 0.1020 & 0.0236 \\
 & 3000 & 92.2000 & 92.2000 & 3.4500 & 0.0021 & 0.0021 & 0.0021 & 0.0021 \\
&   &   &   &   &   &   &   &   \\
\multirow{7}{*}{\shortstack[c]{40\%\\(Miscalibrated)}} & 1    & 88.0667 & 6.6667  & 88.0667 & 0.1417 & 0.1847 & 0.8767 & 0.6111 \\
 & 6    & 88.0667 & 8.7000  & 88.0667 & 0.1179 & 0.1389 & 0.8767 & 0.4811 \\
 & 30   & 88.3333 & 32.2833 & 88.0667 & 0.0993 & 0.0992 & 0.8767 & 0.2264 \\
 & 60   & 87.8667 & 56.7667 & 88.0667 & 0.0971 & 0.0964 & 0.2426 & 0.1827 \\
 & 300  & 96.1000 & 96.4667 & 88.0667 & 0.0963 & 0.0951 & 0.1466 & 0.1314 \\
 & 600  & 96.6000 & 96.7833 & 88.0667 & 0.0963 & 0.0951 & 0.1099 & 0.1092 \\
 & 3000 & 93.9500 & 93.9500 & 88.0667 & 0.0951 & 0.0951 & 0.0951 & 0.0951 \\
\bottomrule
\end{tabular}
\end{table}

\begin{table}[h]
\caption{Sensitivity to the maximum binsize $N_{\text{max}} = 6, 30, 300, 600, 3000, 6000$ in PAVA-BC. For comparison purpose, the number of bins $B$ of quantile-binning and equispaced-binning was varied as $B = 1000, 500, 100, 50, 10, 5, 1$ along with $N_{\text{max}}$. Note that TCE(V) is a constant across all the row because PAVA does not involve any binsize constraint.} \label{tab:section_40_2}
\centering
\begin{tabular}{ccccccccc}
\toprule 
\bfseries Test Prevalence    & \bfseries Max Binsize   & \bfseries TCE(P) & \bfseries TCE(Q) & \bfseries TCE(V) & \bfseries ECE  & \bfseries ACE    & \bfseries MCE    & \bfseries MCE(Q) \\
\midrule
\multirow{7}{*}{\shortstack[c]{50\%\\(Calibrated)}} & 6      & 5.8500 & 5.1000  & 3.4500 & 0.1143 & 0.1651 & 0.8767 & 0.6392 \\
 & 30     & 3.0000 & 4.2000  & 3.4500 & 0.0839 & 0.1142 & 0.8767 & 0.5016 \\
 & 60     & 2.3667 & 4.3000  & 3.4500 & 0.0382 & 0.0457 & 0.8767 & 0.1705 \\
 & 300    & 3.7667 & 3.5667  & 3.4500 & 0.0271 & 0.0370 & 0.2533 & 0.1189 \\
 & 600    & 3.3833 & 10.8833 & 3.4500 & 0.0138 & 0.0150 & 0.1020 & 0.0528 \\
 & 3000   & 3.4500 & 38.7500 & 3.4500 & 0.0116 & 0.0086 & 0.1020 & 0.0236 \\
 & 6000   & 3.4500 & 92.2000 & 3.4500 & 0.0021 & 0.0021 & 0.0021 & 0.0021 \\
&   &   &   &   &   &   &   &   \\
\multirow{7}{*}{\shortstack[c]{40\%\\(Miscalibrated)}} & 6      & 5.5000  & 6.6667  & 88.0667 & 0.1417 & 0.1847 & 0.8767 & 0.6111 \\
 & 30     & 9.1000  & 8.7000  & 88.0667 & 0.1179 & 0.1389 & 0.8767 & 0.4811 \\
 & 60     & 14.3833 & 32.2833 & 88.0667 & 0.0993 & 0.0992 & 0.8767 & 0.2264 \\
 & 300    & 79.6667 & 56.7667 & 88.0667 & 0.0971 & 0.0964 & 0.2426 & 0.1827 \\
 & 600    & 85.6500 & 96.4667 & 88.0667 & 0.0963 & 0.0951 & 0.1466 & 0.1314 \\
 & 3000   & 88.0667 & 96.7833 & 88.0667 & 0.0963 & 0.0951 & 0.1099 & 0.1092 \\
 & 6000   & 88.0667 & 93.9500 & 88.0667 & 0.0951 & 0.0951 & 0.0951 & 0.0951 \\
\bottomrule
\end{tabular}
\end{table}

\begin{table}[h]
\caption{Sensitivity to the pairs $(N_{\text{max}}, N_{\text{min}})$ in PAVA-BC selected so that the number of bins produced falls into ranges $[250, 1000], [50, 200], [25, 100], [10, 20], [3, 10]$. For comparison purpose, the number of bins $B$ of quantile-binning and equispaced-binning was varied as $B = 1000, 500, 100, 50, 10, 5, 1$ along with $(N_{\text{max}}, N_{\text{min}})$. Note that TCE(V) is a constant across all the row because PAVA does not involve any binsize constraint.} \label{tab:section_40_3}
\centering
\begin{tabular}{ccccccccc}
\toprule 
\bfseries Test Prevalence    & \bfseries Binsize Range   & \bfseries TCE(P) & \bfseries TCE(Q) & \bfseries TCE(V) & \bfseries ECE  & \bfseries ACE    & \bfseries MCE    & \bfseries MCE(Q) \\
\midrule
\multirow{5}{*}{\shortstack[c]{50\%\\(Calibrated)}} & [250, 1000] & 3.8000  & 4.2000  & 3.4500  & 0.0839 & 0.1142 & 0.8767 & 0.5016 \\
 & [50, 200] & 1.8333  & 4.3000  & 3.4500  & 0.0382 & 0.0457 & 0.8767 & 0.1705 \\
 & [25, 100] & 0.2833  & 3.5667  & 3.4500  & 0.0271 & 0.0370 & 0.2533 & 0.1189 \\
 & [5, 20] & 7.2833  & 10.8833 & 3.4500  & 0.0138 & 0.0150 & 0.1020 & 0.0528 \\
 & [3, 10] & 13.5667 & 38.7500 & 3.4500  & 0.0116 & 0.0086 & 0.1020 & 0.0236 \\
&   &   &   &   &   &   &   &   \\
\multirow{5}{*}{\shortstack[c]{40\%\\(Miscalibrated)}} & [250, 1000] & 7.7333  & 8.7000  & 88.0667 & 0.1179 & 0.1389 & 0.8767 & 0.4811 \\
 & [50, 200] & 45.7667 & 32.2833 & 88.0667 & 0.0993 & 0.0992 & 0.8767 & 0.2264 \\
 & [25, 100] & 66.1833 & 56.7667 & 88.0667 & 0.0971 & 0.0964 & 0.2426 & 0.1827 \\
 & [10, 20] & 96.1000 & 96.4667 & 88.0667 & 0.0963 & 0.0951 & 0.1466 & 0.1314 \\
 & [3, 10] & 96.6000 & 96.7833 & 88.0667 & 0.0963 & 0.0951 & 0.1099 & 0.1092 \\
\bottomrule
\end{tabular}
\end{table}

\begin{table}[h]
\caption{Sensitivity to a small purtabation to model predictions by a logit-normal noise with scale $\sigma = 0.01, 0.05, 0.10, 0.50, 1.00$. The maximum and minimum binsize of PAVA-BC were set to $1200$ and $300$. The number of bins of quantile-binning and equispaced-binning was set $10$.} \label{tab:section_40_4}
\centering
\begin{tabular}{ccccccccc}
\toprule 
\bfseries Test Prevalence    & \bfseries Noise Level   & \bfseries TCE(P) & \bfseries TCE(Q) & \bfseries TCE(V) & \bfseries ECE  & \bfseries ACE    & \bfseries MCE    & \bfseries MCE(Q) \\
\midrule
\multirow{6}{*}{\shortstack[c]{50\%\\(Calibrated)}} & 0.00    & 7.2833  & 10.8833 & 3.4500  & 0.0138 & 0.0150 & 0.1020 & 0.0528 \\
 & 0.01   & 8.7167  & 9.6167  & 4.8000  & 0.0113 & 0.0125 & 0.0923 & 0.0527 \\
 & 0.05   & 12.8833 & 11.9000 & 7.7667  & 0.0136 & 0.0156 & 0.1198 & 0.0589 \\
 & 0.10    & 8.3500  & 13.0500 & 3.5500  & 0.0109 & 0.0164 & 0.1143 & 0.0587 \\
 & 0.50    & 61.9500 & 65.0500 & 56.1000 & 0.0615 & 0.0618 & 0.3601 & 0.1498 \\
 & 1.00    & 86.1833 & 84.1000 & 88.3833 & 0.1470 & 0.1478 & 0.3364 & 0.2621 \\
&   &   &   &   &   &   &   &   \\
\multirow{6}{*}{\shortstack[c]{40\%\\(Miscalibrated)}}  & 0.00    & 96.1000 & 96.4667 & 88.0667 & 0.0963 & 0.0951 & 0.1466 & 0.1314 \\
 & 0.01   & 96.4000 & 96.4000 & 89.6167 & 0.0962 & 0.0951 & 0.1511 & 0.1332 \\
 & 0.05   & 94.7667 & 95.5333 & 89.1667 & 0.0962 & 0.0951 & 0.1496 & 0.1420 \\
 & 0.10    & 93.8500 & 95.9667 & 86.5833 & 0.0967 & 0.0951 & 0.1852 & 0.1412 \\
 & 0.50    & 86.6667 & 83.9000 & 81.2667 & 0.1071 & 0.1055 & 0.2513 & 0.2203 \\
 & 1.00    & 90.3167 & 88.8500 & 91.2167 & 0.1713 & 0.1698 & 0.4577 & 0.3648 \\
\bottomrule
\end{tabular}
\end{table}

\begin{table}[h]
\caption{Sensitivity to a choice of significance level $\alpha = 0.001, 0.005, 0.01, 0.05, 0.1, 0.5$. The maximum and minimum binsize of PAVA-BC were set to $1200$ and $300$. The number of bins of quantile-binning and equispaced-binning was set $10$.} \label{tab:section_40_5}
\centering
\begin{tabular}{ccccccccc}
\toprule 
\bfseries Test Prevalence    & \bfseries Significant Level   & \bfseries TCE(P) & \bfseries TCE(Q) & \bfseries TCE(V) & \bfseries ECE  & \bfseries ACE    & \bfseries MCE    & \bfseries MCE(Q) \\
\midrule
\multirow{6}{*}{\shortstack[c]{50\%\\(Calibrated)}} & 0.001  & 1.4500  & 4.6833  & 0.1833  & 0.0138 & 0.0150 & 0.1020 & 0.0528 \\
 & 0.005  & 2.4667  & 5.5667  & 1.1333  & 0.0138 & 0.0150 & 0.1020 & 0.0528 \\
 & 0.010   & 3.0500  & 6.2000  & 1.6500  & 0.0138 & 0.0150 & 0.1020 & 0.0528 \\
 & 0.050   & 7.2833  & 10.8833 & 3.4500  & 0.0138 & 0.0150 & 0.1020 & 0.0528 \\
 & 0.100    & 12.8500 & 15.2000 & 6.8667  & 0.0138 & 0.0150 & 0.1020 & 0.0528 \\
 & 0.500    & 53.1000 & 55.3333 & 46.5667 & 0.0138 & 0.0150 & 0.1020 & 0.0528 \\
&   &   &   &   &   &   &   &   \\
\multirow{6}{*}{\shortstack[c]{40\%\\(Miscalibrated)}}  & 0.001  & 77.8000 & 83.3833 & 76.1000 & 0.0963 & 0.0951 & 0.1466 & 0.1314 \\
 & 0.005 & 86.3000 & 92.8000 & 80.0833 & 0.0963 & 0.0951 & 0.1466 & 0.1314 \\
 & 0.010 & 90.1833 & 95.2167 & 83.1167 & 0.0963 & 0.0951 & 0.1466 & 0.1314 \\
 & 0.050 & 96.1000 & 96.4667 & 88.0667 & 0.0963 & 0.0951 & 0.1466 & 0.1314 \\
 & 0.100 & 97.2167 & 96.9167 & 90.1667 & 0.0963 & 0.0951 & 0.1466 & 0.1314 \\
 & 0.500 & 99.3000 & 98.7167 & 97.7500 & 0.0963 & 0.0951 & 0.1466 & 0.1314 \\
\bottomrule
\end{tabular}
\end{table}

\begin{table}[h]
\caption{Comparison of TCE based on the Binomial test and the t-test. TCE(Q)-B denotes TCE(Q) based on the Binomial test and TCE(Q)-T denotes TCE(Q) based on the t-test; the same applies for the other columns. The maximum and minimum binsize of PAVA-BC and the number of bins of quantile-binning and equispaced-binning were varied as in \Cref{tab:section_40_3}.} \label{tab:section_40_6}
\centering
\begin{tabular}{cccccccc}
\toprule 
\bfseries Test Prevalence    & \bfseries Binsize Range   & \bfseries TCE(P)-B & \bfseries TCE(P)-T & \bfseries TCE(Q)-B & \bfseries TCE(Q)-T  & \bfseries TCE(V)-B    & \bfseries TCE(V)-T \\
\midrule
\multirow{5}{*}{\shortstack[c]{50\%\\(Calibrated)}} & [250, 1000]      & 3.8000  & 33.6667 & 4.2000  & 31.9167 & 3.4500 & 34.2167 \\
 & [50, 200]     & 1.8333  & 36.0000 & 4.3000  & 31.4333 & 3.4500 & 34.2167 \\
 & [25, 100]      & 0.2833  & 31.3667 & 3.5667  & 40.4333 & 3.4500 & 34.2167 \\
 & [5, 20]    & 7.2833  & 37.8000 & 10.8833 & 41.8500 & 3.4500 & 34.2167 \\
 & [3, 10]    & 13.5667 & 46.5000 & 38.7500 & 68.8167 & 3.4500 & 34.2167 \\
&   &   &   &   &   &   &   \\
\multirow{5}{*}{\shortstack[c]{40\%\\(Miscalibrated)}} & [250, 1000] & 7.7333  & 50.2833 & 8.7000  & 45.1833 & 88.0667 & 97.7333 \\
 & [50, 200] & 45.7667 & 73.2667 & 32.2833 & 71.1667 & 88.0667 & 97.7333 \\
 & [25, 100] & 66.1833 & 96.5333 & 56.7667 & 85.3833 & 88.0667 & 97.7333 \\
 & [5, 20] & 96.1000 & 99.2667 & 96.4667 & 98.4833 & 88.0667 & 97.7333 \\
 & [3, 10] & 96.6000 & 98.6333 & 96.7833 & 98.4833 & 88.0667 & 97.7333 \\
\bottomrule
\end{tabular}
\end{table}

\begin{table}[h]
\caption{Comparison of TCE by different total sizes $N_{\text{test}} = 30, 60, 300, 600, 3000, 6000, 30000, 60000$ of test dataset. The training prevalence was $P_{\text{training}}(y) = 0.5$ for all datasets. The maximum and minimum binsize of PAVA-BC were set by $N_{\text{max}} = N_{\text{test}} / 20$ and $N_{\text{min}} = N_{\text{test}} / 5$. The number of bins of quantile-binning and equispaced-binning was set $10$.} \label{tab:section_40_7}
\centering
\begin{tabular}{ccccccccc}
\toprule 
\bfseries Test Prevalence    & \bfseries Data Size   & \bfseries TCE(P) & \bfseries TCE(Q) & \bfseries TCE(V) & \bfseries ECE  & \bfseries ACE    & \bfseries MCE    & \bfseries MCE(Q) \\
\midrule
\multirow{8}{*}{\shortstack[c]{50\%\\(Calibrated)}} & 30    & 0.0000  & 0.0000  & 0.0000 & 0.2293 & 0.2631 & 0.4164 & 0.5660 \\
 & 60    & 0.0000  & 3.3333  & 0.0000 & 0.0923 & 0.2158 & 0.7148 & 0.4208 \\
 & 300   & 5.3333  & 11.0000 & 6.3333 & 0.0774 & 0.0867 & 0.1971 & 0.2057 \\
 & 600   & 1.0000  & 4.5000  & 1.6667 & 0.0368 & 0.0445 & 0.3404 & 0.1270 \\
 & 3000  & 8.0667  & 4.6333  & 4.7667 & 0.0190 & 0.0182 & 0.1209 & 0.0304 \\
 & 6000  & 7.2833  & 10.8833 & 3.4500 & 0.0138 & 0.0150 & 0.1020 & 0.0528 \\
 & 30000 & 16.1633 & 31.7167 & 0.7833 & 0.0036 & 0.0061 & 0.9045 & 0.0164 \\
 & 60000 & 19.1483 & 45.7600 & 4.4417 & 0.0035 & 0.0043 & 0.0949 & 0.0100 \\
&   &   &   &   &   &   &   &   \\
\multirow{8}{*}{\shortstack[c]{40\%\\(Miscalibrated)}} & 30    & 13.3333 & 6.6667  & 36.6667 & 0.3164 & 0.3377 & 0.6569 & 0.6338 \\
 & 60    & 0.0000  & 3.3333  & 0.0000  & 0.1072 & 0.1611 & 0.7148 & 0.4208 \\
 & 300   & 27.3333 & 37.3333 & 48.3333 & 0.1240 & 0.1368 & 0.1971 & 0.2665 \\
 & 600   & 14.1667 & 8.0000  & 26.5000 & 0.0694 & 0.0685 & 0.5824 & 0.1350 \\
 & 3000  & 92.2333 & 91.7667 & 76.7667 & 0.0964 & 0.0958 & 0.1495 & 0.1358 \\
 & 6000  & 96.1000 & 96.4667 & 88.0667 & 0.0963 & 0.0951 & 0.1466 & 0.1314 \\
 & 30000 & 99.4700 & 99.2300 & 97.4433 & 0.0907 & 0.0906 & 0.9045 & 0.1064 \\
 & 60000 & 99.7783 & 99.6600 & 98.9000 & 0.0923 & 0.0923 & 0.0972 & 0.1065 \\
\bottomrule
\end{tabular}
\end{table}

\begin{table}[h]
\caption{Comparison of TCE by different prevalences $P$ of training and test dataset. The training data size was $14000$ and the test data size was $6000$. The maximum and minimum binsize of PAVA-BC were set to $1200$ and $300$. The number of bins of quantile-binning and equispaced-binning was set $10$.} \label{tab:section_40_8}
\centering
\begin{tabular}{ccccccccc}
\toprule 
\multicolumn{2}{c}{\bfseries Train - Test Prevalence} & \bfseries TCE(P) & \bfseries TCE(Q) & \bfseries TCE(V) & \bfseries ECE  & \bfseries ACE    & \bfseries MCE    & \bfseries MCE(Q) \\
\midrule
\multirow{9}{*}{Calibrated} & 50\% - 50\%  & 7.2833  & 10.8833 & 3.4500  & 0.0138 & 0.0150 & 0.1020 & 0.0528 \\
 & 40\% - 40\% & 7.5500  & 16.2167 & 8.5667  & 0.0137 & 0.0191 & 0.1632 & 0.0365 \\
 & 30\% - 30\% & 8.1667  & 12.8833 & 2.8167  & 0.0125 & 0.0134 & 0.1042 & 0.0313 \\
 & 20\% - 20\% & 15.9500 & 22.2167 & 15.9167 & 0.0173 & 0.0153 & 0.6238 & 0.0370 \\
 & 10\% - 10\% & 11.9833 & 16.7833 & 15.2333 & 0.0096 & 0.0114 & 0.4361 & 0.0218 \\
 & 8\% - 8\% & 15.7000 & 18.5167 & 23.1500 & 0.0087 & 0.0107 & 0.0700 & 0.0234 \\
 & 6\% - 6\% & 11.5333 & 17.5500 & 13.9833 & 0.0035 & 0.0109 & 0.3064 & 0.0195 \\
 & 4\% - 4\% & 18.5000 & 15.6667 & 20.5833 & 0.0046 & 0.0074 & 0.2240 & 0.0177 \\
 & 2\% - 2\% & 13.1167 & 11.5500 & 20.7667 & 0.0052 & 0.0059 & 0.0052 & 0.0131 \\
&   &   &   &   &   &   &   &   \\
\multirow{9}{*}{Miscalibrated} & 50\% - 40\% & 96.1000 & 96.4667 & 88.0667  & 0.0963 & 0.0951 & 0.1466 & 0.1314 \\
 & 40\% - 30\% & 96.5667 & 96.1833 & 82.7500  & 0.0872 & 0.0869 & 0.1485 & 0.1262 \\
 & 30\% - 20\% & 94.9500 & 94.6667 & 88.5833  & 0.0846 & 0.0846 & 0.2146 & 0.1247 \\
 & 20\% - 10\% & 95.8833 & 95.5833 & 96.4333  & 0.0868 & 0.0868 & 0.6238 & 0.1500 \\
 & 10\% - 8\% & 32.3500 & 26.7000 & 42.5667  & 0.0151 & 0.0173 & 0.4361 & 0.0502 \\
 & 8\% - 6\% & 42.3167 & 38.7833 & 45.8333  & 0.0164 & 0.0186 & 0.3259 & 0.0477 \\
 & 6\% - 4\% & 47.0833 & 39.9500 & 65.9500  & 0.0167 & 0.0188 & 0.3064 & 0.0440 \\
 & 4\% - 2\% & 56.5833 & 42.4333 & 72.4500  & 0.0142 & 0.0142 & 0.2240 & 0.0337 \\
 & 2\% - 0\% & 99.9167 & 96.9000 & 100.0000 & 0.0181 & 0.0181 & 0.0181 & 0.0382 \\
\bottomrule
\end{tabular}
\end{table}
\FloatBarrier

\subsection{Results on Other UCI Datasets} \label{sec:appendix_c2}

Algorithms in Section 4.2 are all trained with the default hyperparameters in the scikit-learn package, except that the maximum depth in the random forest is set to 10 and the number of hidden layers in the multiple perceptron is set to 1 with 1000 units.
For better comparison, we add TCE based on quantile bins, denoted TCE(Q) in each table, to five metrics presented in the main text.
The following \Cref{tab:section_42} compares six different calibration error metrics computed for eight UCI datasets that were not presented in the main text: coil\_2000, isolet, letter\_img, mammography, optimal\_degits, pen\_degits, satimage, spambase \citep{Dua2017,Putten2000,Elter2007}. 
The prevalence of the spambase dataset is well-balanced and that of the rest is imbalanced.
The following \Cref{fig:section_42_1,fig:section_42_2} shows the visual representations of TCE, ECE, and ACE---the test-based reliability diagram and the standard reliability diagram---each for the logistic regression and the gradient boosting algorithm.
We selected four datasets, abalone, coil\_2000, isolet, and webpage, to produce the visual representations in \Cref{fig:section_42_1,fig:section_42_2}.

\subsection{Reliability Diagrams of Results on ImageNet1000} \label{sec:appendix_c3}
	
The following \Cref{fig:section_43_1} shows the viaual representations of TCE, ECE, and ACE---the test-based reliability diagram and the standard reliability diagram---for four different deep learning models presented in the main text, where we omit the model ResNet50 whose result sufficiently resembles that of ResNet18.

\begin{table}[h]
\caption{Comparison of six calibration error metrics for five algorithms trained on eight UCI datasets. The same setting of TCE presented in \Cref{sec:experiment} is used. TCE(Q) and MCE(Q) denotes TCE and MCE each based on quantile bins where the number of bins is set to $10$. } \label{tab:section_42}
\centering
\begin{tabular}{cccccccc}
\toprule
\bfseries Data & \bfseries Algorithm & \bfseries TCE    & \bfseries TCE(Q)     & \bfseries ECE    & \bfseries ACE    & \bfseries MCE    & \bfseries MCE(Q)  \\
\midrule
\multirow{5}{*}{coil\_2000} & LR  & 8.6189  & 11.7408 & 0.0047 & 0.0111 & 0.8558 & 0.0326 \\
 & SVM & 17.0003 & 28.9447 & 0.0071 & 0.0216 & 0.4860 & 0.0381 \\
 & RF  & 22.2260 & 6.1758  & 0.0027 & 0.0125 & 0.2465 & 0.0439 \\
 & GB  & 20.8687 & 12.6569 & 0.0052 & 0.0098 & 0.3738 & 0.0259 \\
 & MLP & 98.7445 & 98.7784 & 0.0652 & 0.0578 & 0.7900 & 0.1649 \\
 &     &        &        &        &        &        \\
\multirow{5}{*}{isolet} & LR  & 28.8462 & 27.3932 & 0.0131 & 0.0051 & 0.2183 & 0.0286 \\
 & SVM & 11.5812 & 13.2479 & 0.0064 & 0.0028 & 0.1969 & 0.0194 \\
 & RF  & 66.4530 & 52.5214 & 0.0524 & 0.0507 & 0.3635 & 0.2137 \\
 & GB  & 25.2991 & 16.6667 & 0.0198 & 0.0174 & 0.4463 & 0.1123 \\
 & MLP & 9.8291  & 17.5641 & 0.0049 & 0.0031 & 0.4173 & 0.0232 \\
 &     &        &        &       &        &        \\
\multirow{5}{*}{letter\_img} & LR  & 10.5167 & 12.0500 & 0.0025 & 0.0008 & 0.1617 & 0.0042 \\
 & SVM & 11.8667 & 14.8167 & 0.0019 & 0.0017 & 0.6257 & 0.0146 \\
 & RF  & 26.5000 & 20.2500 & 0.0097 & 0.0033 & 0.5179 & 0.0131 \\
 & GB  & 25.9500 & 18.7333 & 0.0067 & 0.0029 & 0.3653 & 0.0109 \\
 & MLP & 19.9833 & 9.9833  & 0.0010 & 0.0001 & 0.4550 & 0.0007 \\
 &     &        &        &        &        &        \\
\multirow{5}{*}{mammography} & LR  & 25.0671 & 26.7660 & 0.0027 & 0.0065 & 0.3594 & 0.0208 \\
 & SVM & 20.2683 & 20.1490 & 0.0067 & 0.0088 & 0.6741 & 0.0353 \\
 & RF  & 19.4039 & 9.2996  & 0.0047 & 0.0016 & 0.4465 & 0.0043 \\
 & GB  & 14.5156 & 15.4098 & 0.0061 & 0.0034 & 0.5355 & 0.0124 \\
 & MLP & 20.5663 & 26.9747 & 0.0042 & 0.0027 & 0.4351 & 0.0113 \\
 &     &        &        &        &        &        \\
\multirow{5}{*}{optical\_digits}  & LR  & 11.6251 & 27.1649 & 0.0098 & 0.0037 & 0.2251 & 0.0135 \\
 & SVM & 4.8043  & 10.6762 & 0.0042 & 0.0028 & 0.6608 & 0.0157 \\
 & RF  & 49.6441 & 38.3155 & 0.0451 & 0.0433 & 0.5432 & 0.2271 \\
 & GB  & 13.0486 & 11.2693 & 0.0181 & 0.0168 & 0.5639 & 0.1122 \\
 & MLP & 4.6856  & 12.1590 & 0.0037 & 0.0034 & 0.5992 & 0.0306 \\
 &     &        &        &        &        &        \\
\multirow{5}{*}{pen\_digits}  & LR  & 20.4063 & 23.1049 & 0.0121 & 0.0060 & 0.1652 & 0.0252 \\
 & SVM & 9.7635  & 10.2790 & 0.0017 & 0.0010 & 0.4735 & 0.0068 \\
 & RF  & 29.6240 & 22.8623 & 0.0152 & 0.0132 & 0.4535 & 0.0592 \\
 & GB  & 9.9151  & 13.0988 & 0.0077 & 0.0058 & 0.6543 & 0.0303 \\
 & MLP & 9.4603  & 10.0061 & 0.0014 & 0.0004 & 0.6457 & 0.0037 \\
 &     &        &        &        &        &        \\
\multirow{5}{*}{satimage}  & LR  & 23.6665 & 23.0968 & 0.0215 & 0.0223 & 0.7312 & 0.0767 \\
 & SVM & 10.2020 & 21.8540 & 0.0229 & 0.0163 & 0.1666 & 0.0870 \\
 & RF  & 29.1041 & 20.1450 & 0.0265 & 0.0214 & 0.2084 & 0.1328 \\
 & GB  & 23.2004 & 19.8861 & 0.0154 & 0.0235 & 0.2101 & 0.0902 \\
 & MLP & 58.0528 & 58.4671 & 0.0352 & 0.0328 & 0.5049 & 0.1384 \\
&     &        &        &        &        &        \\
\multirow{5}{*}{spambase}  & LR  & 33.6713 & 56.1188 & 0.0256 & 0.0267 & 0.1539 & 0.0895 \\
 & SVM & 12.8168 & 34.5402 & 0.0177 & 0.0227 & 0.2207 & 0.0465 \\
 & RF  & 66.0391 & 49.4569 & 0.0635 & 0.0601 & 0.2056 & 0.1616 \\
 & GB  & 20.2028 & 20.4200 & 0.0295 & 0.0277 & 0.1409 & 0.0891 \\
 & MLP & 60.9703 & 67.1253 & 0.0413 & 0.0397 & 0.2931 & 0.1076 \\
\bottomrule 
\end{tabular}
\end{table}

\begin{figure}[h]
    \centering
    
    \begin{subfigure}[b]{\textwidth}
    \centering
    \includegraphics[trim={0 10pt 0 5pt},clip,width=0.32\columnwidth]{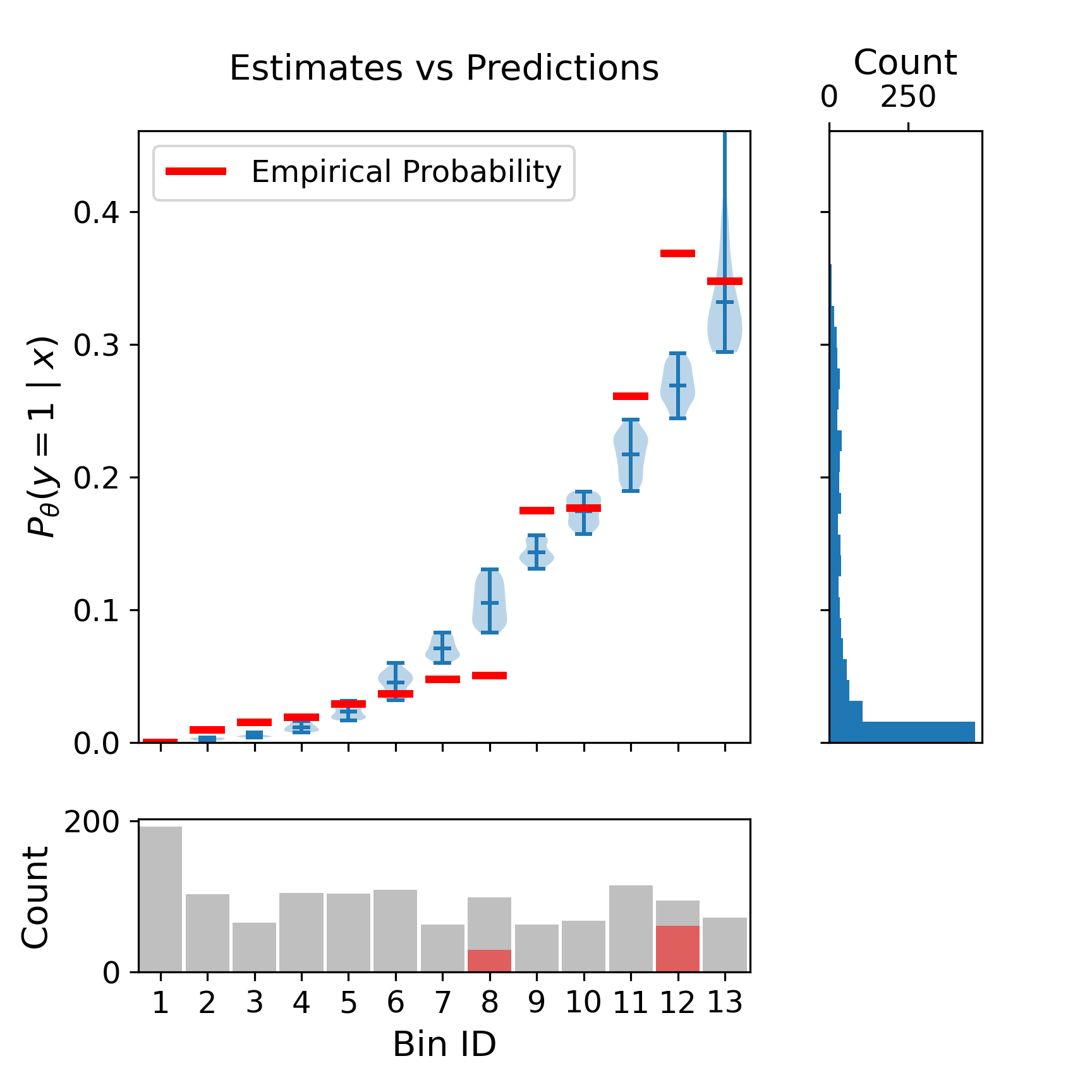}
    \hfill
    \includegraphics[trim={0 10pt 0 5pt},clip,width=0.32\columnwidth]{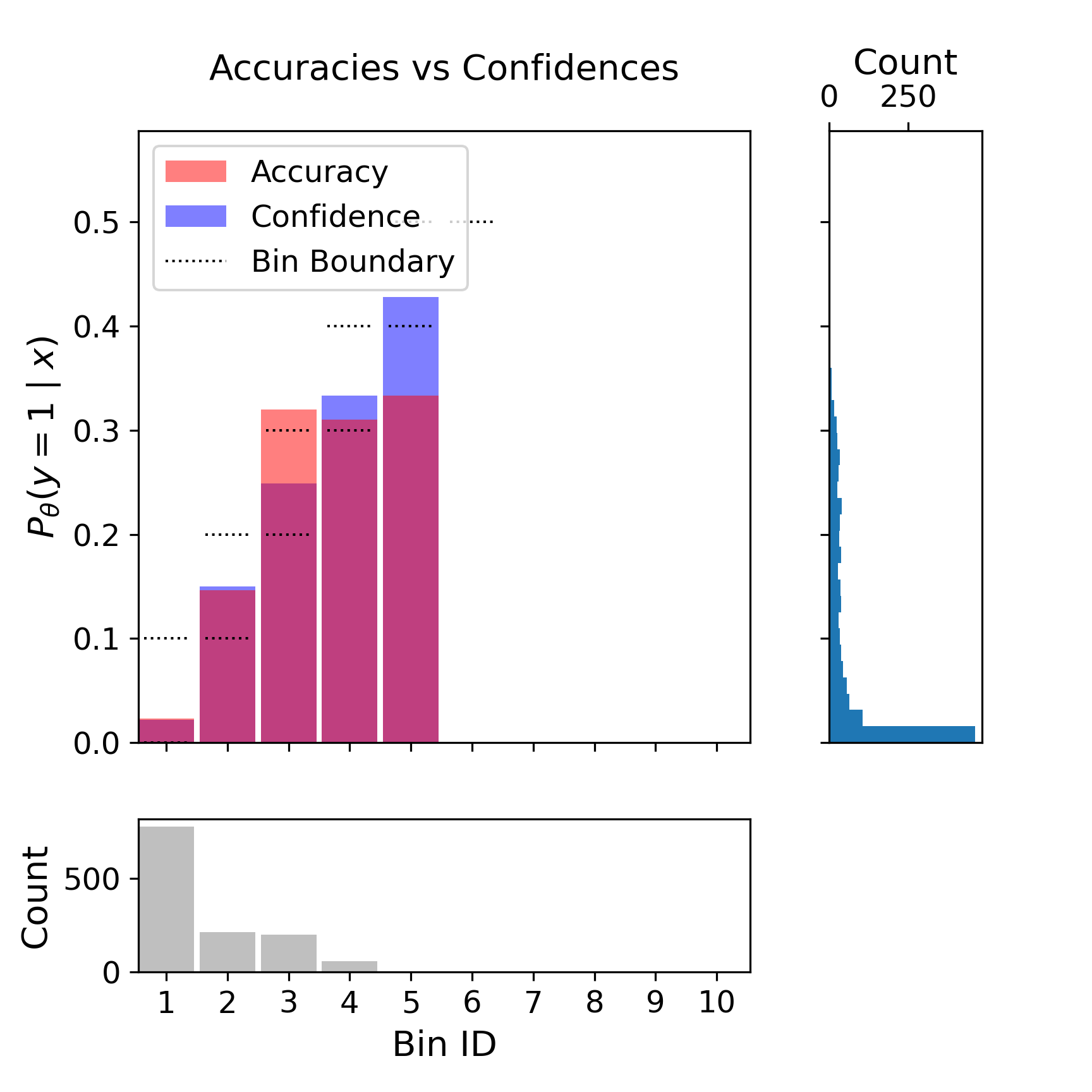}
    \hfill
    \includegraphics[trim={0 10pt 0 5pt},clip,width=0.32\columnwidth]{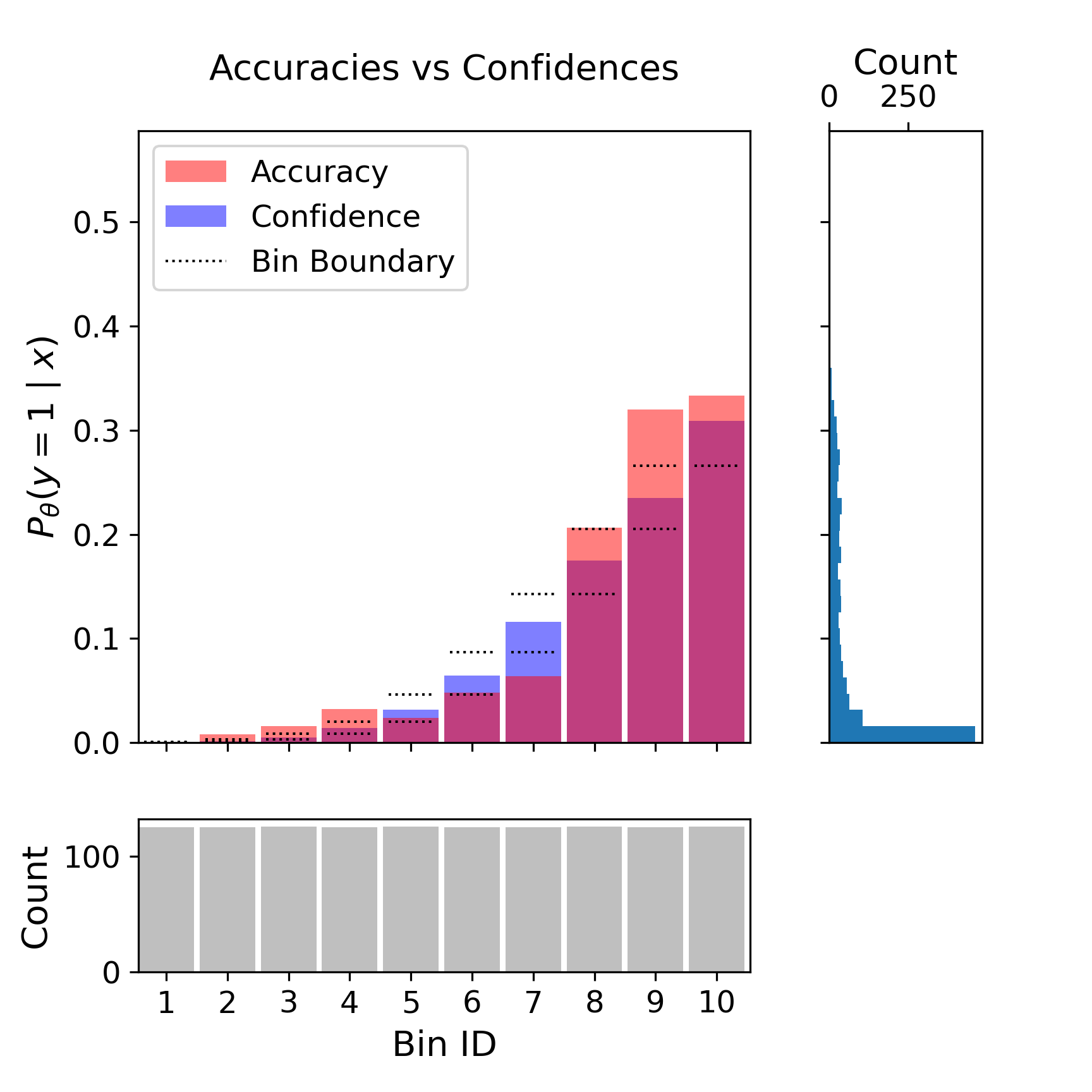}
    \caption{abalone}
    \end{subfigure}
    
    \begin{subfigure}[b]{\textwidth}
    \centering
    \includegraphics[trim={0 10pt 0 5pt},clip,width=0.32\columnwidth]{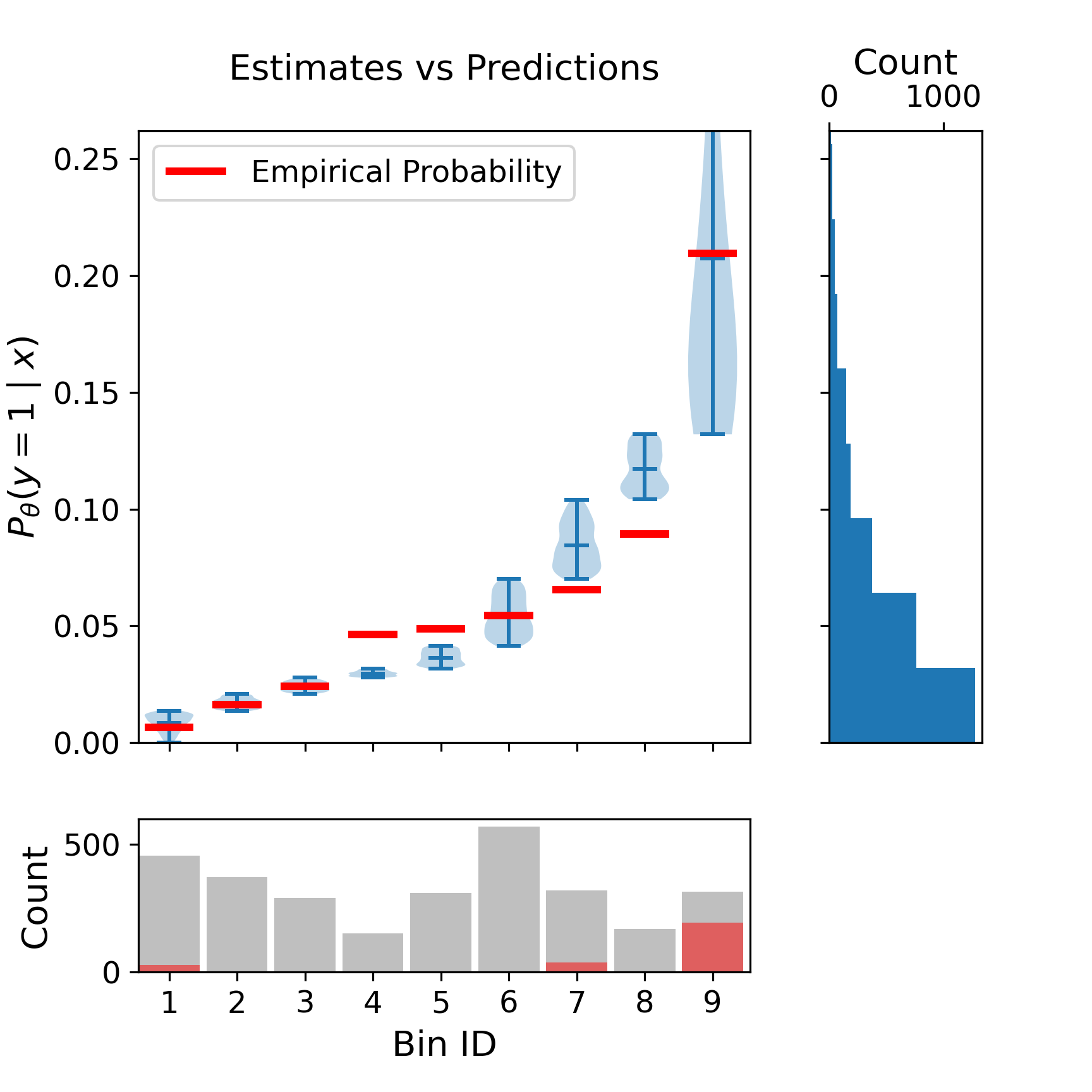}
    \hfill
    \includegraphics[trim={0 10pt 0 5pt},clip,width=0.32\columnwidth]{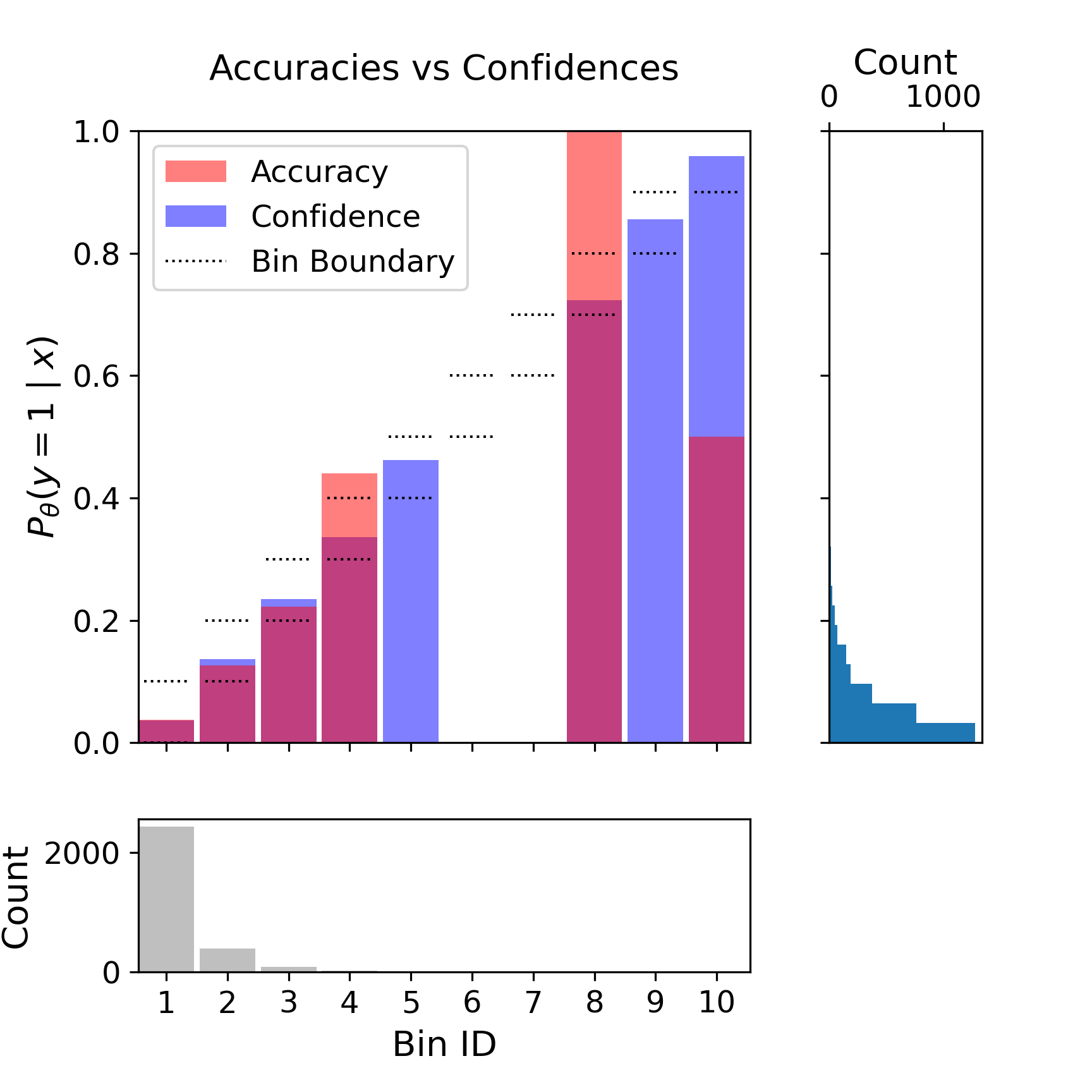}
    \hfill
    \includegraphics[trim={0 10pt 0 5pt},clip,width=0.32\columnwidth]{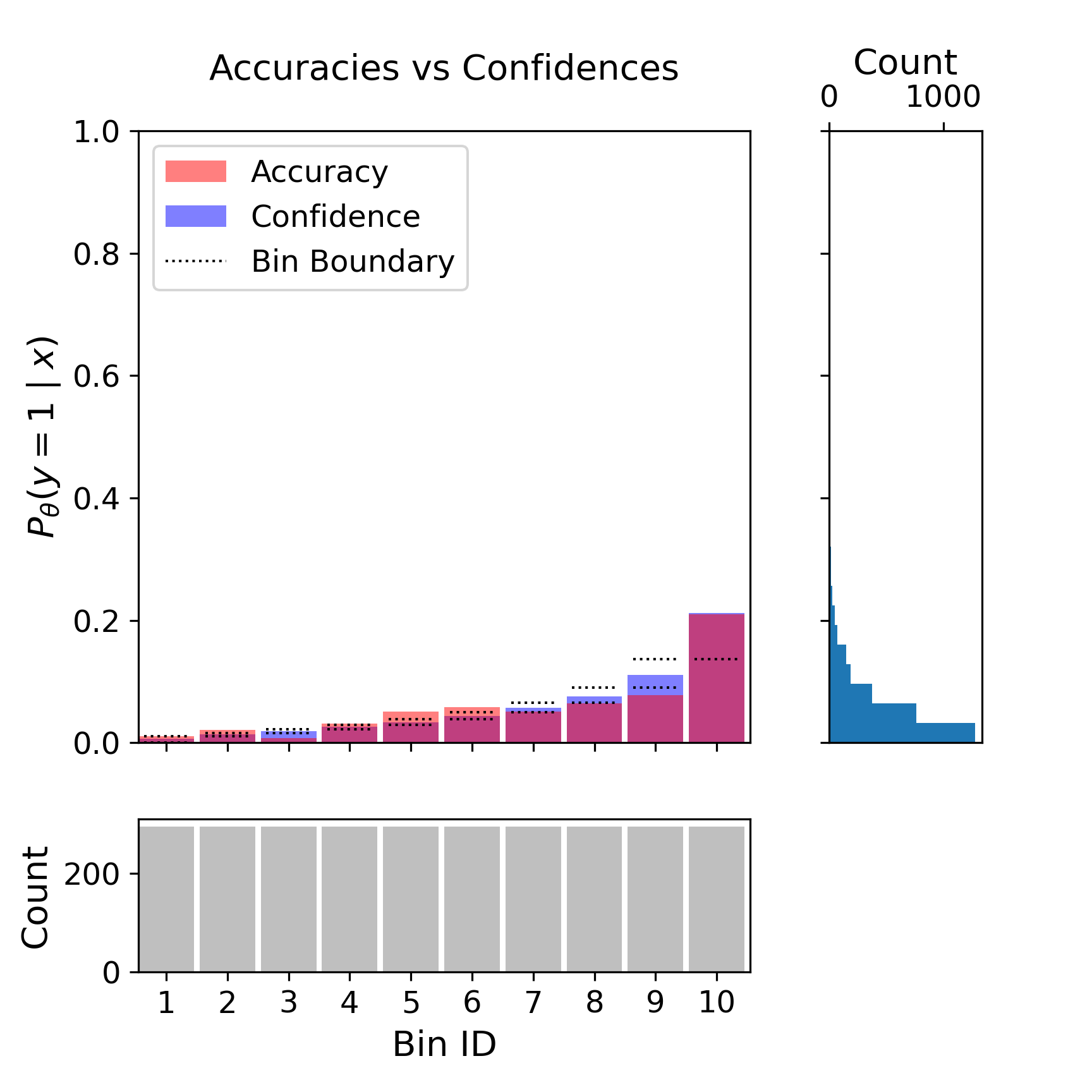}
    \caption{coil\_2000}
    \end{subfigure}
    
    \begin{subfigure}[b]{\textwidth}
    \centering
    \includegraphics[trim={0 10pt 0 5pt},clip,width=0.32\columnwidth]{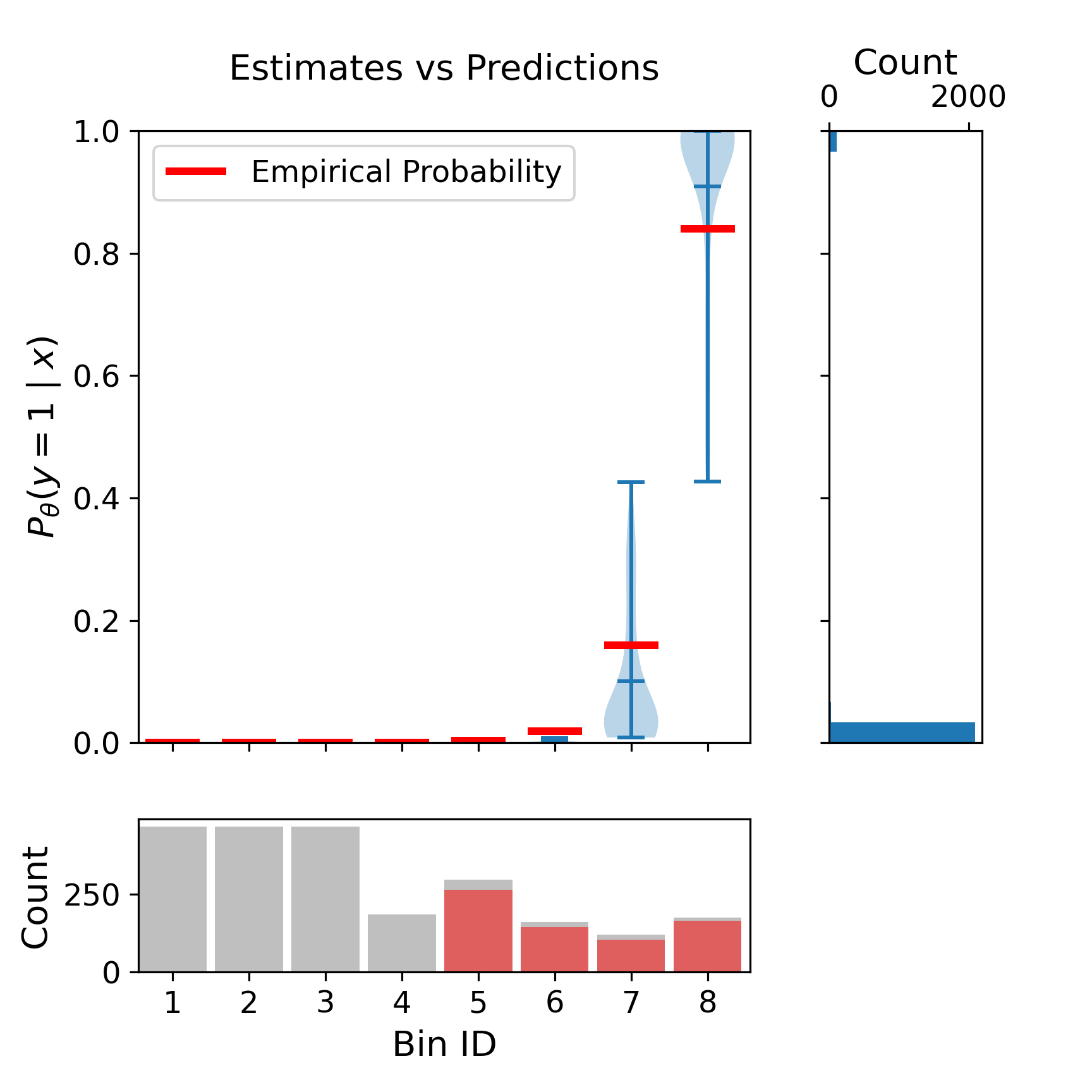}
    \hfill
    \includegraphics[trim={0 10pt 0 5pt},clip,width=0.32\columnwidth]{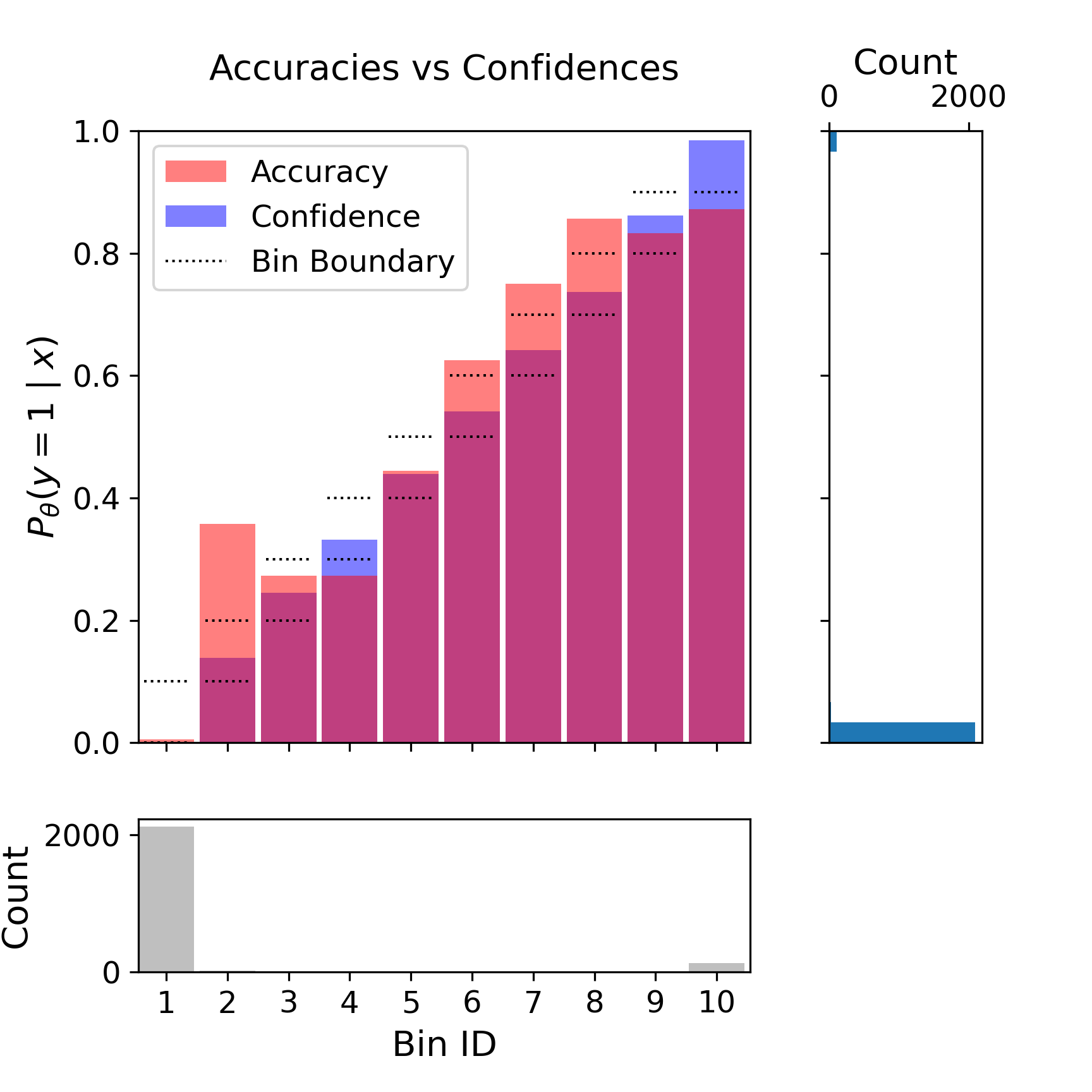}
    \hfill
    \includegraphics[trim={0 10pt 0 5pt},clip,width=0.32\columnwidth]{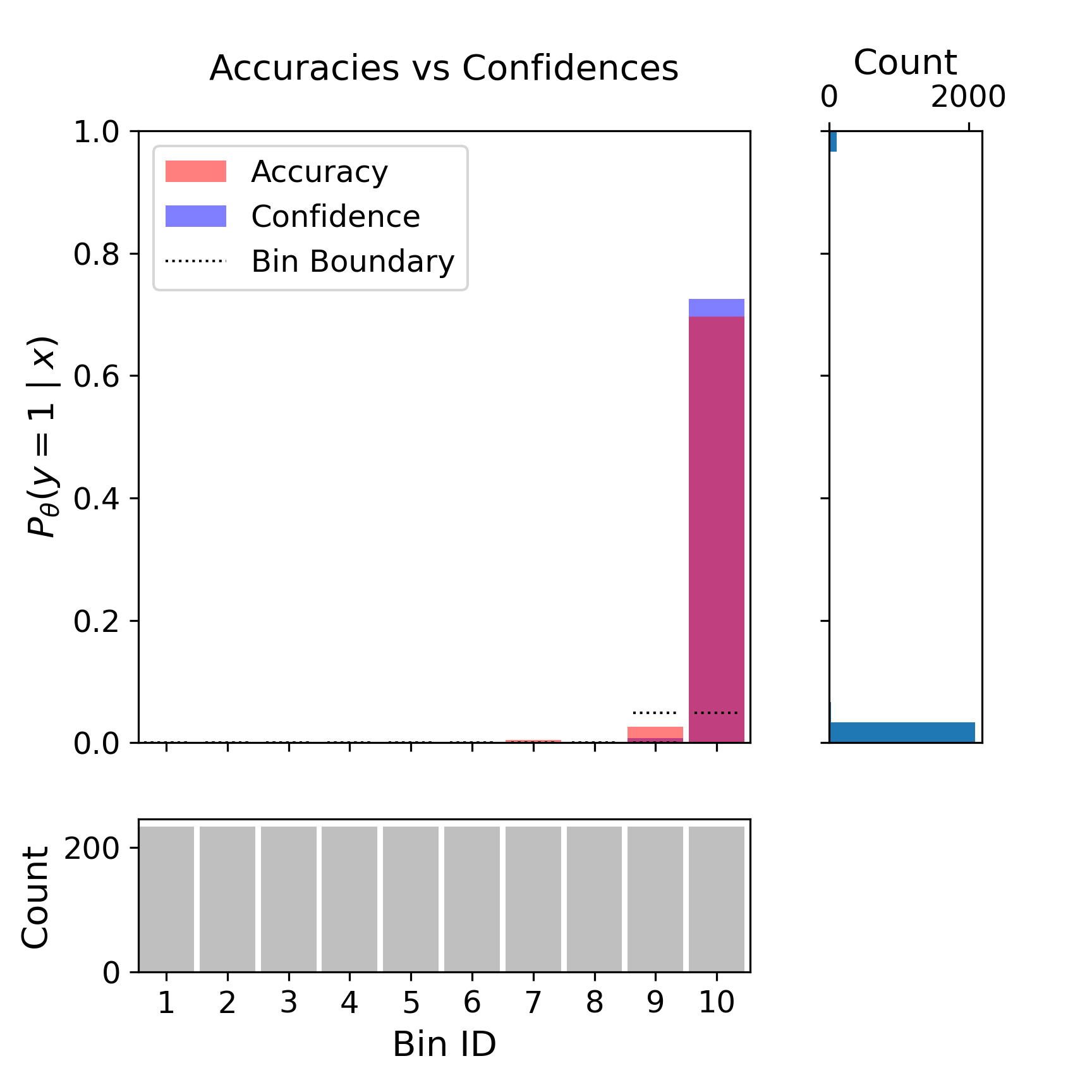}
    \caption{isolet}
    \end{subfigure}

    \begin{subfigure}[b]{\textwidth}
    \centering
    \includegraphics[trim={0 10pt 0 5pt},clip,width=0.32\columnwidth]{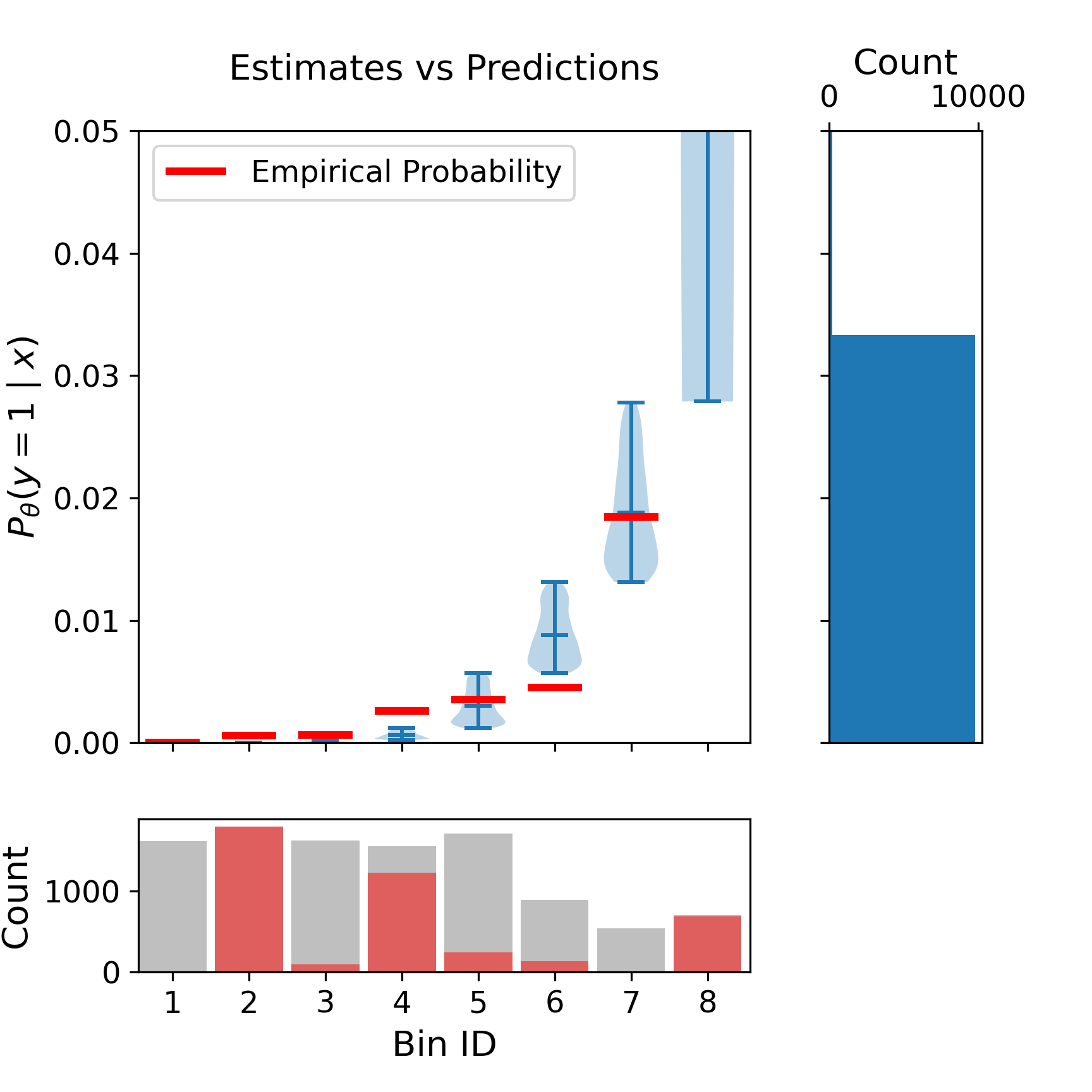}
    \hfill
    \includegraphics[trim={0 10pt 0 5pt},clip,width=0.32\columnwidth]{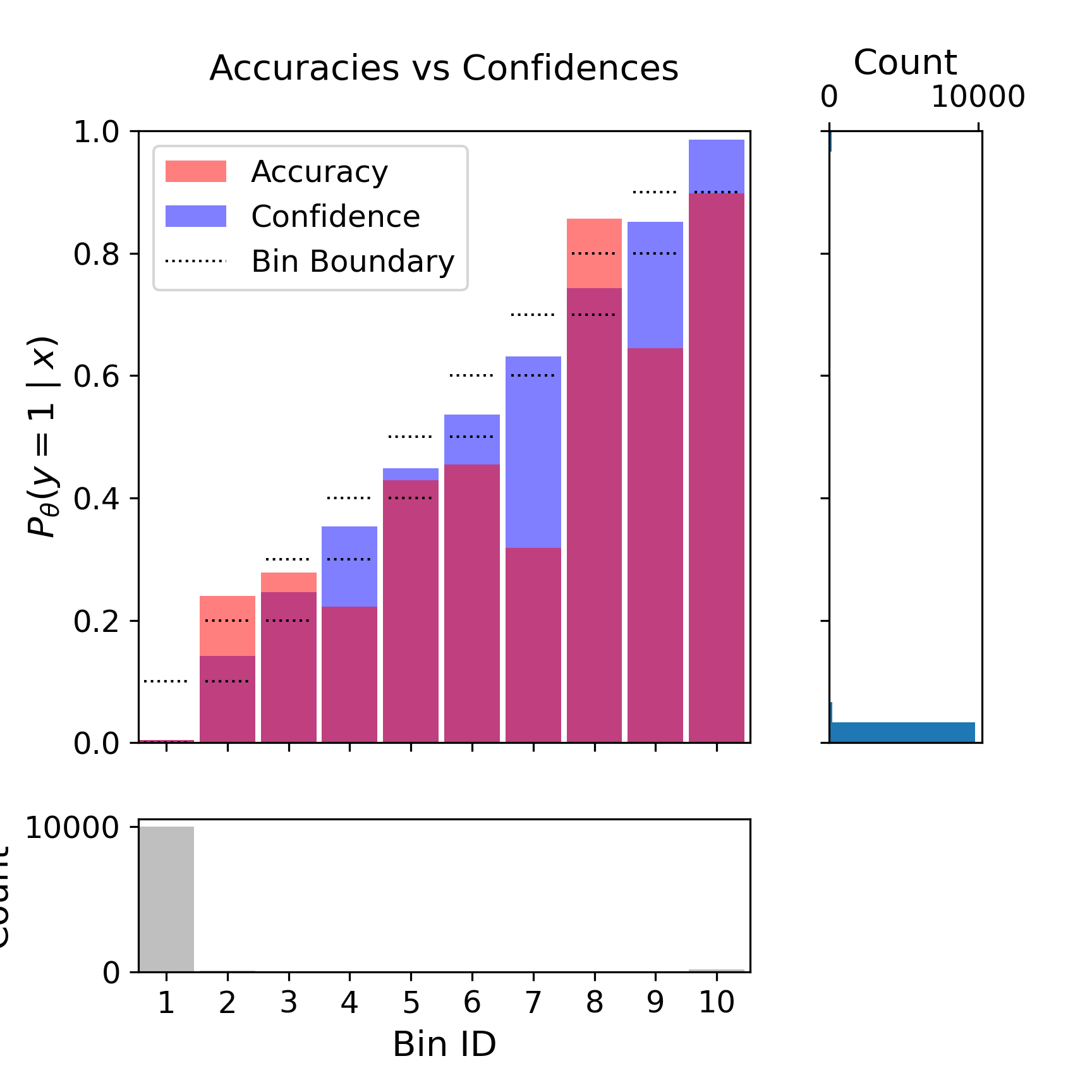}
    \hfill
    \includegraphics[trim={0 10pt 0 5pt},clip,width=0.32\columnwidth]{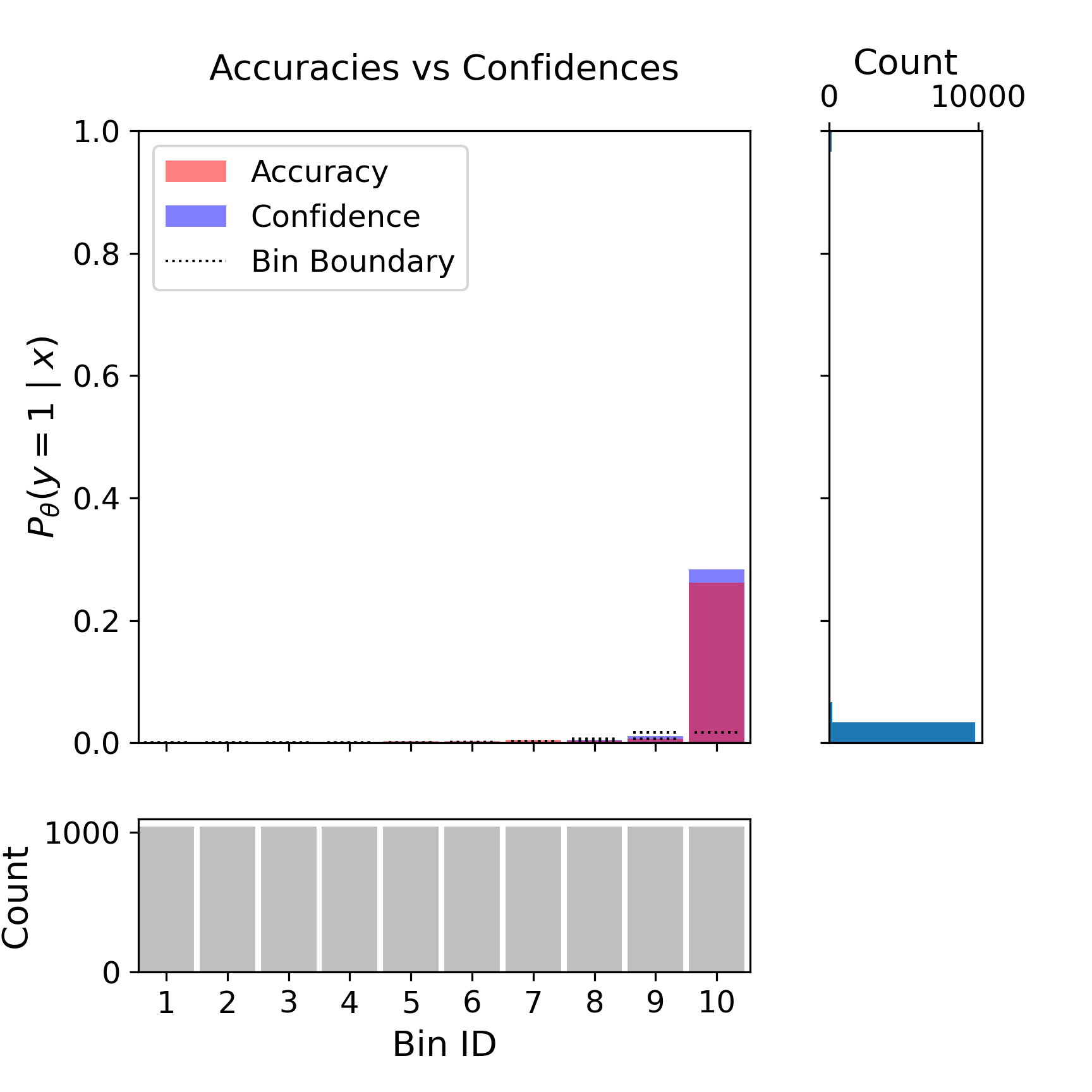}
    \caption{webpage}
    \end{subfigure}

    \caption{Comparison of visual representations of TCE, ECE and ACE for the logistic regression algorithm. (Left) The test-based reliability diagram of TCE. (Middle) The reliability diagram of ECE. (Right) The reliability diagram of ACE. Each row corresponds to a result on the dataset: (a) abalone, (b) coil\_2000, (c) isolet, and (d) webpage.}
    \label{fig:section_42_1}
\end{figure}

\begin{figure}[h]
    \centering
    
    \begin{subfigure}[b]{\textwidth}
    \centering
    \includegraphics[trim={0 10pt 0 5pt},clip,width=0.32\columnwidth]{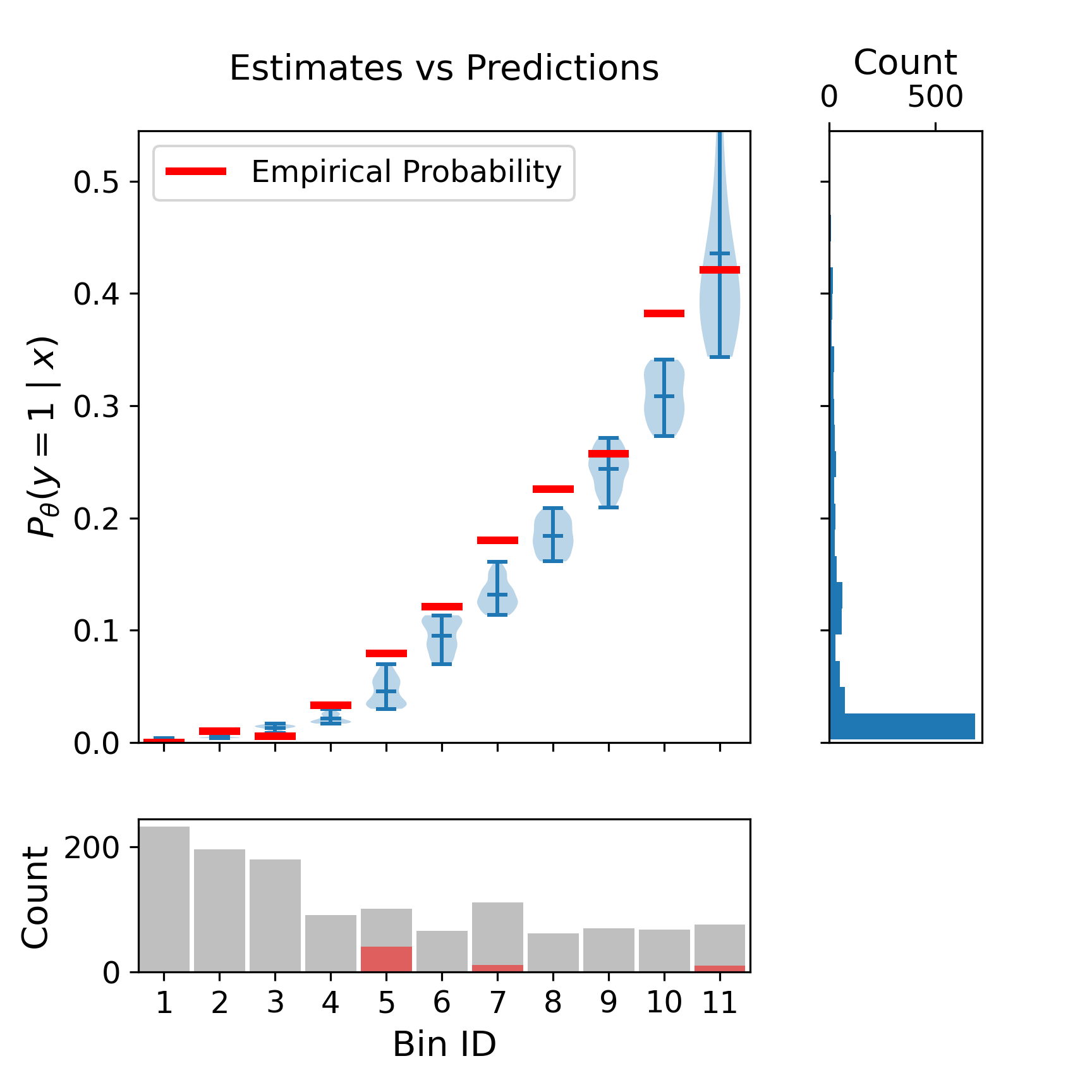}
    \hfill
    \includegraphics[trim={0 10pt 0 5pt},clip,width=0.32\columnwidth]{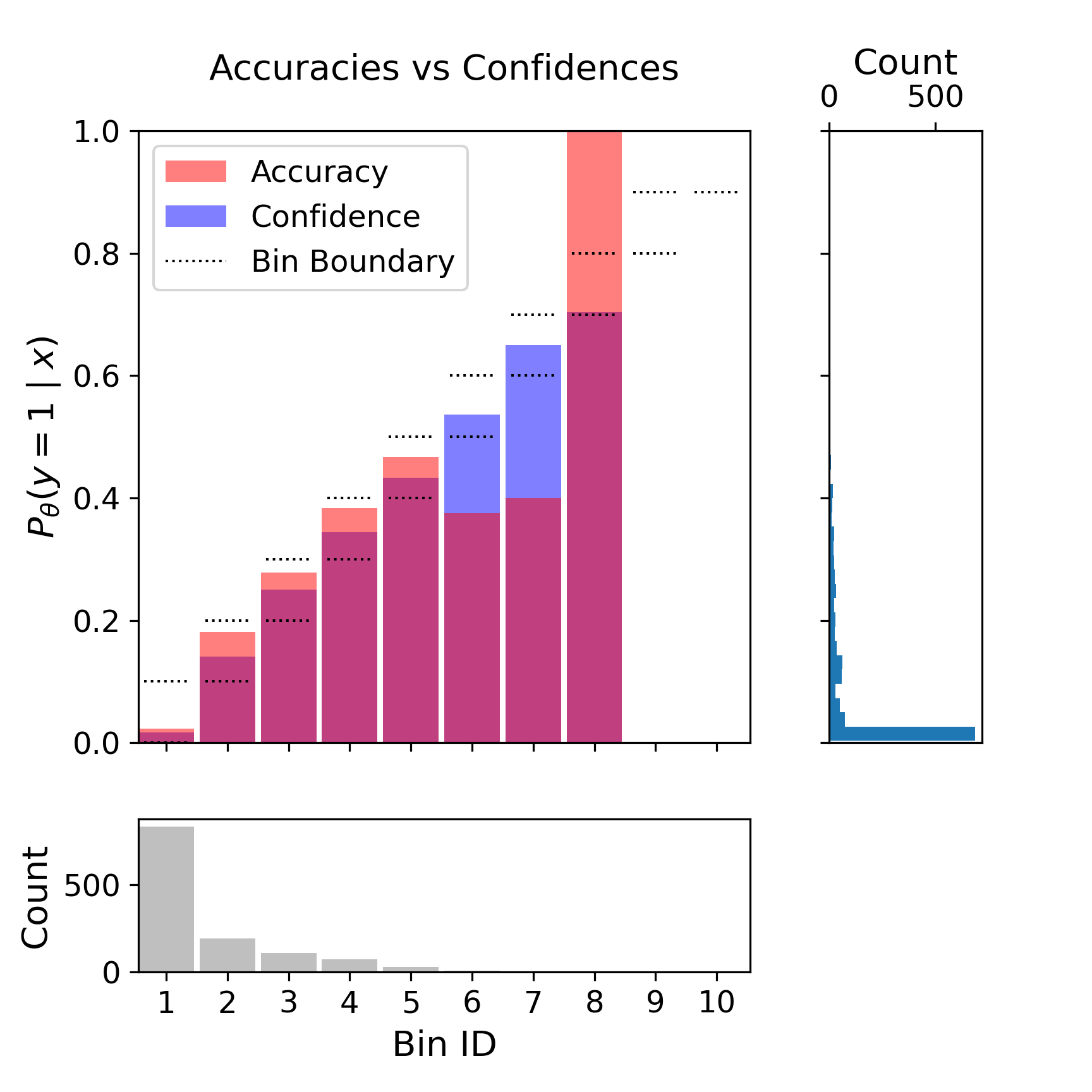}
    \hfill
    \includegraphics[trim={0 10pt 0 5pt},clip,width=0.32\columnwidth]{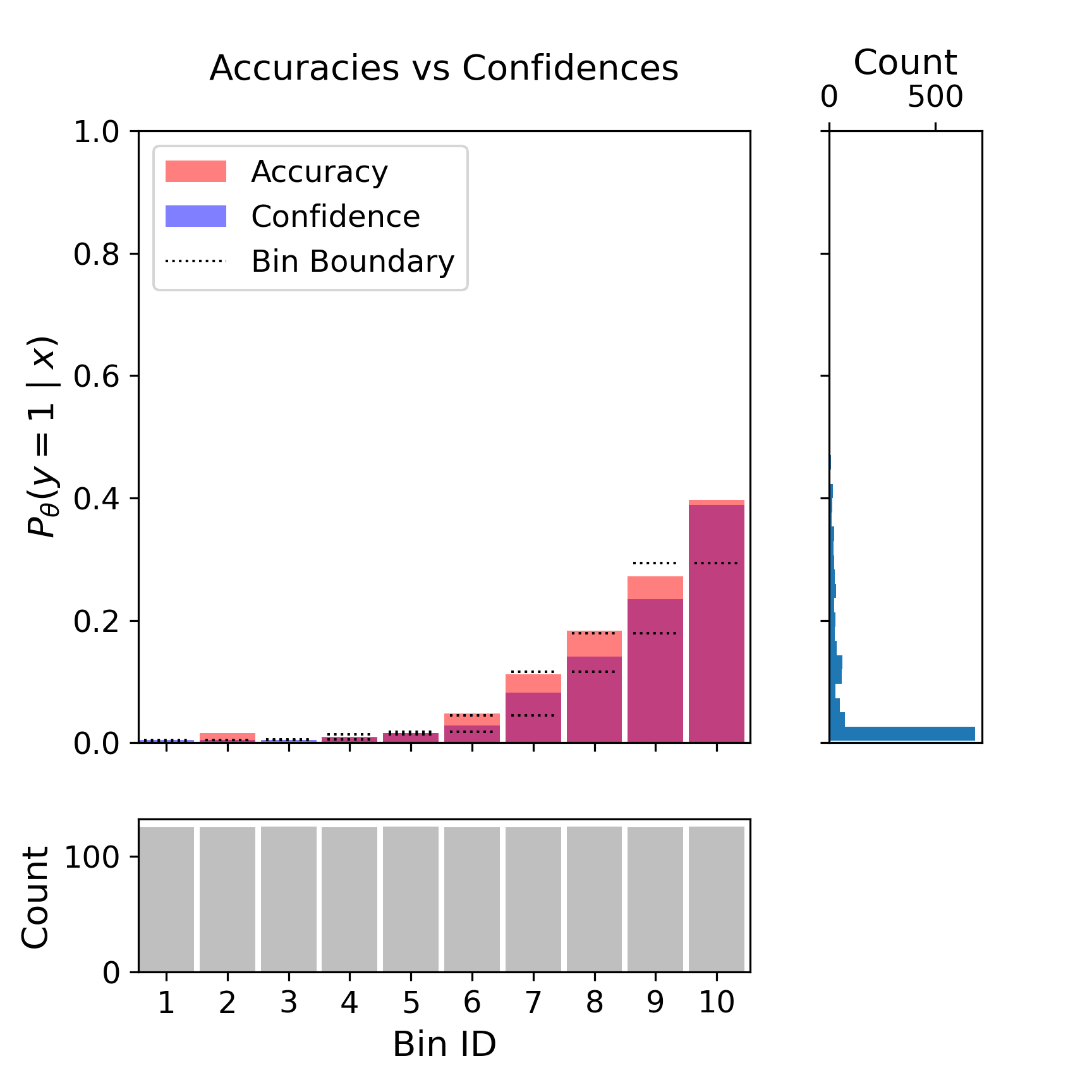}
    \caption{abalone}
    \end{subfigure}
    
    \begin{subfigure}[b]{\textwidth}
    \centering
    \includegraphics[trim={0 10pt 0 5pt},clip,width=0.32\columnwidth]{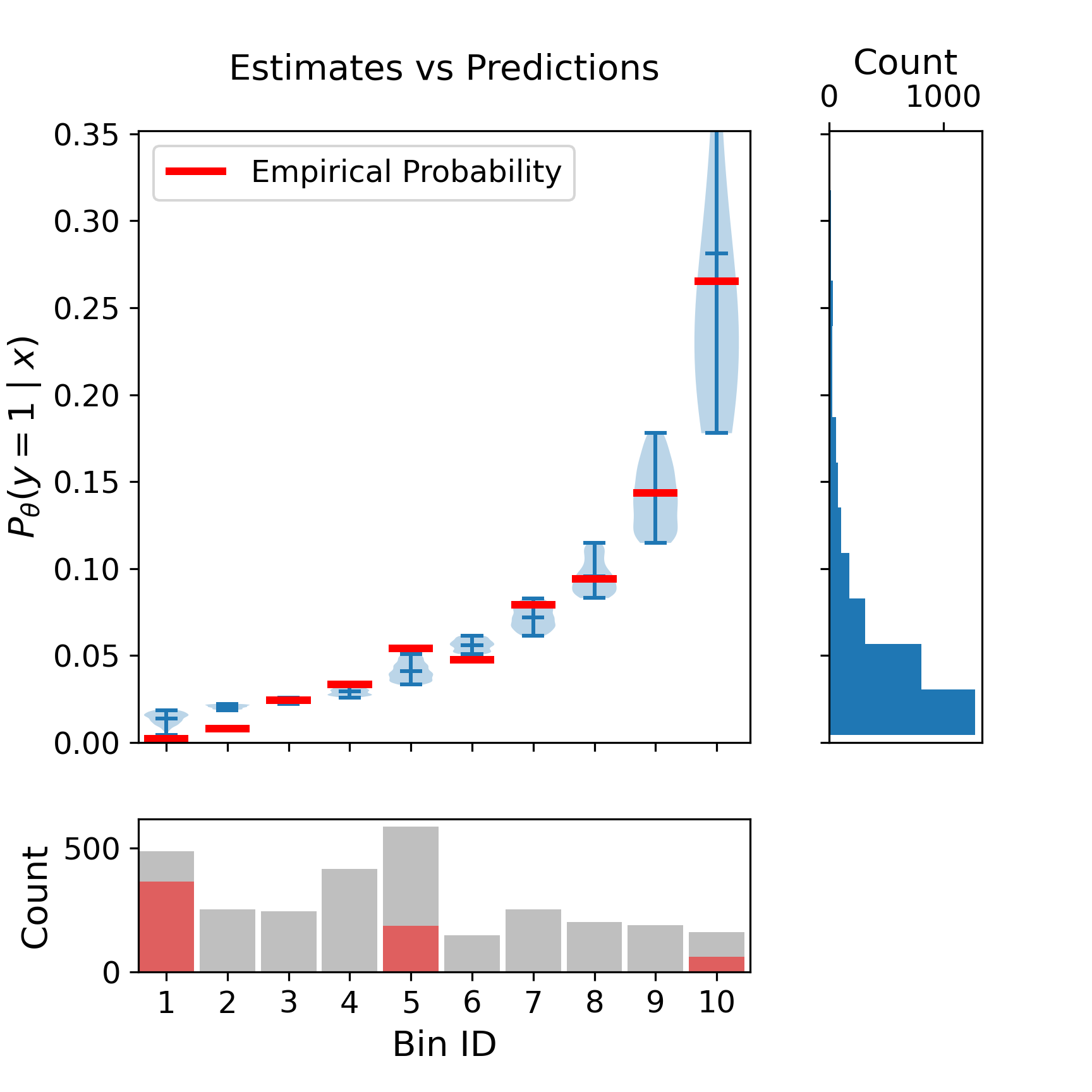}
    \hfill
    \includegraphics[trim={0 10pt 0 5pt},clip,width=0.32\columnwidth]{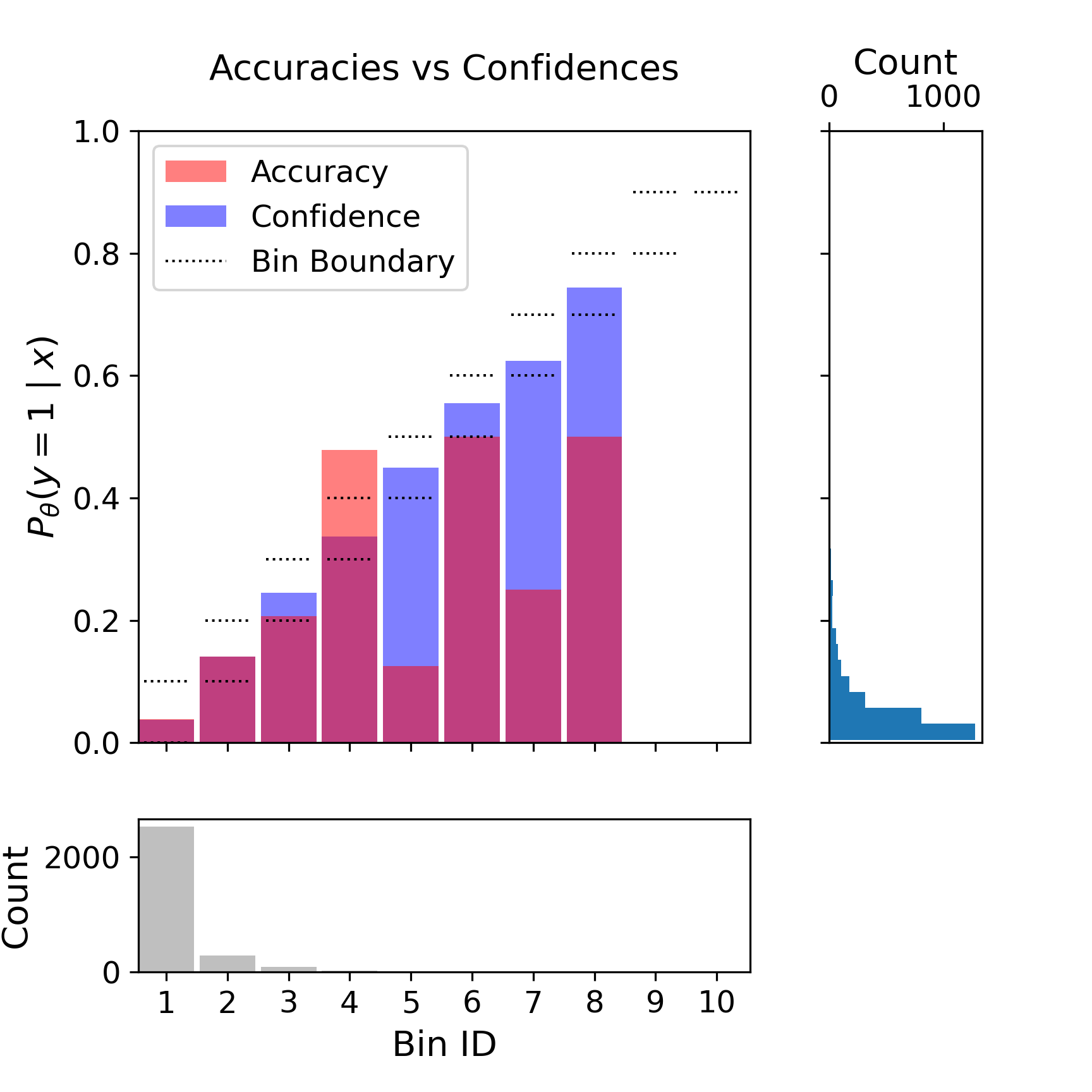}
    \hfill
    \includegraphics[trim={0 10pt 0 5pt},clip,width=0.32\columnwidth]{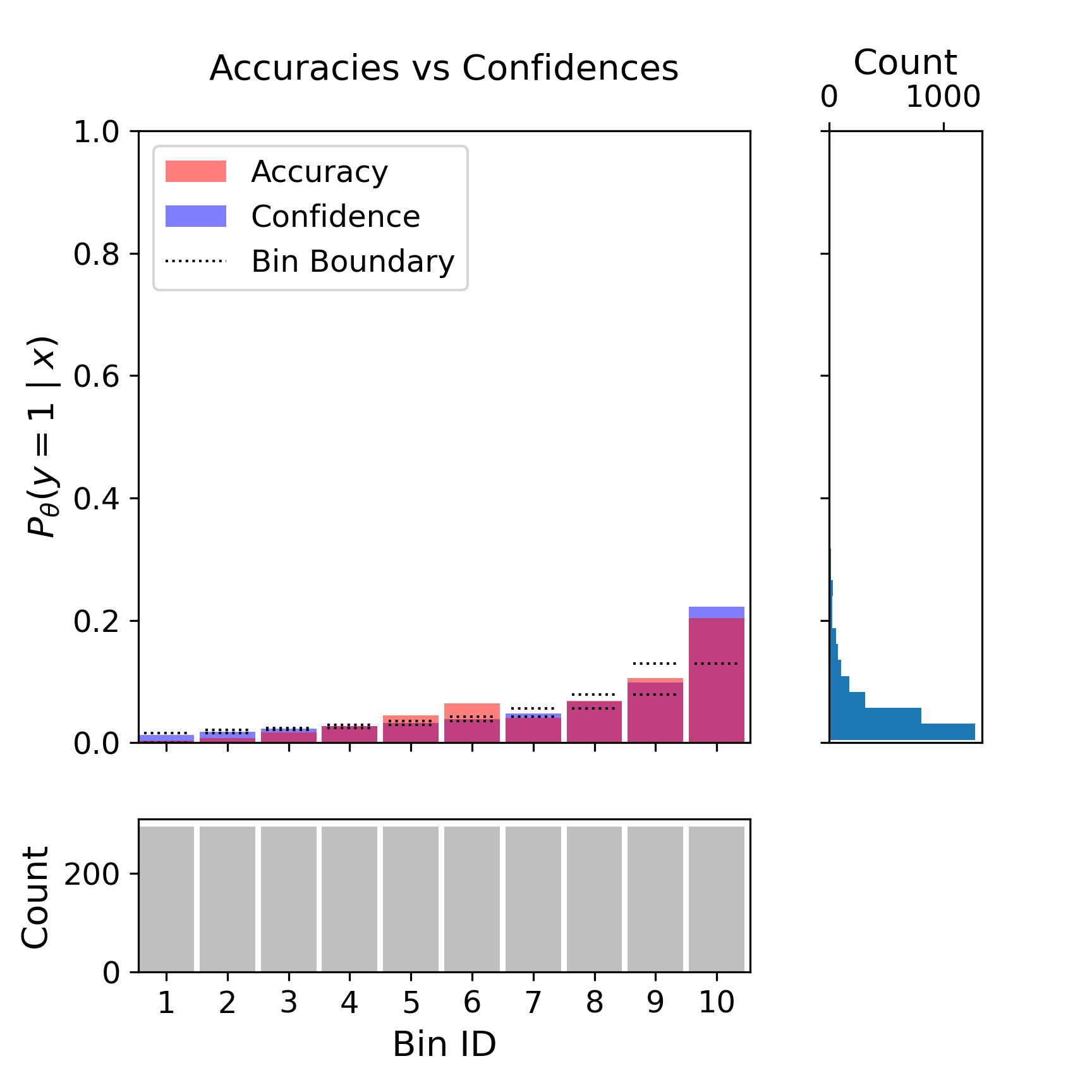}
    \caption{coil\_2000}
    \end{subfigure}
    
    \begin{subfigure}[b]{\textwidth}
    \centering
    \includegraphics[trim={0 10pt 0 5pt},clip,width=0.32\columnwidth]{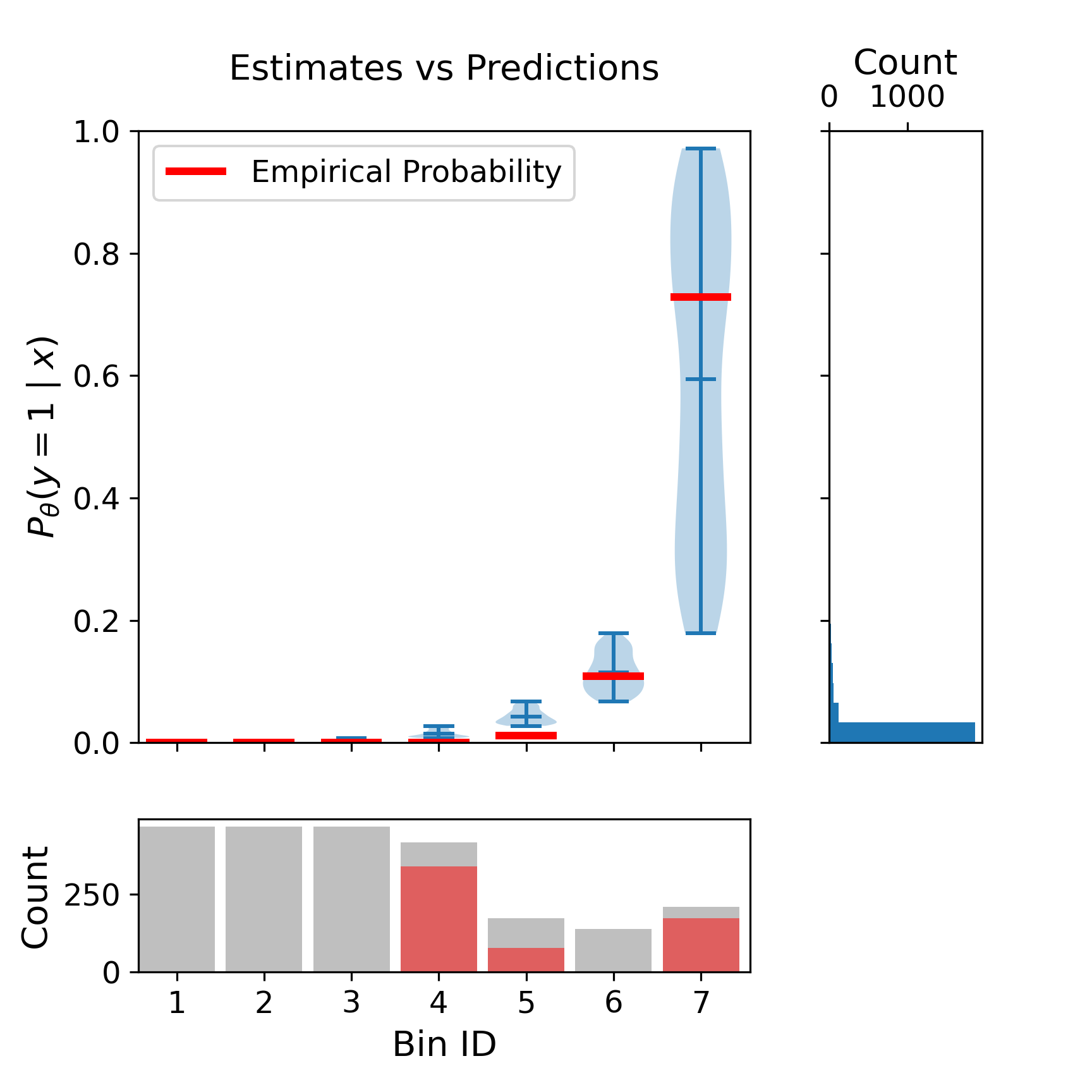}
    \hfill
    \includegraphics[trim={0 10pt 0 5pt},clip,width=0.32\columnwidth]{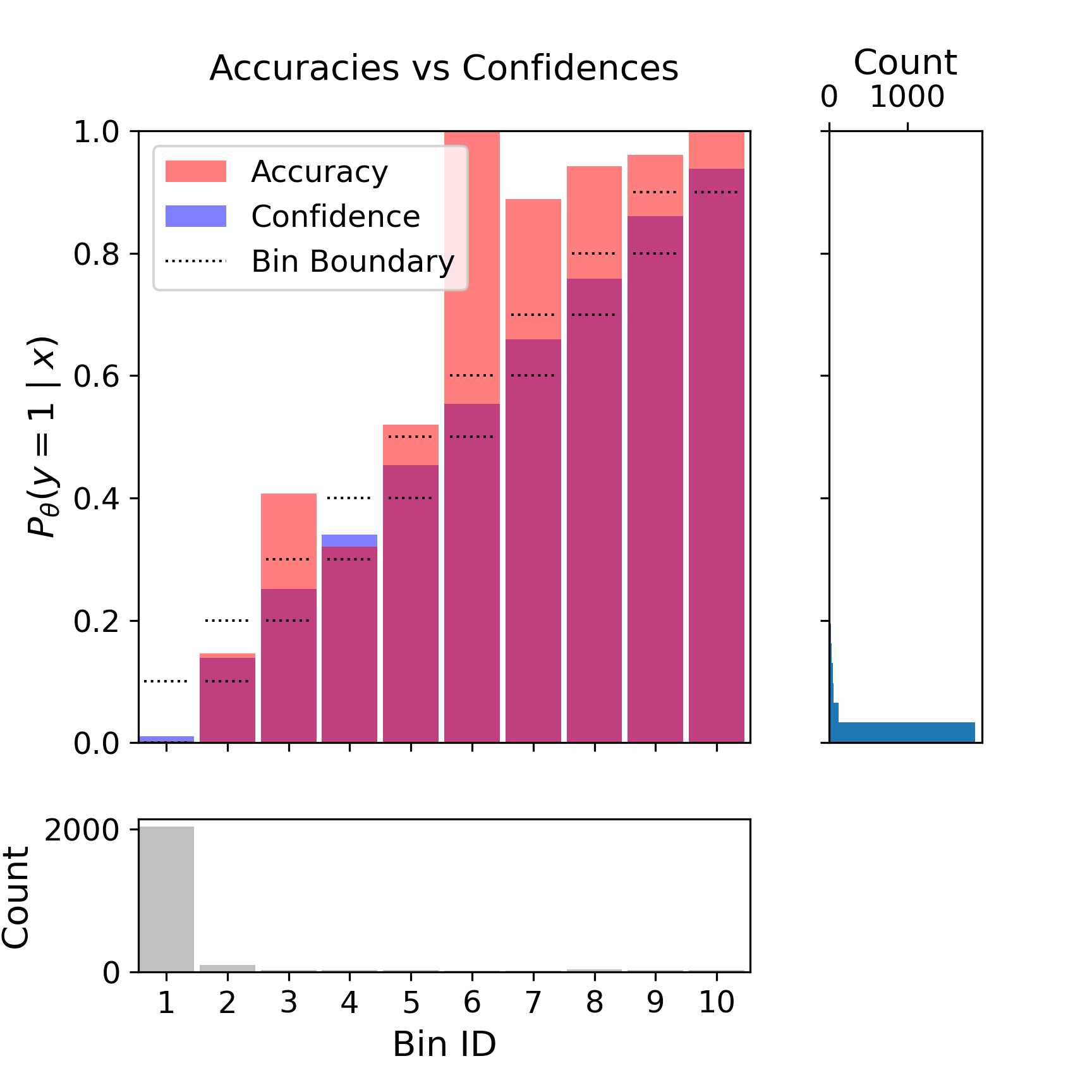}
    \hfill
    \includegraphics[trim={0 10pt 0 5pt},clip,width=0.32\columnwidth]{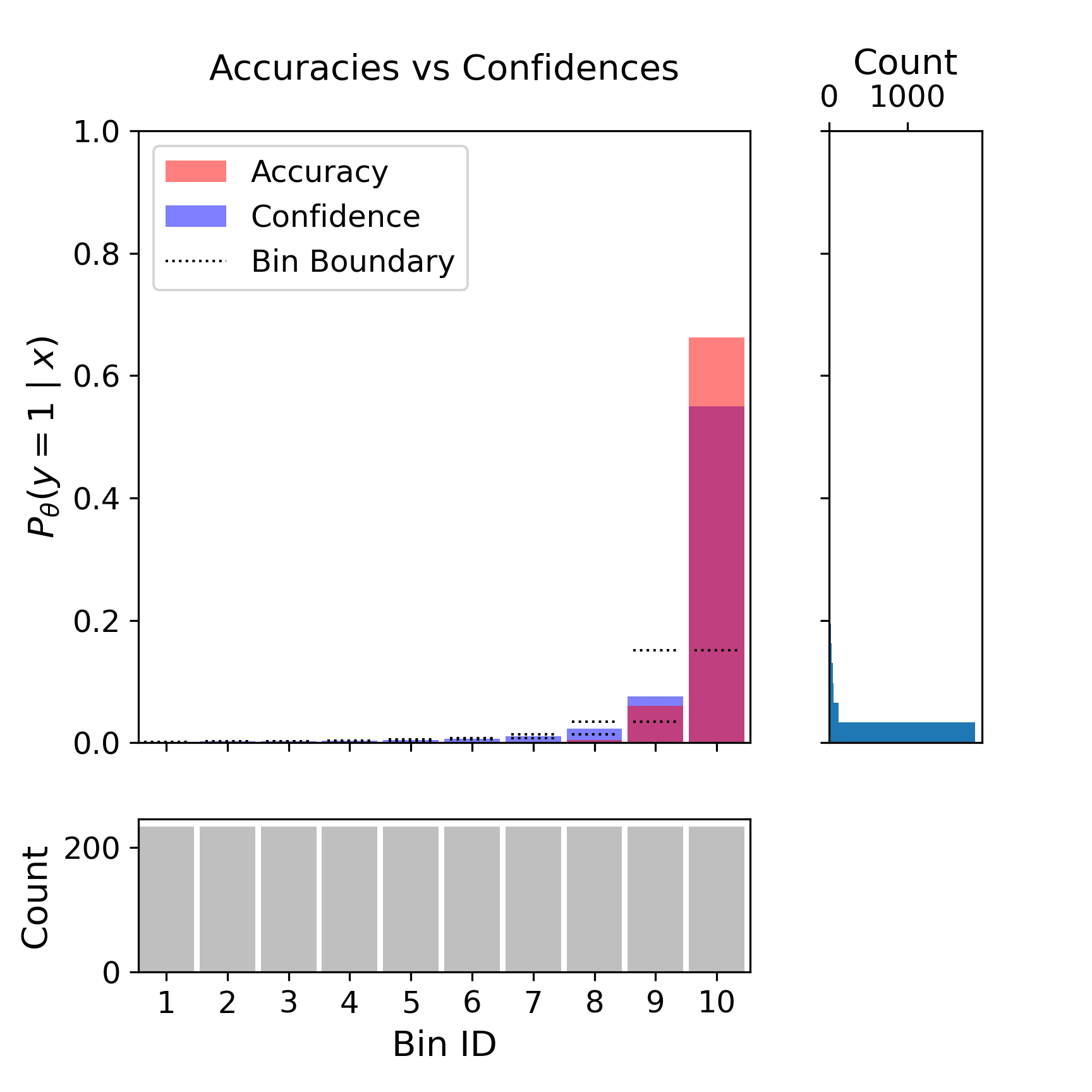}
    \caption{isolet}
    \end{subfigure}

    \begin{subfigure}[b]{\textwidth}
    \centering
    \includegraphics[trim={0 10pt 0 5pt},clip,width=0.32\columnwidth]{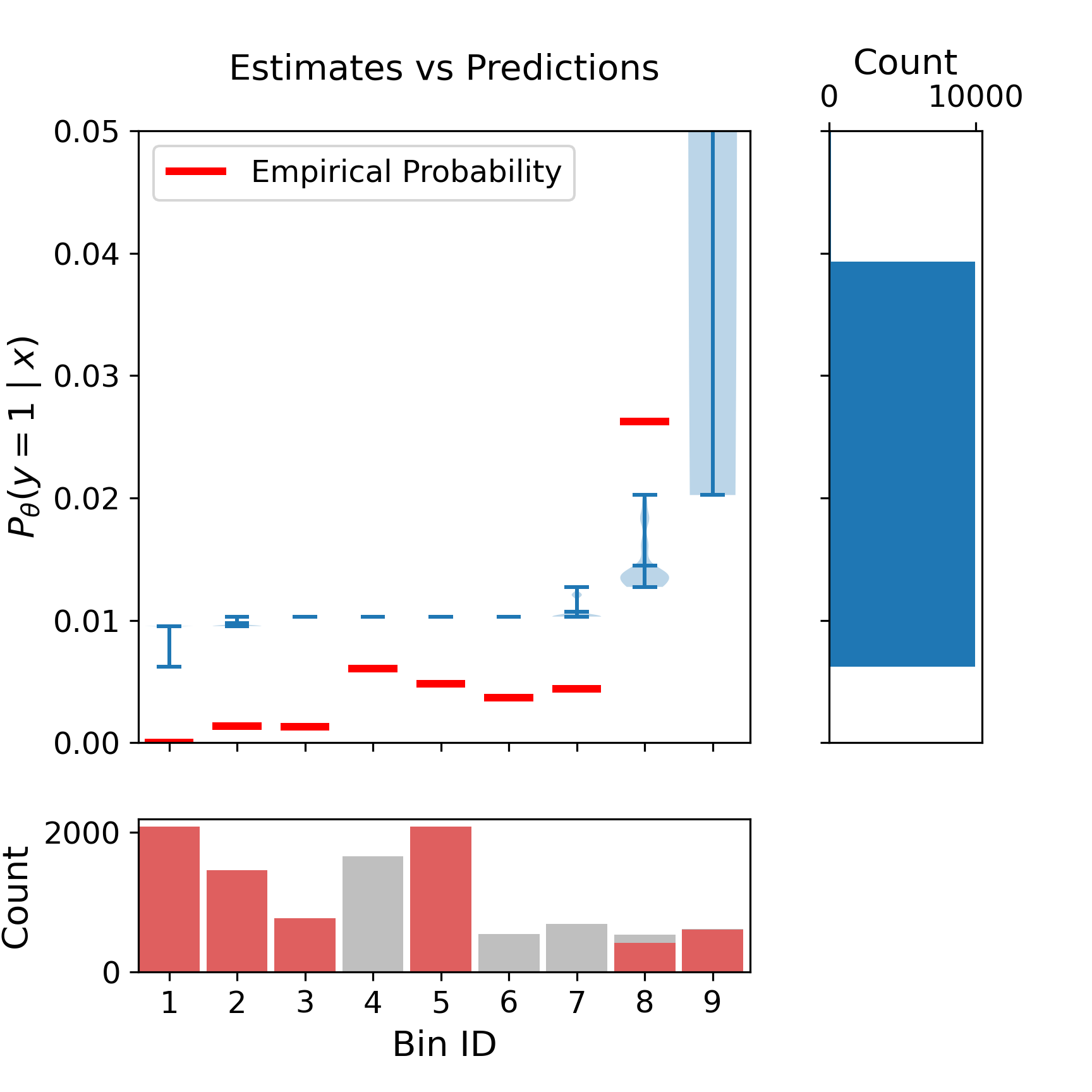}
    \hfill
    \includegraphics[trim={0 10pt 0 5pt},clip,width=0.32\columnwidth]{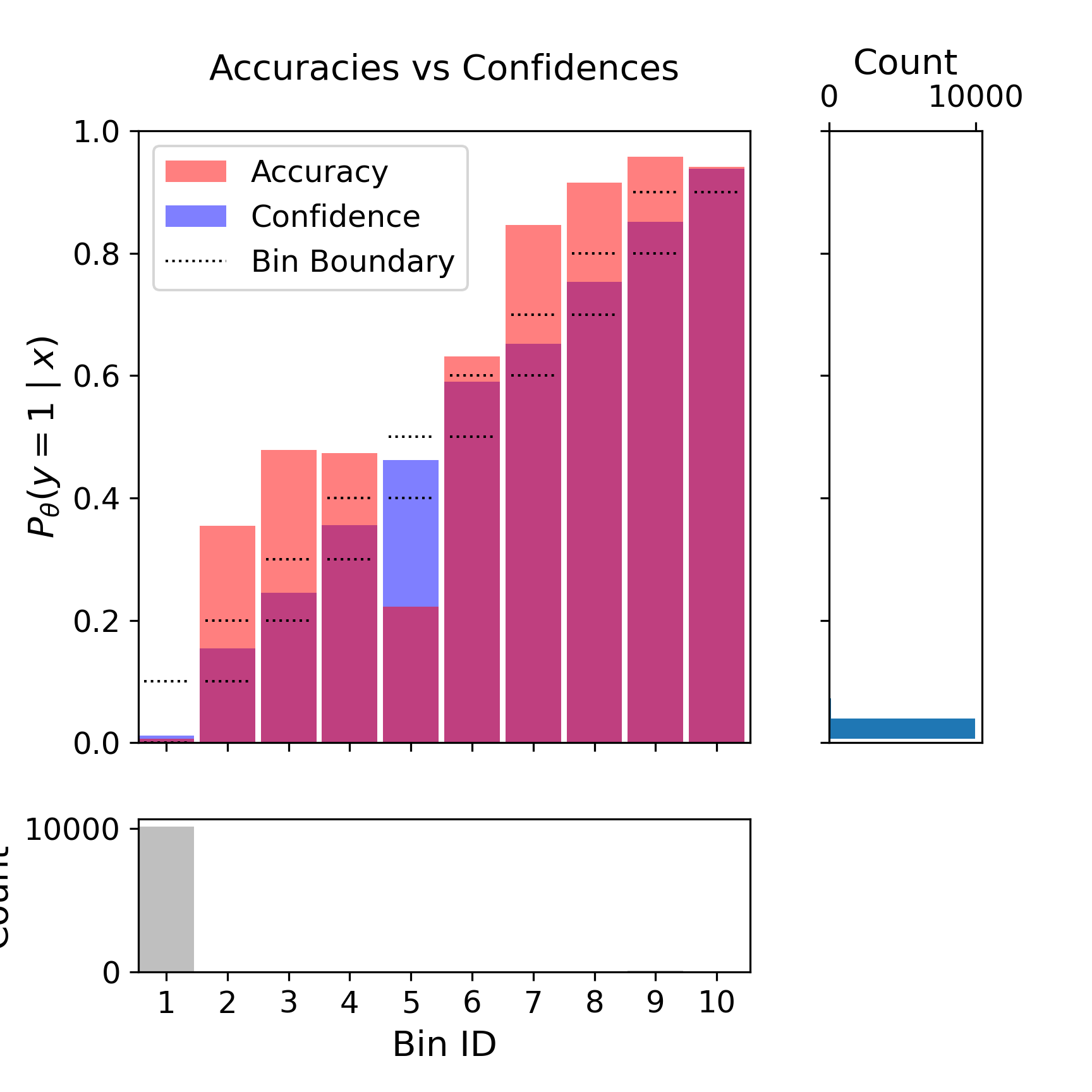}
    \hfill
    \includegraphics[trim={0 10pt 0 5pt},clip,width=0.32\columnwidth]{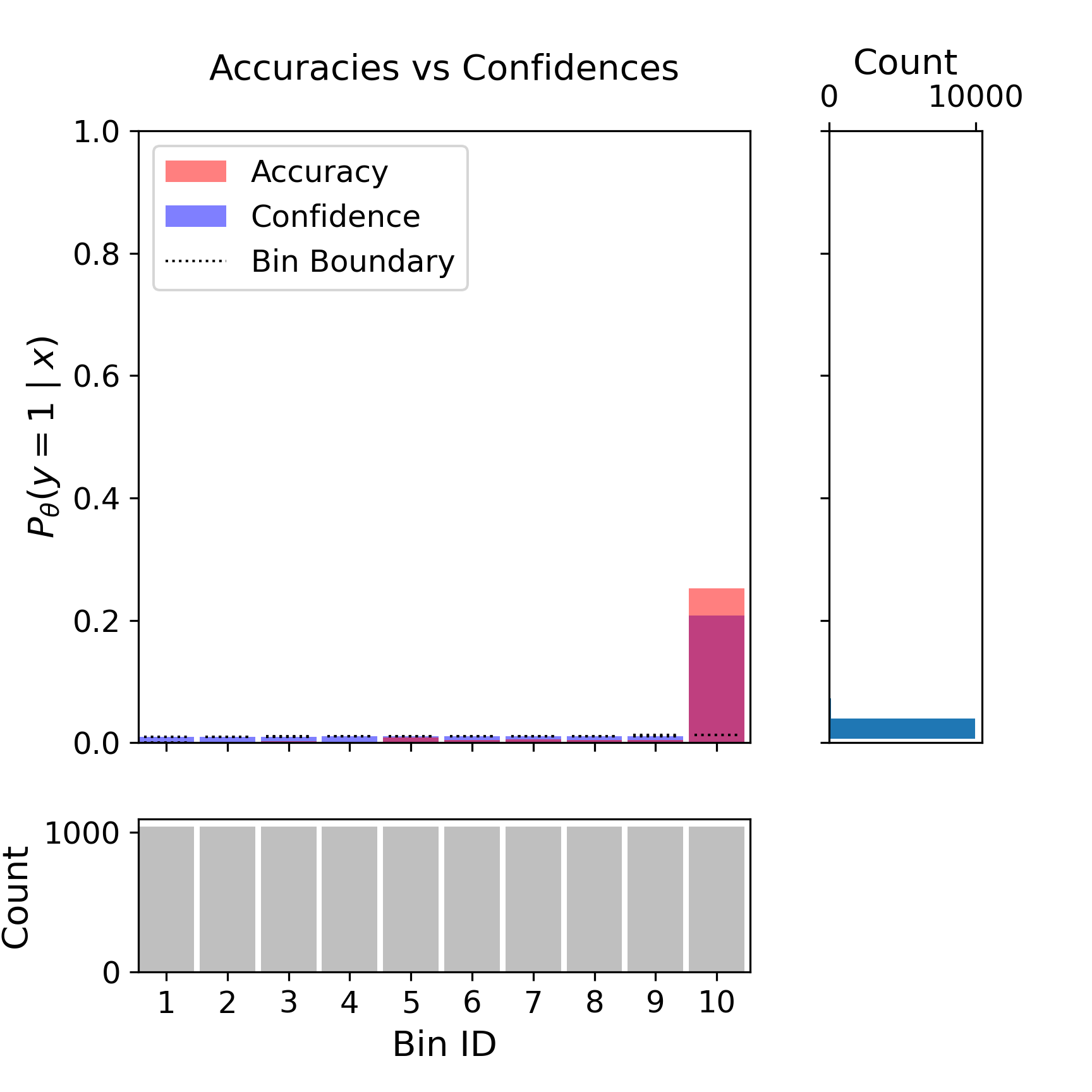}
    \caption{isolet}
    \end{subfigure}

    \caption{Comparison of visual representations of TCE, ECE and ACE for the logistic regression algorithm. (Left) The test-based reliability diagram of TCE. (Middle) The reliability diagram of ECE. (Right) The reliability diagram of ACE. Each row corresponds to a result on the dataset: (a) abalone, (b) coil\_2000, (c) isolet, and (d) webpage.}
    \label{fig:section_42_2}
\end{figure}

\begin{figure}[h]
    \centering
    \begin{subfigure}[b]{\textwidth}
    \centering
    \includegraphics[trim={0 10pt 0 5pt},clip,width=0.32\columnwidth]{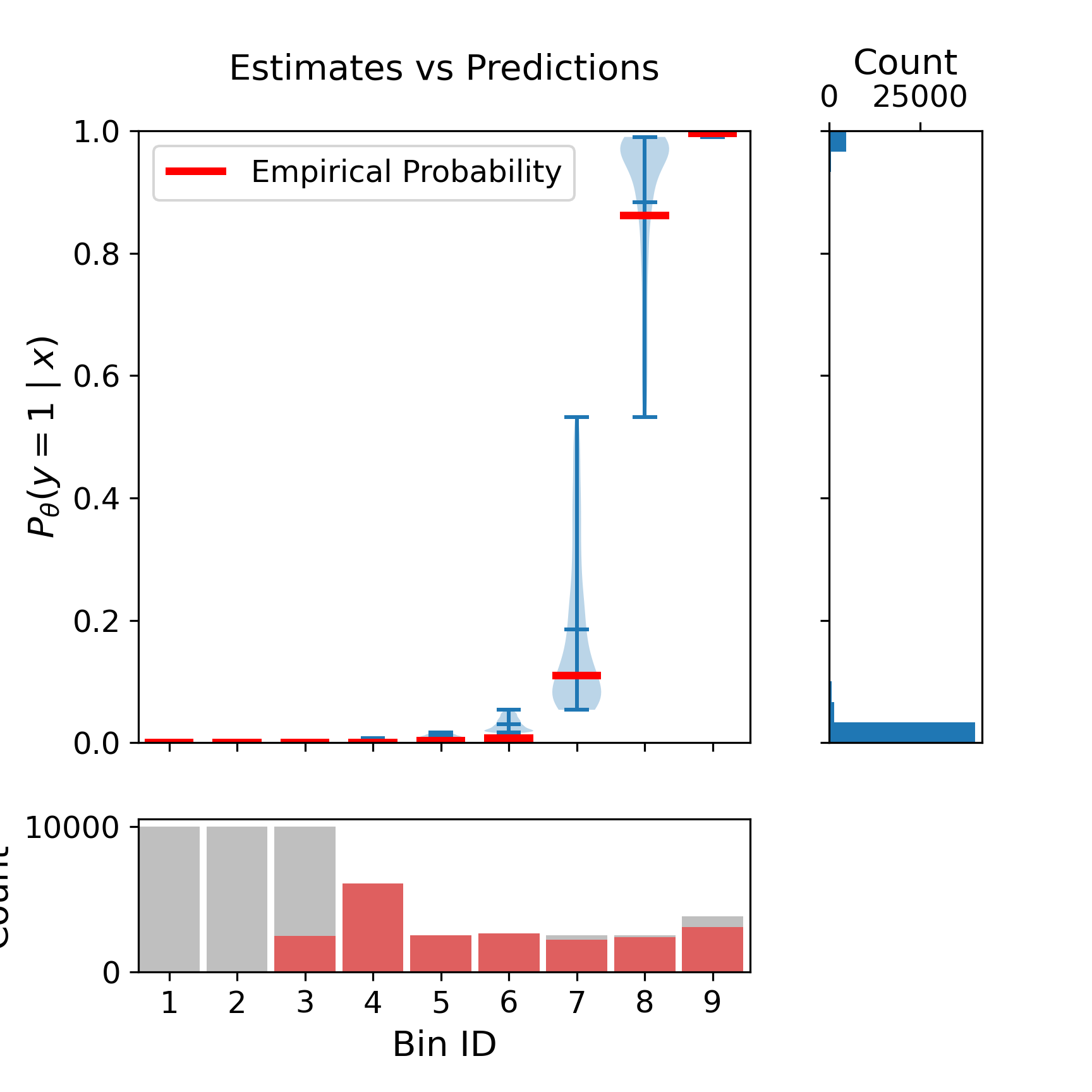}
    \hfill
    \includegraphics[trim={0 10pt 0 5pt},clip,width=0.32\columnwidth]{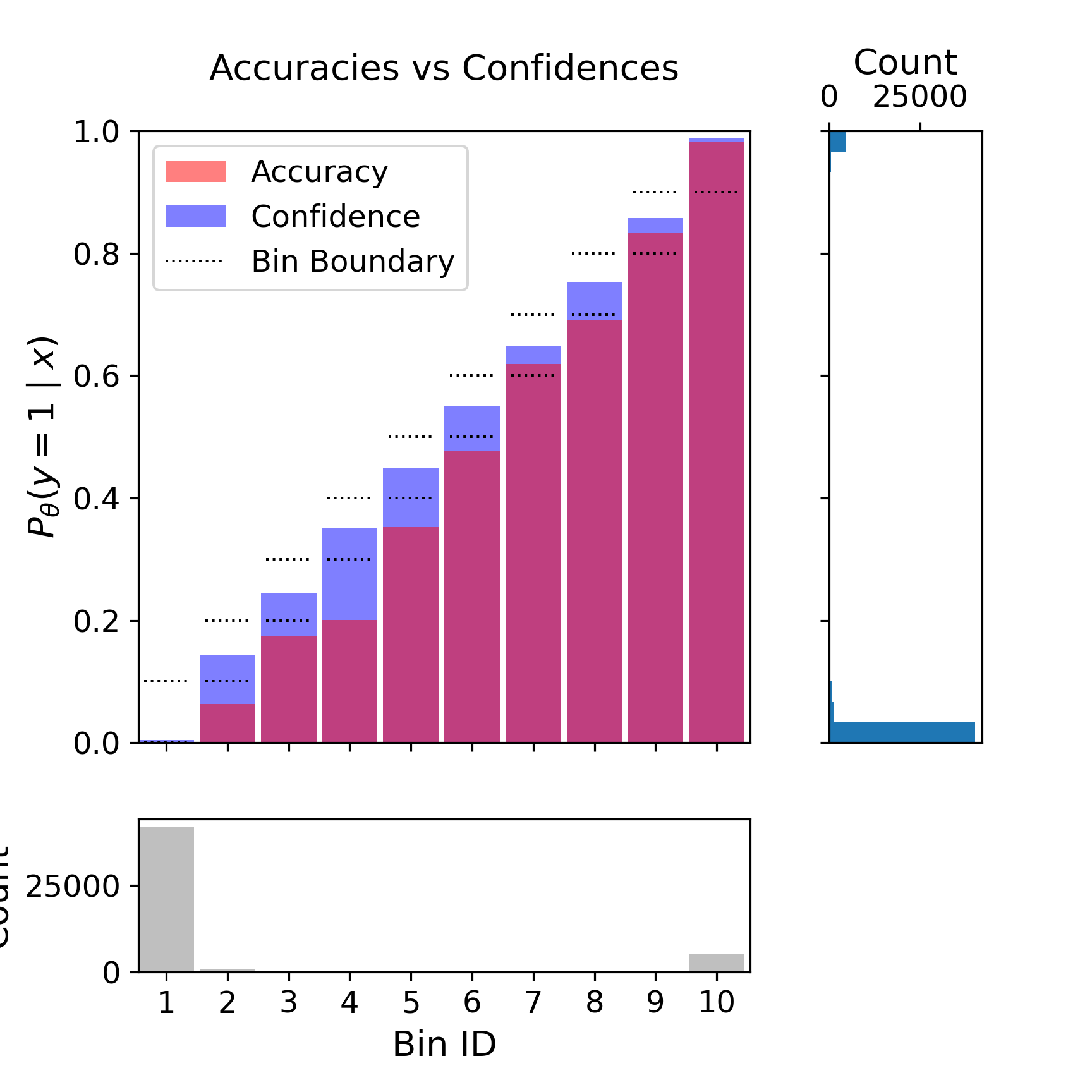}
    \hfill
    \includegraphics[trim={0 10pt 0 5pt},clip,width=0.32\columnwidth]{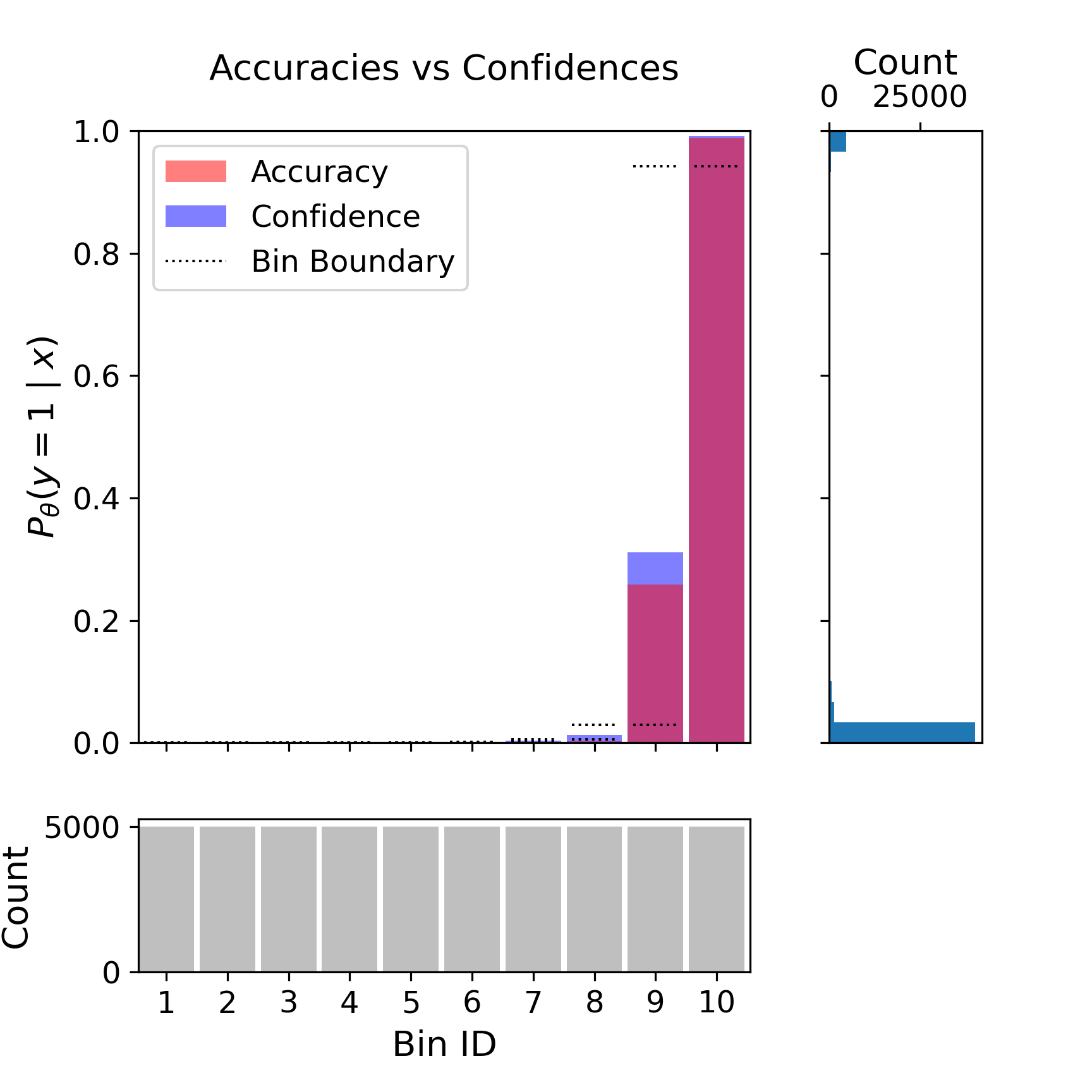}
    \caption{AlexNet}
    \end{subfigure}

    \begin{subfigure}[b]{\textwidth}
    \centering
    \includegraphics[trim={0 10pt 0 5pt},clip,width=0.32\columnwidth]{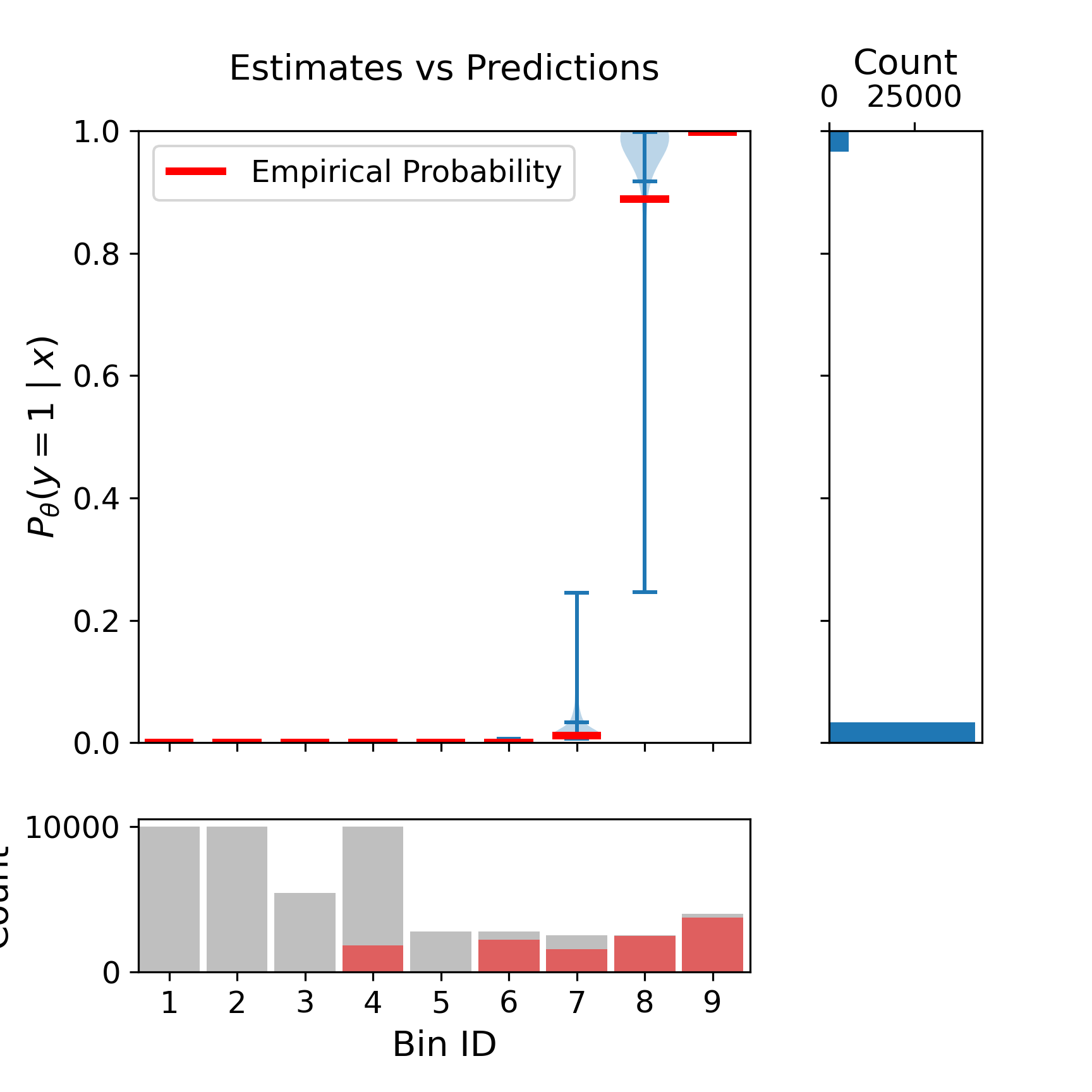}
    \hfill
    \includegraphics[trim={0 10pt 0 5pt},clip,width=0.32\columnwidth]{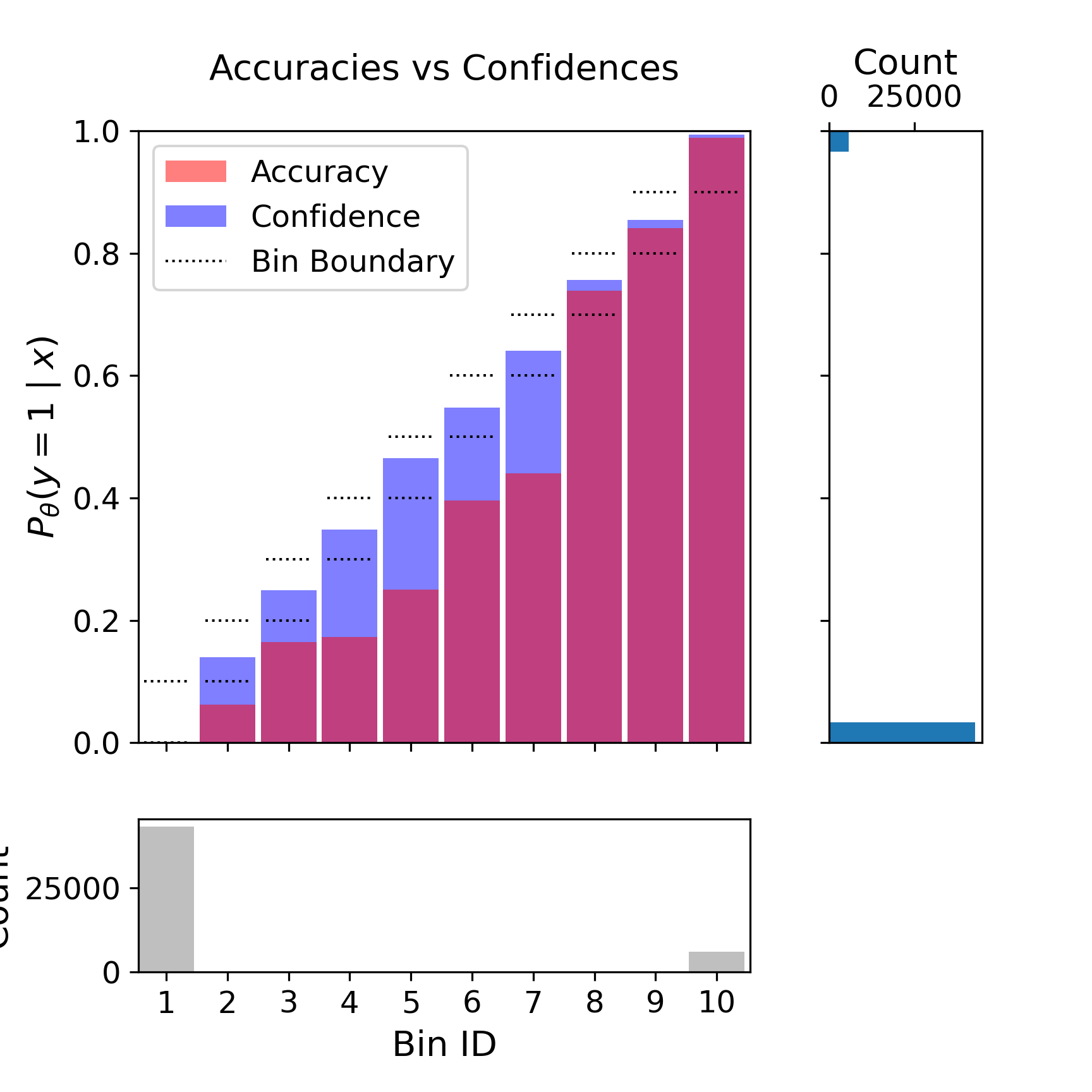}
    \hfill
    \includegraphics[trim={0 10pt 0 5pt},clip,width=0.32\columnwidth]{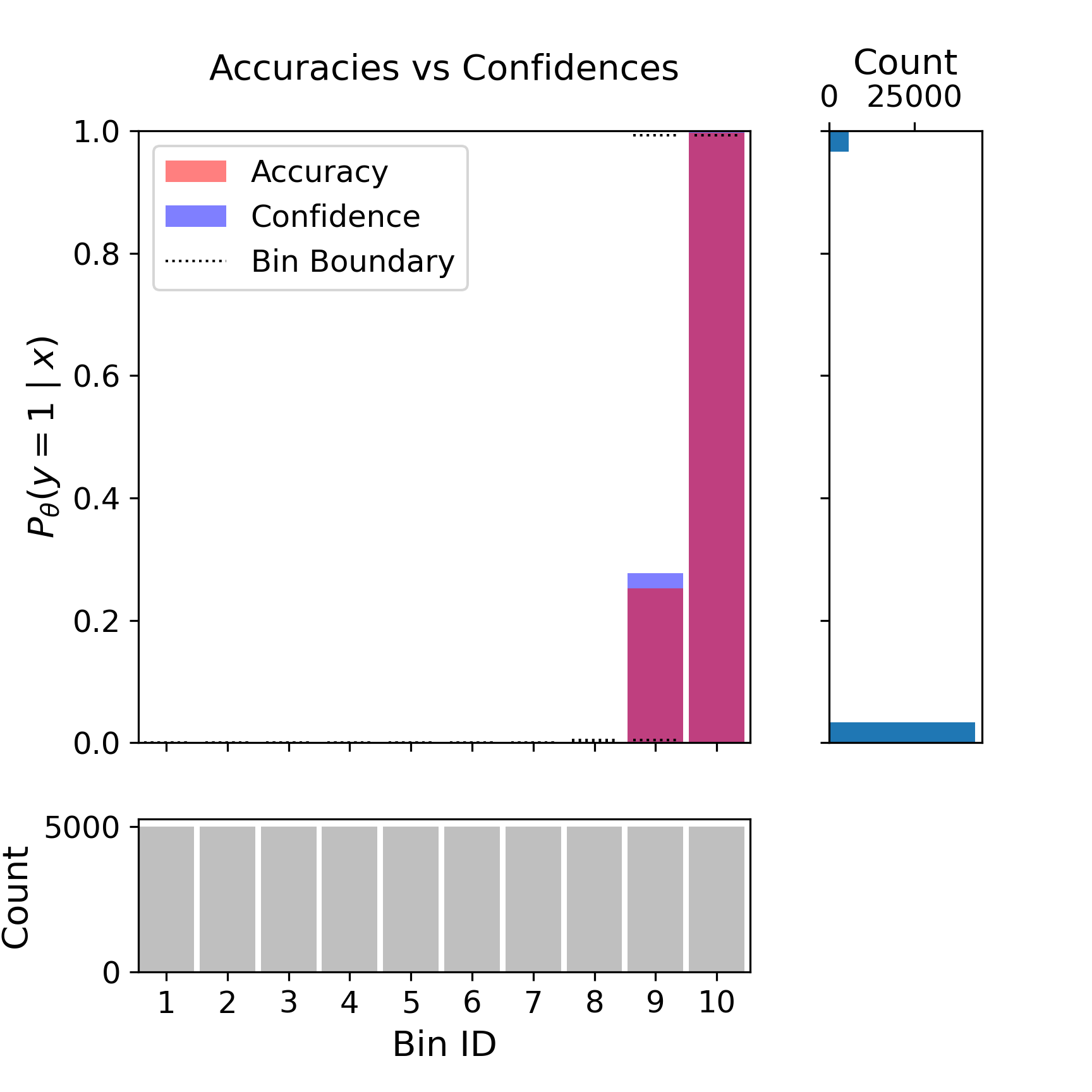}
    \caption{VGG19}
    \end{subfigure}
    
    \begin{subfigure}[b]{\textwidth}
    \centering
    \includegraphics[trim={0 10pt 0 5pt},clip,width=0.32\columnwidth]{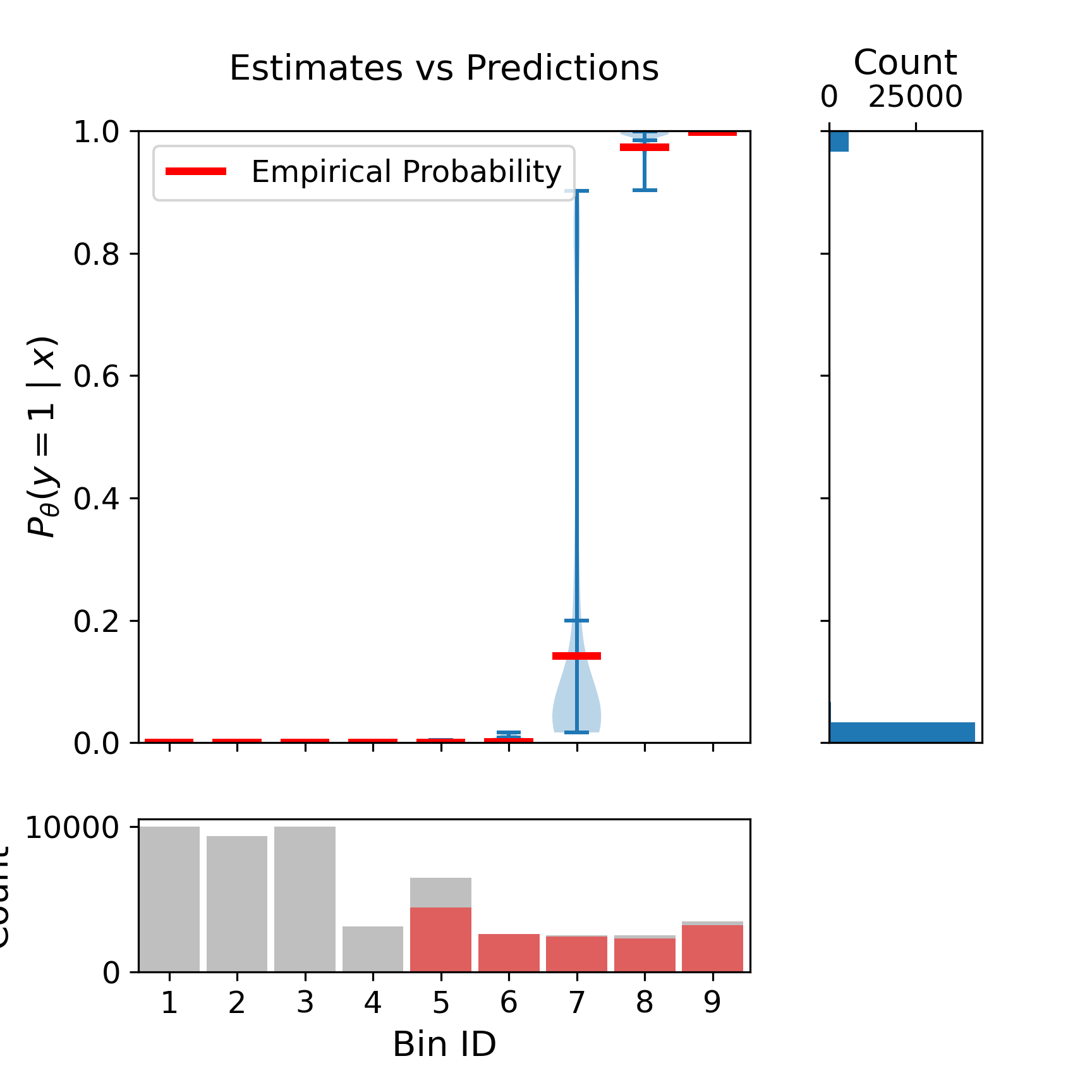}
    \hfill
    \includegraphics[trim={0 10pt 0 5pt},clip,width=0.32\columnwidth]{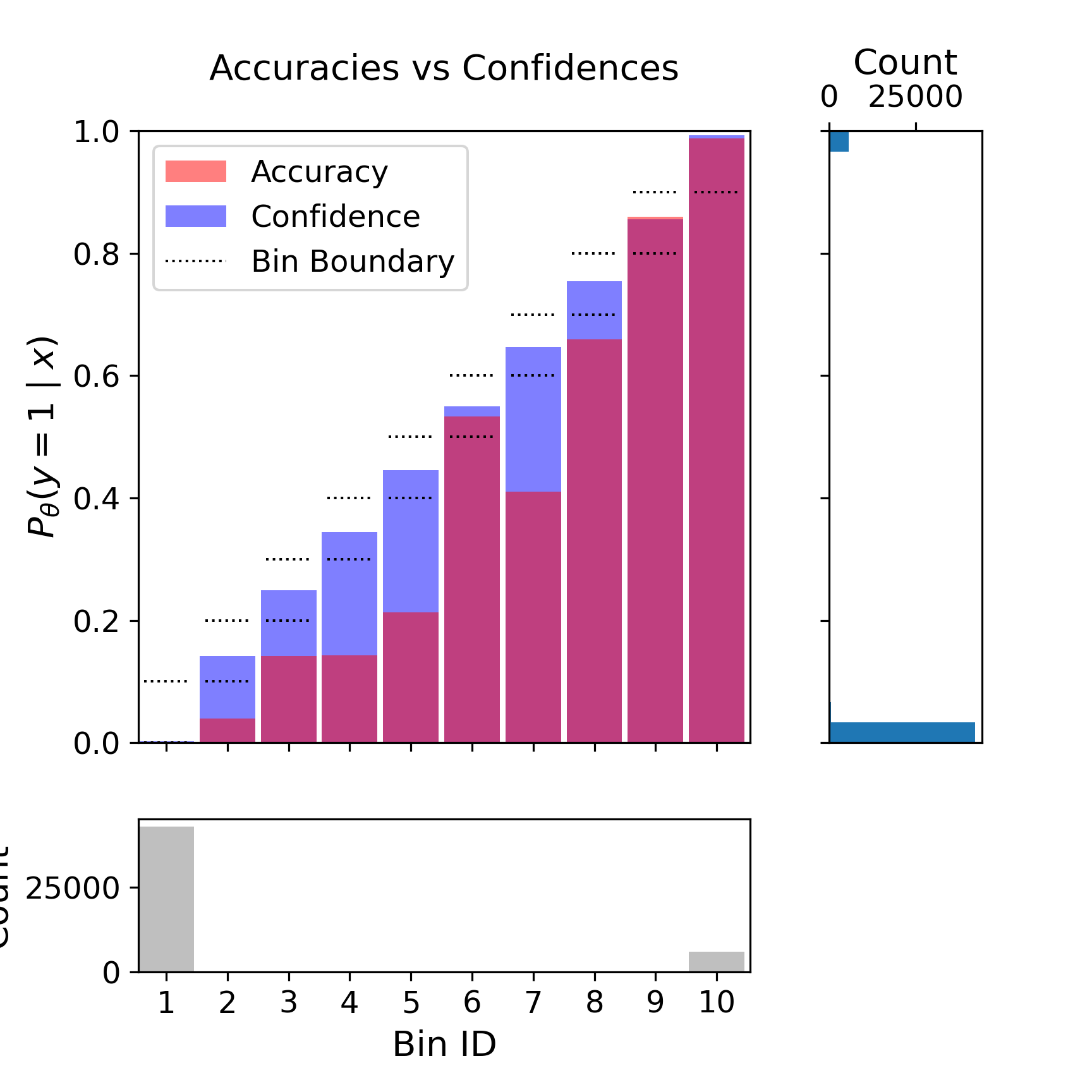}
    \hfill
    \includegraphics[trim={0 10pt 0 5pt},clip,width=0.32\columnwidth]{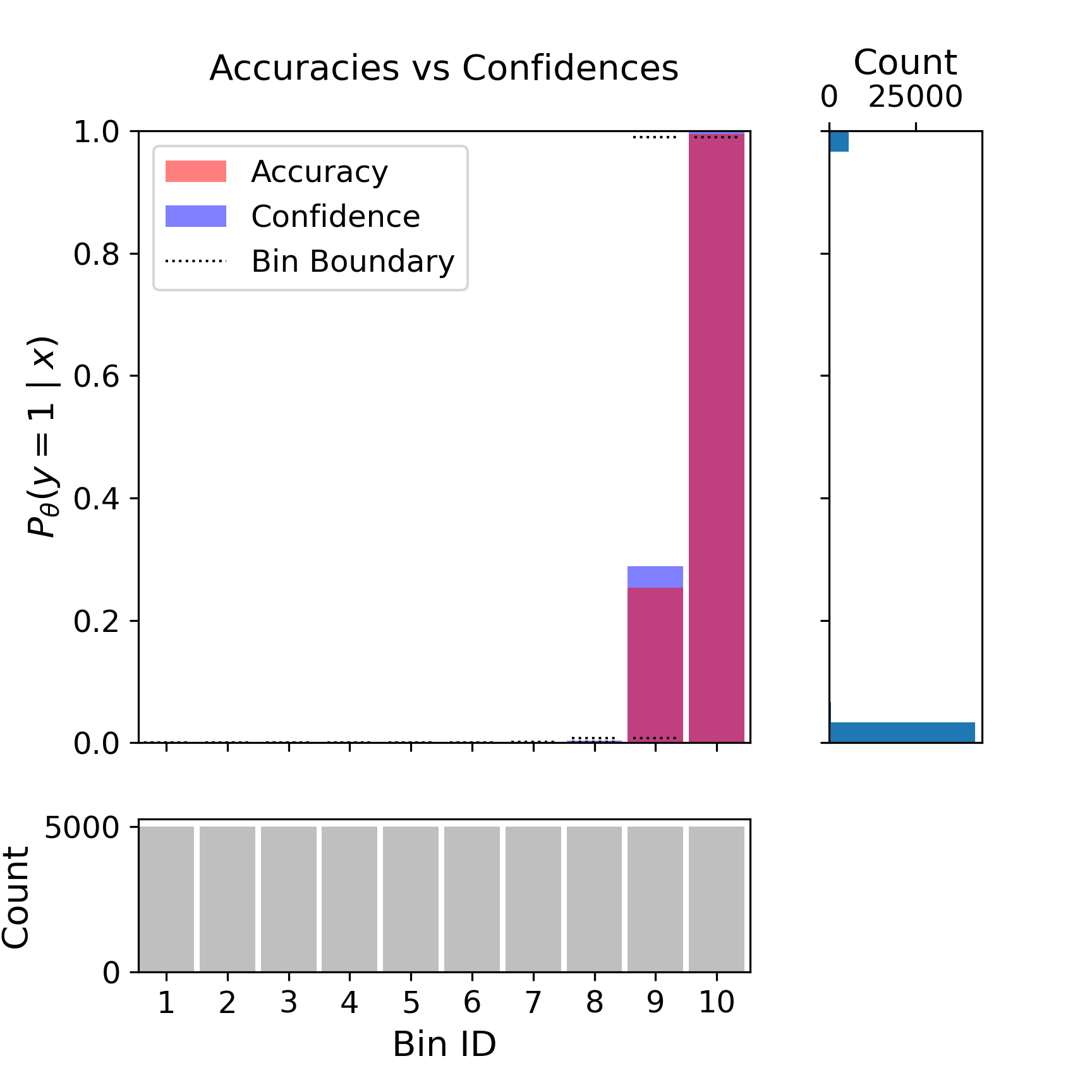}
    \caption{ResNet 18}
    \end{subfigure}
    
    \begin{subfigure}[b]{\textwidth}
    \centering
    \includegraphics[trim={0 10pt 0 5pt},clip,width=0.32\columnwidth]{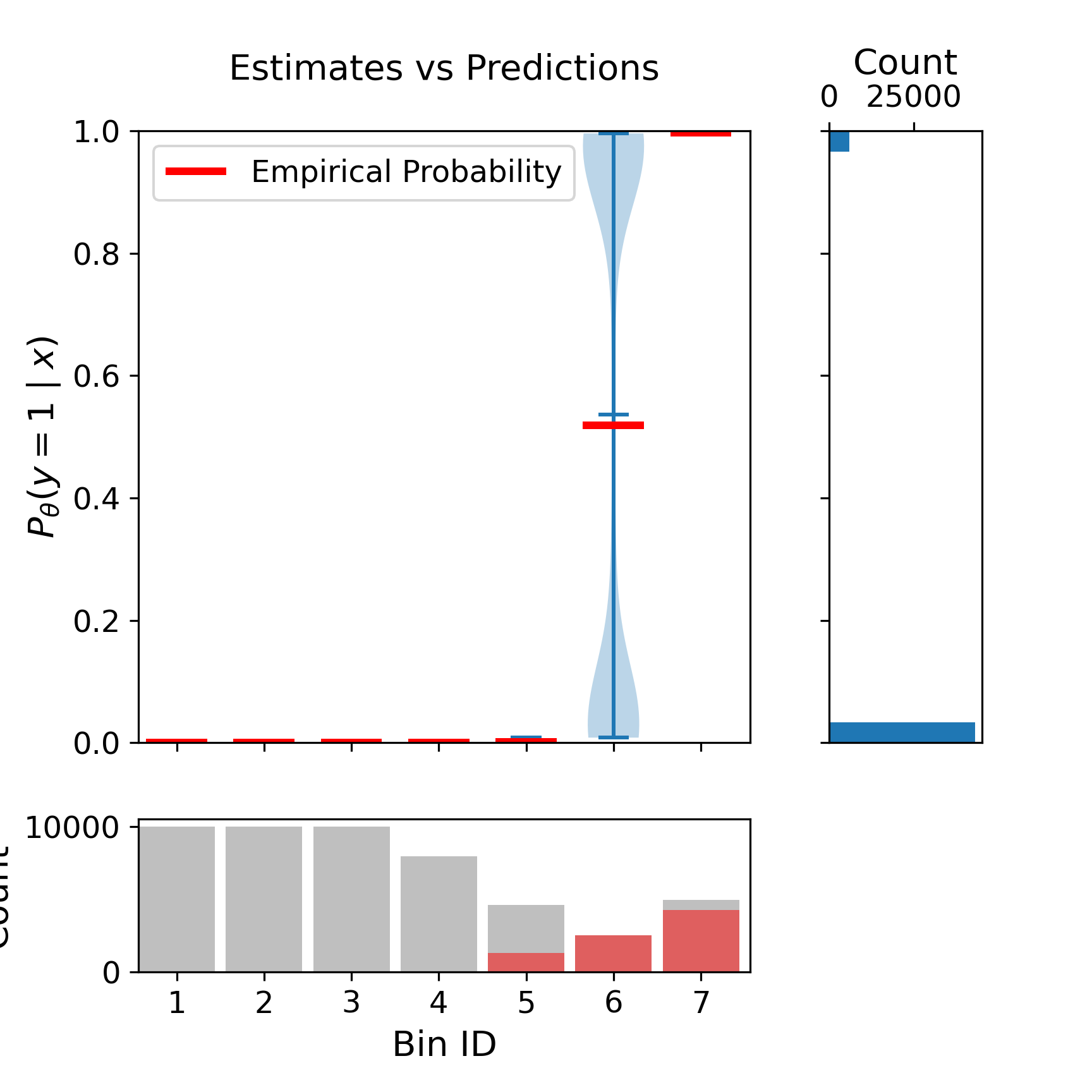}
    \hfill
    \includegraphics[trim={0 10pt 0 5pt},clip,width=0.32\columnwidth]{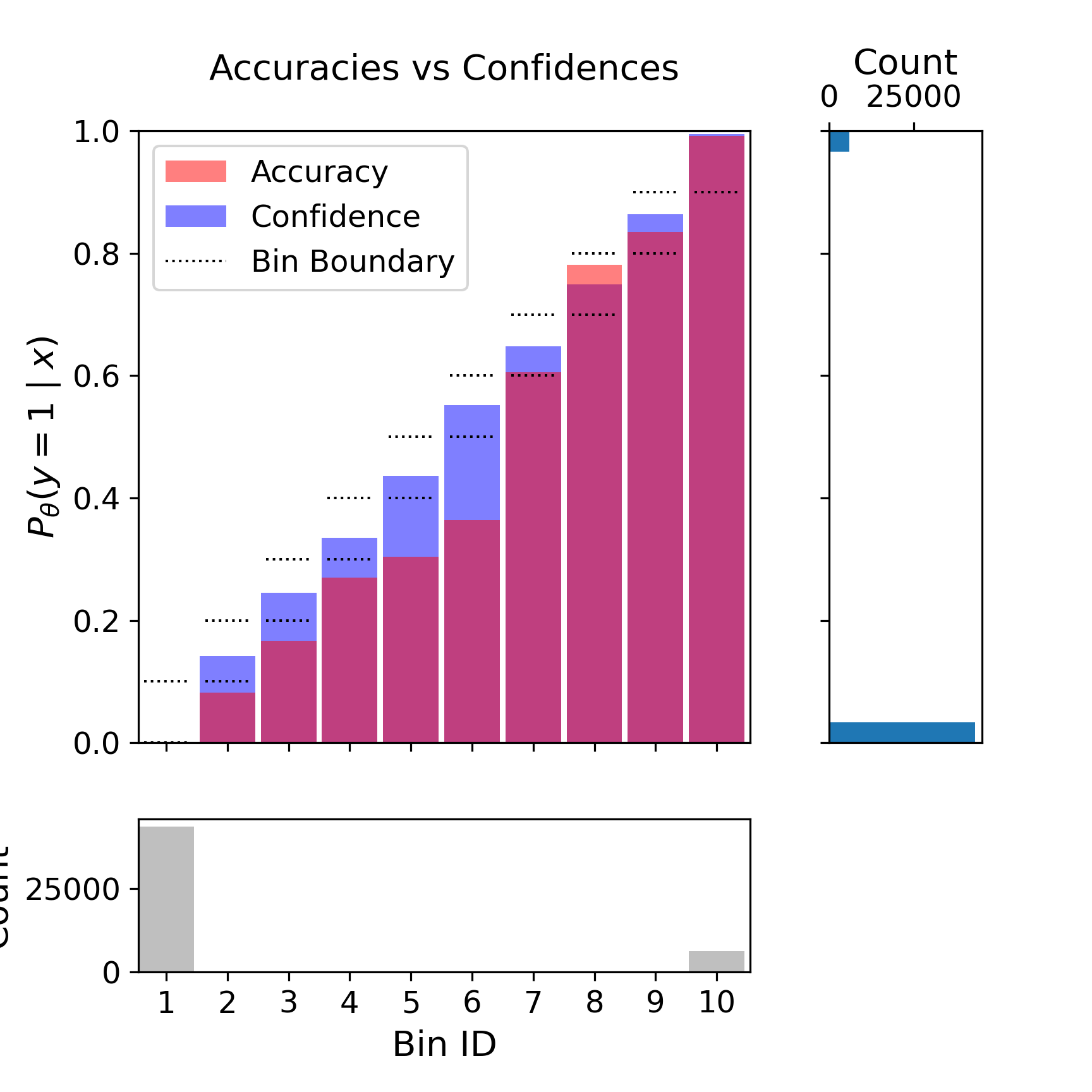}
    \hfill
    \includegraphics[trim={0 10pt 0 5pt},clip,width=0.32\columnwidth]{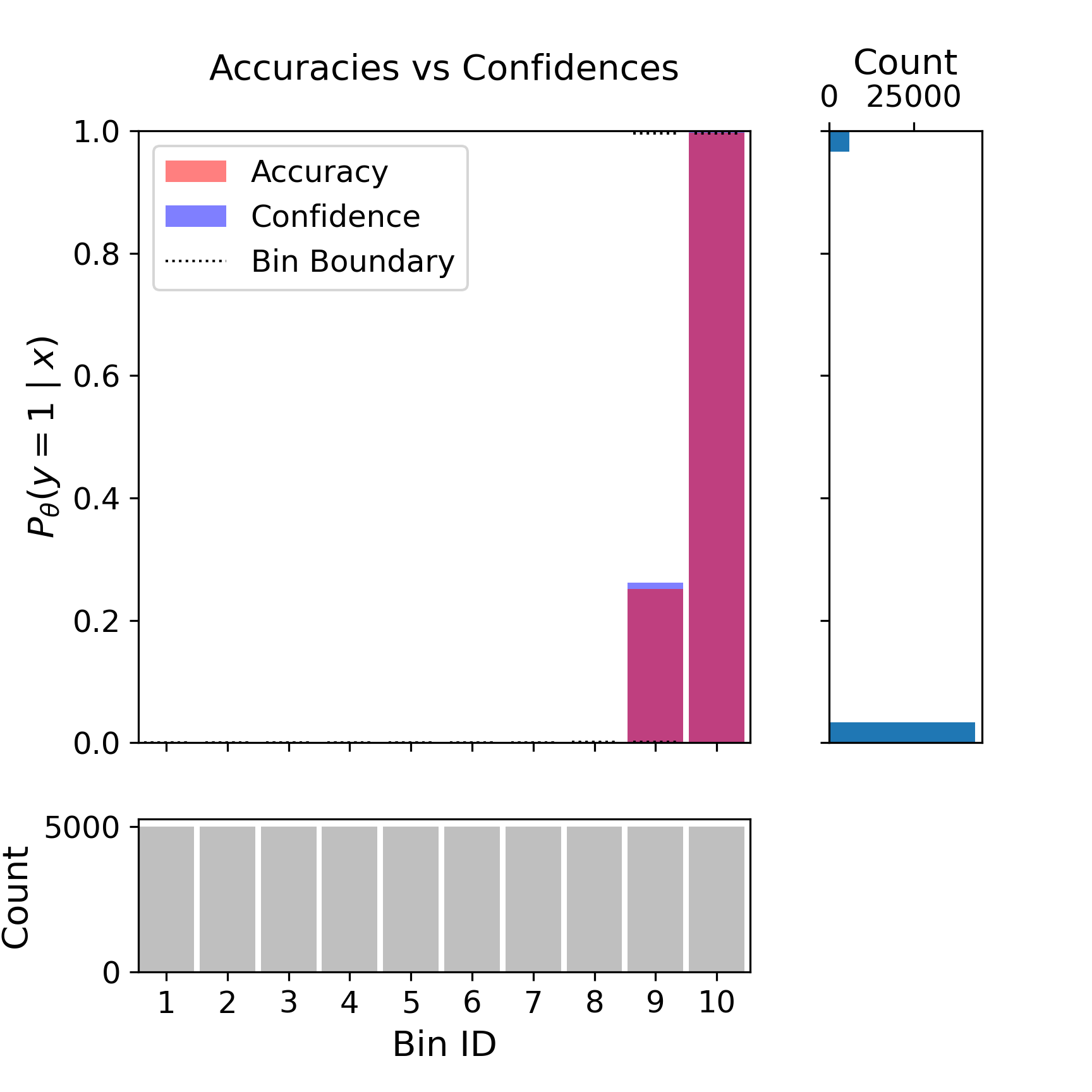}
    \caption{ResNet 152}
    \end{subfigure}
    
    \caption{Comparison of visual representations of TCE, ECE, and ACE on the ImageNet 1000 dataset. (Left) The test-based reliability diagram of TCE, (Middle) The reliability diagram of ECE (Right) The reliability diagram of ACE. Each row corresponds to a result for the model: (a) AlexNet, (b) VGG19, (c) ResNet 18, and (d) ResNet 152.}
    \label{fig:section_43_1}
\end{figure}


\end{document}